\documentclass[10pt,a4paper]{article}

\usepackage[utf8]{inputenc}
\usepackage[T1]{fontenc}
\usepackage{lmodern}
\input{glyphtounicode}
\usepackage[margin=2.4cm]{geometry}
\usepackage{graphicx}
\usepackage{amsmath,amssymb}
\usepackage{booktabs}
\usepackage{caption}
\usepackage{enumitem}
\usepackage{xcolor}
\definecolor{clsgrey}{HTML}{7F8C8D}   
\definecolor{clsorange}{HTML}{E08E0B}
\definecolor{clsblue}{HTML}{2B6CB0}
\definecolor{clspurple}{HTML}{8E44AD}
\usepackage[hidelinks]{hyperref}
\usepackage{microtype}
\usepackage{amsthm}

\theoremstyle{plain}
\newtheorem{theorem}{Theorem}
\newtheorem{proposition}{Proposition}
\newtheorem{corollary}{Corollary}
\theoremstyle{definition}

\graphicspath{{figures/}}
\setlist{itemsep=1pt,topsep=2pt}

\newcommand{\Sig}{\Sigma}
\newcommand{\Del}{\Delta}
\newcommand{\R}{\mathbb{R}}
\DeclareMathOperator{\range}{range}

\newcommand{\sstar}{s^{\ast}}
\newcommand{\IGFA}{igfa}
\newcommand{\OGD}{\textsc{ogd}}

\title{\textbf{Interference and Retention in Continual Learning}}

\author{Julius St\"ork$^{1,\ast}$\\[2pt]
\small $^1$VARTA Microbattery GmbH, Ellwangen, Germany\\
\small $^\ast$Corresponding author: \texttt{julius.stoerk@varta-ag.com}\\[2pt]
\small \textit{10.07.2026}}
\date{}

\begin{document}
\maketitle

\begin{abstract}
\noindent
Continual learning commonly relies on post-hoc mechanisms such as replay, elastic
regularization, or distillation. This work argues that forgetting should instead be modeled directly as interference between tasks.
In the frozen-feature regime,
forgetting from learning a new task is exactly the interference energy induced on the old task.
In deep networks, the same quantity is recovered through path-averaged curvature with minimal additional forward passes.

When task supports are disjoint, forgetting can be eliminated structurally and when task supports overlap in conflicting directions, a non-zero distortion floor is unavoidable.
The same geometry optimally merges models through task-aware orthogonalization.
From this analysis we derive Interference-Gated Functional
Allocation (\IGFA), a replay-free, Fisher-free method that shares
directions when tasks align and protects them when they conflict.
Across benchmarks, igfa achieves lossless retention when tasks are structurally
separable and moves unavoidable cost from irreversible forgetting into deferred but recoverable plasticity when they are not.
It matches the strongest replay-free structural baselines on dissimilar-task streams and
improves on unconditional projection when similarity makes transfer worth preserving.
\end{abstract}

\vspace{1ex}

\begin{figure}[t]\centering
\includegraphics[width=0.9\textwidth]{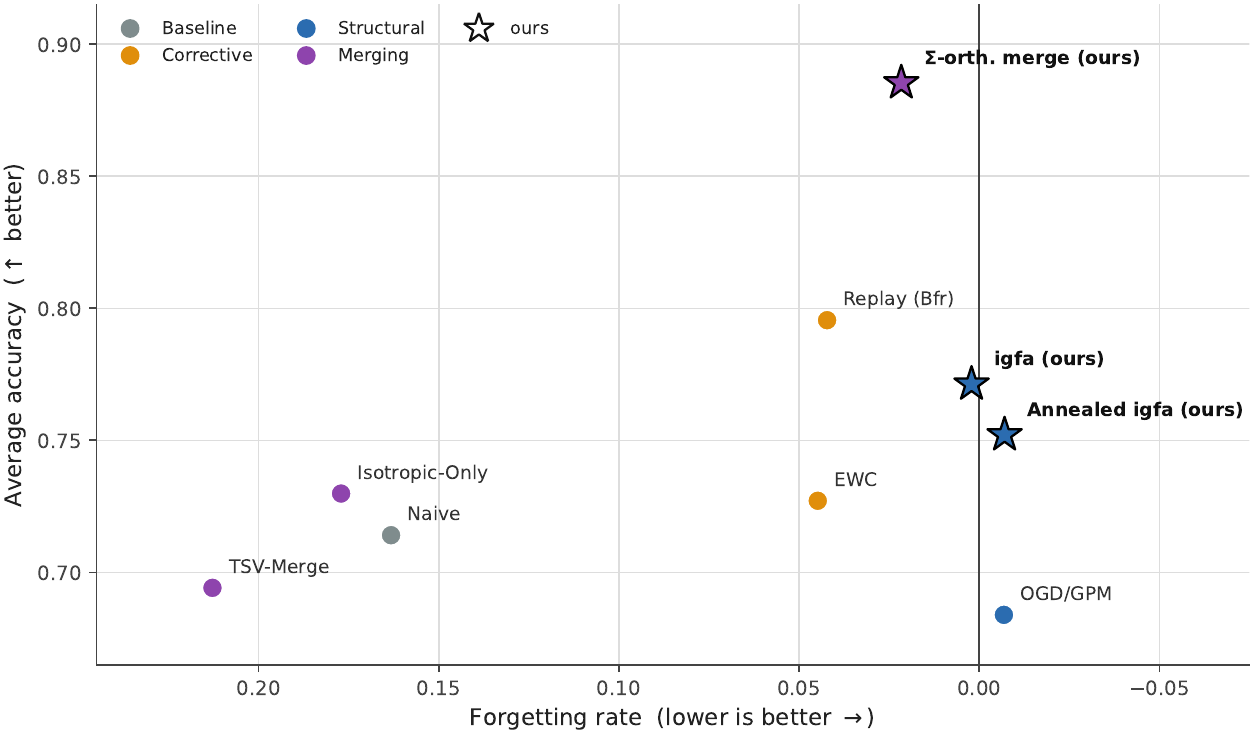}
\caption{\textbf{Accuracy versus forgetting.}  Nine methods are tested on
the exact-A1 Rotated-Digits benchmark (five rotations, frozen random features; mean over
$5$ seeds) and differentiated by method class: baseline
(\textcolor{clsgrey}{$\bullet$}), corrective (\textcolor{clsorange}{$\bullet$}),
structural (\textcolor{clsblue}{$\bullet$}), and merging
(\textcolor{clspurple}{$\bullet$}); our methods are marked ``(ours)'' and drawn as
stars in their class colour (\textcolor{clsblue}{$\bigstar$},
\textcolor{clspurple}{$\bigstar$}); higher accuracy (up) and lower forgetting (right)
are better. Corrective
(EWC, replay) and structural (\OGD{}/\textsc{gpm}, \IGFA) methods learn the
stream sequentially, whereas merging methods solve each rotation independently and
combine offline.
\emph{(i)} The $\Sig$-orthogonal merge of Theorem~\ref{thm:merge} attains near-oracle accuracy
($0.885$ against a $0.903$ per-task ceiling) with the smallest
post-merge residual floor ($D=0.032$, versus $0.20$ and $0.38$ for the isotropic and
Euclidean merges), confirming that the task-induced, non-Euclidean $\Sig$-metric
is the correct inner product for merging. \emph{(ii)} Among online, buffer- and
Fisher-free methods, our signed gate (\IGFA) attains the highest accuracy; annealing the
protection boundary (Annealed \IGFA) trades slightly lower accuracy for lower, even negative,
forgetting. Full numbers
with confidence intervals and the residual-$D$ ablation are in Sec.~\ref{si:compare}.}
\label{fig:teaser}
\end{figure}

\section{Introduction}
\label{sec:intro}

A model trained sequentially on task $B$ after task $A$ often degrades on task $A$,
a phenomenon known as catastrophic forgetting. Networks trained by stochastic gradient
descent (SGD) or its variants update from the current minibatch alone (or a smoothed
average of a short window of minibatches); the update is therefore \emph{oblivious to past
knowledge}. This obliviousness is desirable when the training data are
i.i.d., but harmful once the training distribution shifts over time in the continual
setting. Standard approaches such as replay
\cite{lopezpaz2017,chaudhry2019,buzzega2020}, elastic weight consolidation
\cite{kirkpatrick2017,zenke2017}, and distillation \cite{li2017lwf} mostly repair forgetting after interference has already occurred.
What is still missing is a predictive account of when interference is avoidable,
when it is inevitable, and what structure an update must satisfy to retain
old behavior without sacrificing useful transfer.

This work focuses on the linear-on-features regime in which
a frozen feature extractor $\phi(x)$ is paired with a trainable linear head,
covering the common frozen-backbone and parameter-efficient fine-tuning
(PEFT) settings and the first-order neural tangent \cite{jacot2018} approximation of
a fully trained network.
The forgetting of an earlier task $A$ after an
update $\Del$ caused by learning task $B$
is exactly the interference energy
$\tfrac12\,\Del^\top\Sig_A\Del$, where $\Sig_A\ = \mathbb{E}_{x \sim D_A}\left[\phi(x)\phi(x)^\top\right]$ denotes the feature second moment of task
$A$, measuring which feature directions are active for that task.
The same geometry predicts per-task forgetting accurately and
extends approximately to deeper networks with feature drift.

The paper makes three contributions.
\begin{enumerate}[leftmargin=1.6em,itemsep=3pt]
\item \textbf{Interference functional and removability.}
We derive an exact interference functional for forgetting.
In the frozen-feature regime, forgetting is exactly the old task’s interference energy under the new update,
which yields a clean criterion for when retention is lossless and when a distortion floor is unavoidable.
The geometry of $\Sig_A$ penalizes only
components in its active subspace, while updates in $\ker \Sig_A$ leave task
$A$'s loss unchanged. This yields a structural criterion for lossless retention:
interference is removable for disjoint task supports (reducing to ordinary
orthogonality in isotropic geometry), whereas overlapping supports imply a
non-zero distortion floor. The same quantity also admits an
information-theoretic interpretation through the task-inference problem.

\item \textbf{Optimal merging as $\Sig$-orthogonalization.}
We show that optimal merging is task-aware orthogonalization.
This unifies offline model merging and online continual learning as batch and incremental solutions for minimizing the same interference objective.
Section~\ref{sec:theory:merge} shows that post-merge loss decomposes into
task-wise fit and cross-task interference. Minimizing this objective requires
mutual $\Sig$-orthogonality of the task vectors, yielding the optimality
condition for model merging.

\item \textbf{Interference-Gated Functional Allocation.}
We introduce igfa, a replay-free structural controller. The method shares capacity when tasks align, protects it when they conflict, and retains only a low-rank subspace summary.
Section~\ref{sec:algo} introduces \IGFA, an online method motivated by the same functional (Fig.~\ref{fig:overview}).
\end{enumerate}

\begin{figure}[tb]\centering
\includegraphics[width=\textwidth]{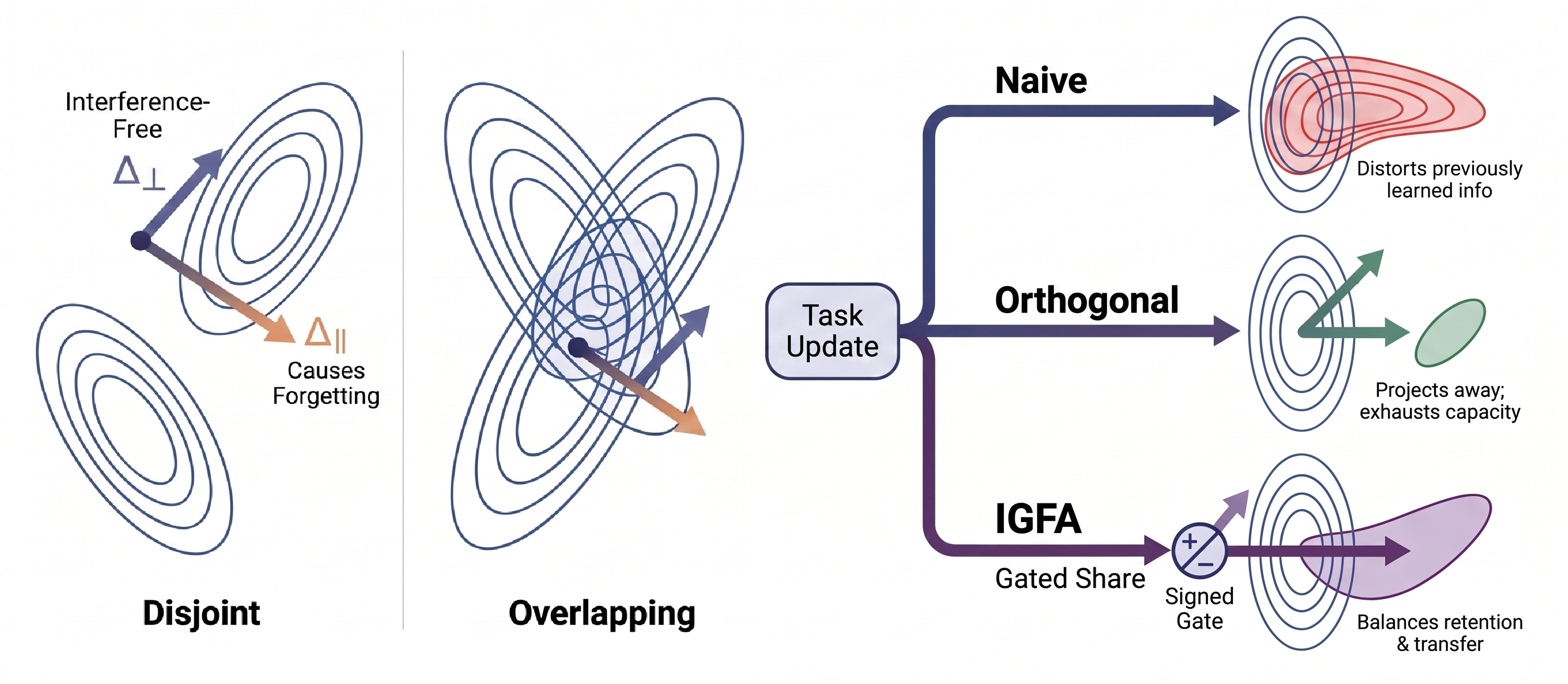}
\caption{The task-induced geometry in feature space motivates either full sharing, full separation, or selective parameter sharing.
\emph{Left (Disjoint).} When two tasks excite orthogonal feature subspaces,
the update $\Del$ is composed of a component $\Del_\parallel\in\range\Sig_A$ between the iso-loss contours that incurs forgetting
and $\Del_\perp\in\ker\Sig_A$ along the iso-loss contours where allocation is interference-free.
\emph{Centre (Overlapping).} When the supports share an active direction,
no update for $B$ can avoid $\range\Sig_A$ and a distortion floor is unavoidable.
\emph{Right.} The same functional yields three update rules. Naive descent
incurs the full interference (distorts old knowledge); orthogonal
projection (\OGD{}/\textsc{gpm} \cite{farajtabar2020,saha2021}) removes all overlap
(no forgetting, but no transfer and capacity exhaustion); \IGFA{} applies a
\emph{signed} gate---sharing aligned directions and orthogonalizing conflicting
ones---balancing retention and transfer while carrying only the occupied subspace as
state.}
\label{fig:overview}
\end{figure}

The experiments
in Section~\ref{sec:exp} are designed to test the theory directly, including lossless
allocation on disjoint supports, the relocation of unavoidable cost from forgetting
into deferred plasticity, the similarity threshold at $\sstar\approx0.26$, and the
onset of a capacity floor once the cumulative occupied rank reaches the feature
dimension, $T\,r\approx d$ (with $T$ the number of tasks in the stream, $r$ the
rank of the feature subspace each task occupies, and $d$ the feature dimension).
In the frozen-feature Split-Digits
setting, \IGFA{} matches the strongest replay-free and Fisher-free baseline at about
$0.98$ accuracy with near-zero forgetting while using no replay buffer.
Extensions to
full-network training, large-scale overlays, and direct comparison with replay at
scale are discussed in Section~\ref{sec:discussion}, with the supporting experiments
detailed in the Supporting Information.

\section{Related work}
\label{sec:related}

\paragraph{Corrective continual learning.} Replay \cite{lopezpaz2017,chaudhry2019,buzzega2020}, regularization \cite{kirkpatrick2017,zenke2017}, and distillation \cite{li2017lwf}
remain the main corrective approaches to continual learning.
They do not explicitly model interference through
task geometry however, while the herein derived interference functional predicts the
collision of the task updates and identifies when it can be removed before occuring. Replay can be
viewed as a stochastic Monte-Carlo estimator of the same task second-moment $\Sig_t$
that we estimate sample-free via a low-rank subspace.

\paragraph{Geometry- and gradient-based methods.} A related line of work controls the
update direction itself. Orthogonal Gradient Descent \cite{farajtabar2020}, Gradient
Projection Memory \cite{saha2021}, and Orthogonal Weight Modification
\cite{zeng2019owm} project updates away from subspaces associated with previous
tasks. Under a fixed-feature map, this corresponds to enforcing $\Del\in\ker\Sig_A$.
PCGrad-style methods \cite{yu2020pcgrad} resolve conflicts through pairwise gradient
comparisons rather than explicit task-subspace protection. Recent work also studies
closely related regimes: Geodesic-Aligned Gradient Projection \cite{geodesic2025}
addresses feature drift during sequential learning, and MINGLE \cite{mingle2025} uses
a null-space criterion to gate low-rank experts during test-time merging. Relative to
these approaches, the present framework derives the projection condition from an exact
forgetting functional, replaces unconditional projection with a similarity-gated
share-or-orthogonalize rule, and uses the same functional to characterize both
merging and the distortion floor. NTK-based analyses \cite{doan2021,bennani2020}
likewise relate forgetting to overlap in representation space, but do not develop the
$\Sig$-orthogonality criterion or the associated transfer gate.

\paragraph{The averaged Jacobian as a shared object with interpretability.} A
convergent notion of geometry appears in mechanistic interpretability. The Jacobian
lens of Gurnee, Lindsey, et al.~\cite{gurnee2026gws} characterizes a residual-stream activation by its
\emph{context-averaged} forward Jacobian $J_\ell=\mathbb{E}[\partial h_{\mathrm{final}}/
\partial h_\ell]$ composed with the unembedding, and identifies the low-rank span of its
active readout directions---the ``J-space''---as the subspace of \emph{verbalizable}
representations, limited in capacity and competitively occupied, a language-model global
workspace. Our interference geometry is the learning-dynamics counterpart of the same
move: where the J-lens averages the forward activation$\to$output Jacobian, we average
the Gauss--Newton curvature $\Sig_A=\mathbb{E}[J_wJ_w^\top]$ (Sec.~\ref{sec:theory:general}),
and both replace a single-point Jacobian with a distribution-averaged one for the same
stated reason---to isolate an architecture-level invariant from context-specific use.
The two induce the same partition of representation space: their causally load-bearing
J-space plays the role of our $\range\Sig_A$ (perturbations there change outputs and
incur forgetting), and the inert complement that their coordinate-patching ``leaves
unchanged'' is our $\ker\Sig_A$ (free directions). Read this way, our interference
energy and distortion floor supply a quantitative model of workspace \emph{competition}
and our occupied-rank capacity knee a rate--distortion account of its \emph{limited
capacity}; conversely, the released lens is a ready estimator of the averaged-Jacobian
subspace our method tracks at scale. We make the connection testable in
Sec.~\ref{sec:discussion}.

\paragraph{Model merging and task arithmetic.} In a parallel line of literature,
independently fine-tuned are
combined into a single model that retains useful knowledge from each task
through arithmetic on weight-space task vectors
\cite{ilharco2023,yadav2023}. These methods usually operate in Euclidean parameter
space, whereas the present analysis identifies the task-dependent metric induced by
$\Sig_t$ as the relevant geometry. Under this metric, optimal merging becomes mutual
$\Sig$-orthogonalization, with ordinary weight-space orthogonality recovered only when
$\Sig_t\propto I$. Task Singular Vectors \cite{gargiulo2025tsv} promotes orthogonality
by whitening singular directions of task vectors; by contrast, isotropic merging
\cite{marczak2025iso} shows that shared-subspace isotropy can outperform strict
orthogonalization. WUDI-Merging \cite{cheng2025wudi} is particularly close in spirit
because it links task vectors to input subspaces and minimizes interference without
data; the key difference in our work is that the metric and residual are \emph{derived} from
an exact forgetting functional rather than an approximate subspace argument.
Tangent-space continual merging \cite{holistic2026} also operates in a first-order
regime, but with similar trade-offs between scalability and functional specificity.
The common perspective adopted here is that the same interference functional governs
both offline merging and online continual allocation.

\paragraph{Scenarios, scale, and plasticity.} The standard task-, domain-, and
class-incremental taxonomy \cite{vandeven2022} distinguishes whether task identity is
given or must be inferred. In this context, the distortion floor developed later
provides a quantitative link between support overlap and task inference. Related
empirical and analytical work has examined the effects of scale and pretraining
\cite{ramasesh2022}, frozen-backbone prompt and adapter methods
\cite{wang2022l2p,wang2022dualprompt,smith2023coda}, which learn prompt pools with
input-conditioned selection under a class-incremental protocol and provide strong empirical
accuracy without structural retention guarantees (differentiated in
Sec.~\ref{sec:exp:cifar}), forgetting in linear models
\cite{evron2022}, plasticity loss \cite{dohare2024}, and convergence or task ordering
in continual linear classification \cite{jung2025,lihiratani2025}.

\section{Preliminaries: the interference functional}
\label{sec:theory}
This section specifies the setting and assumptions (Sec.~\ref{sec:theory:setup}), derives the
interference functional and the removability dichotomy (Sec.~\ref{sec:theory:functional}),
generalizes the identity beyond frozen features (Sec.~\ref{sec:theory:general}), and derives
the optimal merge together with the distortion floor (Sec.~\ref{sec:theory:merge}).

\begin{figure}[tb]\centering
\includegraphics[width=\textwidth]{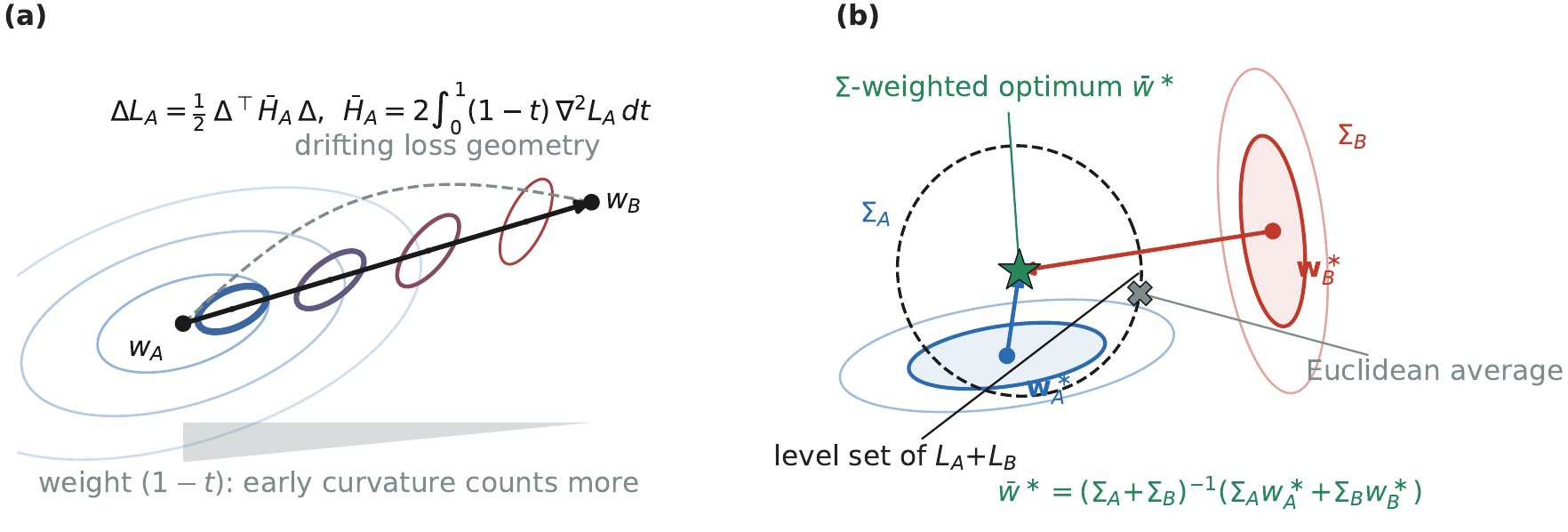}
\caption{Geometry of the general identities. \emph{(a)} Beyond frozen features the
quadratic form survives with $\Sig_A$ replaced by the path-averaged curvature
$\bar H_A$: the local curvature of $L_A$ rotates and rescales along the segment
$w_A\!\to\!w_B$ (ellipses, blue$\to$red), earlier points weighted by $(1-t)$ (wedge),
and the straight chord stands in for the drifted geodesic (dashed)
(Theorem~\ref{thm:general}). \emph{(b)} Optimal merging is $\Sig$-weighted, not
Euclidean: drawn from the exact $2{\times}2$ geometry, the joint optimum $\bar w^\ast$
(star) minimizes $L_A{+}L_B$, while the Euclidean average (cross) sits strictly outside
the optimum's level set and pays avoidable interference (Theorem~\ref{thm:merge}).}
\label{fig:thmgeom}
\end{figure}

\subsection{Setting and assumptions}
\label{sec:theory:setup}
Let $\phi:\mathcal X\to\R^d$ be a fixed feature map and $w\in\R^d$ a trainable linear
head with prediction $\hat y=w^\top\phi(x)$. For task $t$, let
$\Sig_t=\mathbb{E}_{x\sim\mathcal D_t}[\phi(x)\phi(x)^\top]\succeq0$ denote the feature
second moment, and let $w_t^\ast$ be a realizable optimum. The per-task excess loss is
\begin{equation}
L_t(w)=\tfrac12\,(w-w_t^\ast)^\top\Sig_t\,(w-w_t^\ast)\;\ge 0.
\label{eq:excess}
\end{equation}
Under squared loss, this coincides with the excess population risk.
We write the target as $y=\phi(x)^\top w_t^\ast+\varepsilon$ to separate the predictable linear component
from the irreducible residual variation. With $y=\phi(x)^\top w_t^\ast+\varepsilon$, $\mathbb{E}[\varepsilon]=0$, and
$\varepsilon\!\perp\!\phi(x)$, the squared risk satisfies $\mathbb{E}[(\hat y-y)^2]/2=L_t(w)+\tfrac12
\mathbb{E}[\varepsilon^2]$, so all subsequent loss differences depend only on $L_t(w)$.

Three
assumptions define this regime: \textbf{(A1) frozen features}---$\phi$ is fixed during training,
exact for frozen-backbone and PEFT settings and the first-order NTK \cite{jacot2018} model of
end-to-end training, so full-network claims are first-order and are tested in
Sec.~\ref{sec:exp:breakdown}; \textbf{(A2) realizability}---each $w_t^\ast$ attains zero
excess loss, so mis-specification adds only a task-constant and leaves the interference terms
untouched; and \textbf{(A3) the quadratic regime}---training minimizes Eq.~\eqref{eq:excess}
by (projected) gradient descent, so motion is confined to the excited subspace while
directions in $\ker\Sig_t$ leave $L_t$ unchanged.

\subsection{The functional and the removability dichotomy}
\label{sec:theory:functional}
Train task $A$ from a base point $w_0$ to convergence. Under A3, gradient descent
moves only in $\range\Sig_A$, so the converged solution $w_A$ satisfies $L_A(w_A)=0$
and equals $w_A^\ast$ up to an arbitrary component in $\ker\Sig_A$. For loss
differences we therefore take $w_A=w_A^\ast$. Now train task $B$, reaching $w_B$, and
define $\Del=w_B-w_A^\ast$.

\begin{theorem}[Forgetting is interference energy]
\label{thm:functional}
Under A1--A3, the forgetting of task $A$ caused by learning $B$ is
\begin{equation}
\Del L_A=L_A(w_B)-L_A(w_A^\ast)=\tfrac12\,\Del^\top\Sig_A\,\Del,
\qquad \Del=w_B-w_A^\ast.
\label{eq:functional}
\end{equation}
\end{theorem}
\begin{proof}
Since $L_A(w_A^\ast)=0$ by A2, substituting $w_B=w_A^\ast+\Del$ into
Eq.~\eqref{eq:excess} yields $L_A(w_B)=\tfrac12\Del^\top\Sig_A\Del$. No approximation
is involved; the only essential assumption is that $\phi$ remains fixed, so $\Sig_A$
is unchanged before and after learning $B$.
\end{proof}

\noindent Decompose the update as $\Del=\Del_\parallel+\Del_\perp$, with
$\Del_\perp\in\ker\Sig_A$ and $\Del_\parallel\in\range\Sig_A$. Then $\Del L_A=
\tfrac12\Del_\parallel^\top\Sig_A\Del_\parallel$, so only the component of the update
lying in task $A$'s active subspace contributes to forgetting
(Fig.~\ref{fig:overview}, left).

\begin{corollary}[Removability dichotomy]
\label{cor:removability}
An update is lossless for $A$ if and only if $\Del\in\ker\Sig_A$. Consequently there
exists a way to learn $B$ with zero forgetting of $A$ if $w_B^\ast$ is reachable
through $\ker\Sig_A$, i.e. if $w_B^\ast-w_A^\ast\in\ker\Sig_A$. Even when this fails,
$L_B$ can still be strictly reduced without forgetting whenever $w_B^\ast-w_A^\ast$ has
a non-zero component in $\ker\Sig_A\cap\range\Sig_B$. If $\range\Sig_A$ and
$\range\Sig_B$ share an active direction on which $w_A^\ast$ and $w_B^\ast$ disagree,
then no lossless update reaches $w_B^\ast$, and a positive floor remains
(Theorem~\ref{thm:floor}).
\end{corollary}
\begin{proof}
Since $\Sig_A\succeq0$, $\Del^\top\Sig_A\Del=0\Leftrightarrow\Sig_A^{1/2}\Del=0
\Leftrightarrow\Del\in\ker\Sig_A$. Projected gradient descent on $L_B$ restricted to
$w_A^\ast+\ker\Sig_A$ converges to the minimizer on that affine set. It reaches
$w_B^\ast$ if $w_B^\ast-w_A^\ast\in\ker\Sig_A$, and it strictly reduces $L_B$ whenever
the projection of $-\nabla L_B(w_A^\ast)$ onto $\ker\Sig_A$ is non-zero, equivalently
whenever $\Sig_B(w_B^\ast-w_A^\ast)$ has a non-zero component in $\ker\Sig_A$. In the
disjoint-support case, $\range\Sig_B\subseteq\ker\Sig_A$, so the entire update is
lossless.
\end{proof}

Here, ``task support'' refers to the column space $\range\Sig_t$, that is, the
subspace of feature directions activated by task $t$. Disjoint supports therefore mean
orthogonal feature subspaces, not disjoint label sets.

\paragraph{Relation to standard continual-learning metrics.} Under squared
loss the per-task term $\Del L_A=\tfrac12\Del^\top\Sig_A\Del$ is exactly the per-task
summand of the Forgetting Measure and the negative of Backward Transfer; positive
backward transfer is simply $\Del L_A<0$, which the signed gate permits and unconditional
projection forbids. The distortion floor (Sec.~\ref{sec:theory:merge}) is then a
\emph{lower bound} on the average Forgetting Measure achievable by any single-head
method, set by task confusability (Sec.~\ref{si:mi}). The geometry thus explains the
benchmark metrics and bounds them.

\subsection{The architecture-general identity}
\label{sec:theory:general}
Theorem~\ref{thm:functional} assumes frozen features. In end-to-end training, feature
drift degrades the frozen-curvature prediction, especially in deeper networks and at
larger step sizes. The natural generalization is to replace $\Sig_A$ by the curvature
of $L_A$ averaged along the displacement from $w_A$ to $w_B$.

\begin{theorem}[Architecture-general forgetting identity]
\label{thm:general}
Let $L_A$ be any twice-differentiable task-$A$ loss, and let training on $B$ move the
parameters from $w_A$ to $w_B$. Writing $\Del=w_B-w_A$,
\begin{equation}
\Del L_A=\nabla L_A(w_A)^\top\Del+\tfrac12\,\Del^\top \bar H_A\,\Del,
\qquad
\bar H_A=2\!\int_0^1(1-t)\,\nabla^2 L_A(w_A+t\Del)\,dt,
\label{eq:general}
\end{equation}
where $\bar H_A$ is the Hessian of $L_A$ averaged along the segment $w_A\!\to\!w_B$. If
$w_A$ is a stationary point, then $\nabla L_A(w_A)=0$ and
$\Del L_A=\tfrac12\Del^\top\bar H_A\Del$.
\end{theorem}
\noindent\emph{Intuitively}, a second-order Taylor expansion of $L_A$ along the segment
$w_A\!\to\!w_B$, with exact integral remainder, produces the gradient term and the
path-averaged Hessian (full derivation in Sec.~\ref{si:proofs}).

\noindent In the constant-curvate case $\nabla^2 L_A\equiv\Sig_A$, the path-averaged Hessian $\bar H_A$
collapses to the fixed-curvature matrix $\Sig_A$ by Theorem~\ref{thm:functional}.
For squared loss,
$\nabla^2 L_A(w)=\mathbb{E}_{\mathcal D_A}[J_w J_w^\top]+R(w)$, so $\bar H_A$ is a
path-averaged Gauss--Newton geometry, where $J_w=\partial f(x;w)/\partial w$ is the
parameter Jacobian and $R(w)$ is a residual-curvature term (Fig.~\ref{fig:thmgeom}a).
The removability, merging, and floor results therefore extend by replacing $\Sig_A$
with $\bar H_A$. In practice, $\bar H_A$ can be estimated from a small number of
evaluations of $L_A$ along the segment, without forming the Hessian explicitly. This principle motivates
 \IGFA{}  to track the running Gauss--Newton
subspace $\mathbb{E}[J_w J_w^\top]$.

The generalization in fact reaches its cleanest form not in parameter space but in
\emph{function} space, where the geometry is exact for arbitrary architectures.

\begin{theorem}[Function-space interference identity]
\label{thm:funcspace}
Let $f_A,f_B$ be the functions realized before and after learning $B$, by any
architecture and training procedure, and let $L_A(f)=\tfrac12\mathbb{E}_{\mathcal
D_A}(f(x)-y)^2$. Writing $\Del\!f=f_B-f_A$ and $r_A=f_A-y$,
\begin{equation}
\Del L_A=\mathbb{E}_{\mathcal D_A}\!\big[\Del\!f\;r_A\big]
+\tfrac12\,\mathbb{E}_{\mathcal D_A}\!\big[\Del\!f^{\,2}\big].
\label{eq:funcspace-main}
\end{equation}
If $A$ is realizable and trained to convergence ($r_A=0$), forgetting equals the
function-space interference energy $\tfrac12\|\Del\!f\|_{L^2(\mathcal D_A)}^2$. Learning
$B$ is lossless iff $\Del\!f=0$ on $\mathrm{supp}\,\mathcal D_A$; and if the training
velocity stays in the kernel of the current Gauss--Newton metric, $L_A$ is conserved
exactly in continuous time.
\end{theorem}
\noindent Theorem~\ref{thm:functional} is the linearized version of
Eq.~\eqref{eq:funcspace-main}: under A1, $\Del\!f=\Del^\top\phi$ recovers
$\tfrac12\Del^\top\Sig_A\Del$. The dichotomy, the distortion floor
(Sec.~\ref{sec:theory:merge}), and the retention guarantee thus hold with A1 removed
entirely, with feature support replaced by input support; A1 is the regime in which the
geometry becomes computable in closed form and finite steps stay exact. Nor is the
geometry a quadratic-loss artifact: for the whole canonical (Bregman) loss
family---softmax cross-entropy included---the identity holds exactly with the energy
replaced by a Bregman divergence. On a converged task, cross-entropy forgetting
\emph{equals} the expected Kullback--Leibler divergence between the old and new
predictive distributions on that task's data, and the distortion floor becomes a
density-weighted Jensen--Shannon divergence between the tasks' conditionals, in closed
form (Sec.~\ref{si:beyond}). The full
statement, proofs, and validations on deep networks---machine-precision agreement at
depths where the frozen-curvature predictor degrades to $r\approx0.4$---are in
Sec.~\ref{si:beyond}, along with two routes to finite-step parameter-space exactness
(activation-region certificates for ReLU networks and predictor--corrector retraction).

\begin{proposition}[Estimation-error budget]
\label{prop:esterror}
Protection with an \emph{estimated} basis $\hat U$ whose largest principal angle to
$\range\Sig_A$ is $\epsilon$ is not exactly lossless: under gradient noise of variance
$\sigma^2$ and step size $\eta$ over $t$ steps, the induced forgetting obeys the
diffusion law $\mathbb{E}[\Del L_A]=\tfrac12\eta^2\sigma^2 t\,r\sin^2\epsilon+o(\sin^2\epsilon)$,
vanishing as $\epsilon\to0$.
\end{proposition}
\noindent The method is therefore a \emph{controlled} approximation of the exact
identity with the retention loss governed by a single measurable
quantity (basis error $\epsilon$), which the recursive tracker drives down and the
drift velocity monitors (validated in \texttt{evaluate\_upgrades.py}. This law inverts to
a retention-aware step-size bound, Sec.~\ref{si:memory}) and isolates the two residual
gaps for end-to-end deployment as quantified basis error and finite step size.

\subsection{Optimal Merging and the Distortion Floor}
\label{sec:theory:merge}
Represent $K$ task solutions as $\Del_t=w_t^\ast-w_0$ from a shared base $w_0$, and
merge them as $\theta^\ast=w_0+\sum_t\Del_t$. Substituting into
Eq.~\eqref{eq:excess}, the loss incurred on task $t$ is
\begin{equation}
L_t(\theta^\ast)=\tfrac12\Big(\textstyle\sum_{t'\neq t}\Del_{t'}\Big)^{\!\top}
\Sig_t\Big(\textstyle\sum_{t'\neq t}\Del_{t'}\Big),
\label{eq:merge-energy}
\end{equation}
that is, the interference energy of the other task vectors measured in task $t$'s
geometry.

\begin{theorem}[Optimal merging is $\Sig$-orthogonalization]
\label{thm:merge}
The total merge loss $\sum_t L_t(\theta^\ast)$ vanishes whenever the task vectors are
pairwise $\Sig$-orthogonal, i.e. $\Del_{t'}\in\ker\Sig_t$ for all $t'\neq t$; this
condition is also generically necessary. When it fails, the minimum total loss over
shared heads is attained at
$\theta_J=\big(\sum_t\Sig_t\big)^{+}\big(\sum_t\Sig_t w_t^\ast\big)$, and for two
tasks
\begin{equation}
w_J=(\Sig_A+\Sig_B)^{+}(\Sig_A w_A^\ast+\Sig_B w_B^\ast).
\label{eq:merge-two}
\end{equation}
The residual $\mathcal{D}:=\sum_t L_t(\theta_J)$ is the \emph{distortion floor}.
\end{theorem}
\begin{proof}
Since each $L_t(\theta^\ast)\ge0$, the sum vanishes iff every term vanishes, i.e.
$\Sig_t^{1/2}\sum_{t'\neq t}\Del_{t'}=0$. A sufficient, and generically necessary,
condition is $\Del_{t'}\in\ker\Sig_t$ for all $t'\neq t$. The minimizer of $\sum_t
L_t(\theta)=\tfrac12\sum_t(\theta-w_t^\ast)^\top\Sig_t(\theta-w_t^\ast)$ satisfies
$\sum_t\Sig_t(\theta-w_t^\ast)=0$, which gives $\theta_J$; Eq.~\eqref{eq:merge-two} is
the two-task case.
\end{proof}

\noindent The relevant inner product for merging is therefore the task-dependent
$\Sig_t$-metric
(Fig.~\ref{fig:thmgeom}b); Euclidean orthogonality is the white-feature special case
$\Sig_t\propto I$. The same residual $\mathcal{D}$ also governs the online
continual-learning setting, so merging and continual learning optimize the same
functional under batch and streaming schedules.

\paragraph{Residual distortion under non-orthogonal task geometry.}\label{sec:theory:ceiling}
Optimal merging needs $\Sig$-orthogonality; when the supports force it to fail, the
residual $\mathcal{D}$ of Theorem~\ref{thm:merge} quantifies the unavoidable error of
task ambiguity. Consider two tasks as a single problem in which each input carries a
hidden label $T\in\{A,B\}$ and a shared head must predict the target of the hidden task.

\begin{figure}[tb]\centering
\includegraphics[width=\textwidth]{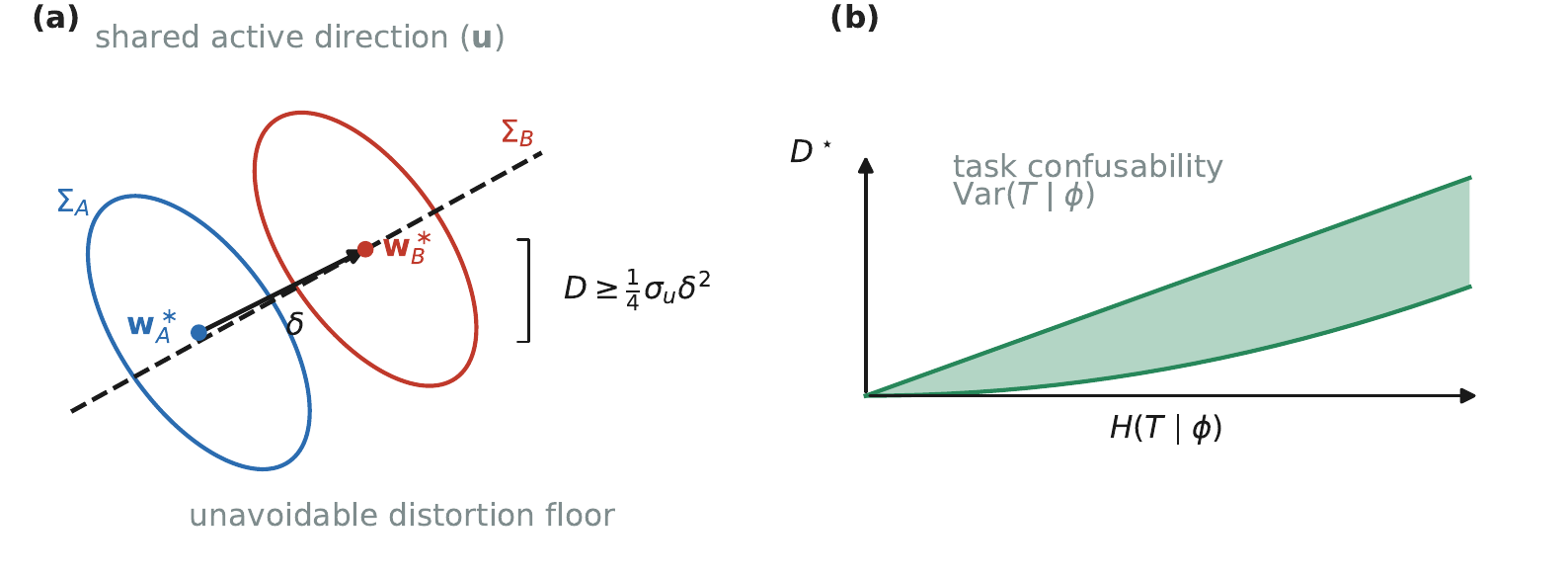}
\caption{The distortion floor and its information-theoretic reading.
\emph{(a)} A floor appears only on a shared active direction $u\in\range\Sig_A\cap
\range\Sig_B$ on which the targets disagree by $\delta$, with
$\mathcal{D}\ge\tfrac14\sigma_u\delta^2$ (Theorem~\ref{thm:floor}).
\emph{(b)} The Bayes single-head excess risk $\mathcal{D}^\star$ factors into the
geometric target gap $(\phi^\top\delta)^2$---the displacement $\delta$ in (a)---and the
task-confusability factor $\operatorname{Var}(T\!\mid\!\phi)$, which is two-sidedly
controlled by the conditional entropy $H(T\!\mid\!\phi)$ (Theorem~\ref{thm:mi}; band).}
\label{fig:thmfloor}
\end{figure}

\begin{theorem}[Distortion floor and inference limit]
\label{thm:floor}
Suppose the tasks share an active direction $u$, with $\|u\|=1$ and
$u\in\range\Sig_A\cap\range\Sig_B$, and let $\delta=u^\top(w_B^\ast-w_A^\ast)$. Then
the distortion floor obeys the lower bound
\begin{equation}
\mathcal{D}\ \ge\ \tfrac14\,\sigma_u\,\delta^2,\qquad
\sigma_u=\big(u^\top(\Sig_A^{-1}+\Sig_B^{-1})u\big)^{-1},
\label{eq:floor-bound}
\end{equation}
where $\Sig_A^{-1}$ and $\Sig_B^{-1}$ denote the inverses of $\Sig_A$ and $\Sig_B$
restricted to the shared active subspace $S=\range\Sig_A\cap\range\Sig_B$. On the
shared support, the excess Bayes risk of the Bayes-optimal single head over the task
oracle equals the merge residual.
\end{theorem}
\noindent  Restricting the merge to the shared direction $u$ reduces it to a
scalar two-point problem whose $\Sig$-weighted optimum leaves residual
$\tfrac14\sigma_u\delta^2$; summing over shared directions and identifying the Bayes single
head with the posterior mean gives the bound and the inference claim (Sec.~\ref{si:proofs}).

\noindent Support overlap is therefore the common limit to both
$\Sig$-orthogonal merging and task inference; a positive floor appears only on shared
active directions where the tasks also disagree. This yields the geometric counterpart
of the task- versus class-incremental distinction \cite{vandeven2022}: disjoint
supports make the task identity effectively free, whereas overlapping supports can
induce a non-zero floor when the tasks disagree on shared active directions.

\begin{theorem}[Information--estimation identity for the floor]
\label{thm:mi}
Let $T\sim\mathrm{Unif}\{A,B\}$, $\phi\mid T\sim\mathcal D_T$,
$y=\phi^\top w_T^\ast+\varepsilon$, with $\varepsilon\perp(T,\phi)$,
$\mathbb{E}\varepsilon=0$, and $\mathbb{E}\varepsilon^2=\sigma^2$. Writing
$\eta(\phi)=\Pr(T{=}A\mid\phi)$ and $\delta=w_A^\ast-w_B^\ast$, the excess Bayes risk
of a single shared head over the task oracle is
\begin{equation}
\mathcal{D}^\star
=\tfrac12\,\mathbb{E}_\phi\!\big[\operatorname{Var}(T\!\mid\!\phi)\,(\phi^\top\delta)^2\big]
=\tfrac12\,\mathbb{E}_\phi\!\big[\eta(1-\eta)\,(\phi^\top\delta)^2\big].
\label{eq:mi-identity}
\end{equation}
With $h_b(p)=-p\ln p-(1-p)\ln(1-p)$ the binary entropy in nats, this implies the
two-sided bounds
\begin{equation}
\tfrac18\,\mathbb{E}_\phi\!\big[(\phi^\top\delta)^2\,h_b(\eta)^2\big]\ \le\
\mathcal{D}^\star\ \le\
\tfrac14\,\mathbb{E}_\phi\!\big[(\phi^\top\delta)^2\,h_b(\eta)\big].
\label{eq:mi-sandwich}
\end{equation}
At a fixed non-zero target-gap profile $(\phi^\top\delta)^2$, the floor increases with
$H(T\!\mid\!\phi)$ and decreases with $I(T;\phi)$; in particular it vanishes when the
label is fully identifiable, $I(T;\phi)=H(T)$. The merge floor of
Theorem~\ref{thm:merge} is the linear-head restriction, so
$\mathcal{D}^\star\le\mathcal{D}_{\mathrm{merge}}$, with equality at both extremes.
\end{theorem}
\noindent The law of total variance splits the Bayes single-head risk into
the oracle noise plus the posterior-variance-weighted target gap
$\eta(1-\eta)(\phi^\top\delta)^2$, giving Eq.~\eqref{eq:mi-identity}; the entropy sandwich
follows from the pointwise bound $\tfrac14 h_b(p)^2\le p(1-p)\le\tfrac12 h_b(p)$
(full proof in Sec.~\ref{si:proofs}).

\noindent Equation~\eqref{eq:mi-identity} is exact for arbitrary feature distributions
and noise models. More generally, the floor depends jointly on task confusability and
target disagreement (Fig.~\ref{fig:thmfloor}b), so no exact proportionality
$\mathcal{D}^\star\propto I(T;\phi)$ holds: the floor vanishes as the target gap
$\delta\to0$ even at fixed $I(T;\phi)>0$, so no tight weight-free constant can relate
the floor to mutual information. The same identity extends to $K$ tasks with arbitrary
priors, with corresponding multiclass entropy bounds in terms of the Gini impurity
(Sec.~\ref{si:mi}); both the binary identity and the multiclass bounds are verified
numerically there. Under cross-entropy the sandwich closes to an \emph{identity}: for
any number of tasks, the floor of a single predictor equals $T\cdot I(T;Y\mid X)$---the
conditional mutual information between task identity and label given the input---and
removability becomes the statement $Y\perp T\mid X$
(Sec.~\ref{si:beyond}). Under squared loss,
Theorem~\ref{thm:functional} also recovers the standard per-task forgetting term used
in the Forgetting Measure and Backward Transfer; Theorem~\ref{thm:mi} therefore gives a
lower bound on the average forgetting achievable by any single-head method in terms of
task confusability.

\section{Method: interference-gated allocation}
\label{sec:algo}

The functional above implies a family of interference-gated functional allocation (\IGFA)
methods that control \emph{where} an update is allowed to act, parameterized by
schedule (offline or online) and by the share-versus-orthogonalize gate
(Fig.~\ref{fig:overview}, right).

\paragraph{Offline: optimal merge.} Given $K$ trained task vectors $\{\Del_t\}$ and
per-task second moments $\{\Sig_t\}$, estimated from each task's activations or
probes, solve for the $\Sig$-orthogonalized vectors and coefficients that minimize the
total interference in Eq.~\eqref{eq:merge-energy}. This is a generalized eigenvalue or
Procrustes problem, whose two-task case is Eq.~\eqref{eq:merge-two}, and provides a
principled replacement for weight-space task arithmetic.

\paragraph{Online: \IGFA.} As tasks stream, maintain a low-rank orthonormal summary
$U$ of the occupied function subspace, given by the top directions of the accumulated
$\Sig_t$ under streaming PCA. $U$ is the only persistent state: there is no exemplar
buffer and no per-weight Fisher matrix. For each new-task gradient $g$:
\begin{enumerate}[leftmargin=1.5em,itemsep=0pt]
\item Split $g=g_\parallel+g_\perp$, with $g_\parallel=UU^\top g$ the component in the
occupied subspace.
\item For each occupied direction $u$ touched by $g_\parallel$, estimate the signed
similarity $s$ between the new task's target direction and $u$.
\item Gate the update: if $s>\sstar$ (aligned, transfer), keep $g_\parallel$ and share
the direction; if $s<\sstar$ (conflict), remove $g_\parallel$ and step along $g_\perp$,
protecting the old task's function by construction.
\item Update $w$ with the gated gradient, then grow $U$ with the novel directions in
$g_\perp$ by streaming PCA, rank-capping and overwriting the least important direction
when full---rate--distortion-graceful forgetting.
\end{enumerate}
Because the old-task function is first-order invariant along orthogonalized directions,
retention is structural rather than penalty-based. Because aligned directions are
shared rather than suppressed, transfer is preserved, which unconditional projection
removes (see Algorithm~1).

\begin{figure}[tb]\centering
\fbox{\begin{minipage}{0.95\textwidth}\small
\textbf{Algorithm~1: the interference-ledger loop (measure $\to$ attribute $\to$ act),
with \IGFA{} as its default control instance.}\\[2pt]
\textbf{State (one object):} tracked second moment $\widehat\Sig$ (or its
Frequent-Directions sketch); per-task bases $\{B_t\}$ with rank cap $k_{\max}$;
held-out micro-cache gradients $\{\nabla L_u\}$. Every quantity below is a read-out
of this state.\\[2pt]
\textbf{For each task} $n$ (gradient stream $g$, feature second-moment estimate
$\widehat\Sig_n$; $B_n\leftarrow$ top-$r$ eigenvectors, streaming PCA):
\begin{enumerate}[leftmargin=1.5em,itemsep=0pt]
\item \textbf{Measure:} overlaps $\mathrm{sim}(B_t,B_n)=\|B_t^\top B_n\|_F^2/r$, share
density, predicted floor, occupied rank, drift velocity, OOD ratio, estimator
confidence.
\item \textbf{Attribute} the predicted loss: floor (irreducible) $/$ capacity (rank cap
binding) $/$ control (the rest)---each term is measurable and has exactly one remedy
(Sec.~\ref{sec:discussion}).
\item \textbf{Act:} \emph{skip} if OOD; \emph{expand or drop} the smallest-eigenvalue
direction if the cap binds (rate--distortion-graceful); \emph{replay} where recoverable
excess per exemplar is largest; otherwise the \textbf{gated update}: per minibatch
solve the monotone-retention QP of Corollary~\ref{cor:qpgate},
$v=\operatorname{argmin}\|v+g\|^2$ subject to the chosen constraint rows, and step
$w\leftarrow w+\eta\,v$. \IGFA's default rows: equality constraints from every past
basis with $\mathrm{sim}(B_t,B_n)<\sstar$ (conflicting), aligned tasks left shared---%
equivalently $g\leftarrow(I-QQ^\top)g$ with $Q=\mathrm{orth}$ of the conflicting bases.
\item \textbf{Update state:} fold the batch into $\widehat\Sig$; grow the basis store;
refresh micro-caches; \emph{defer} (protect-all) while confidence is low.
\end{enumerate}
\textbf{Inference:} single forward pass through $w$ (plus a task router only in the
class-incremental case; Sec.~\ref{sec:theory:ceiling}).
\end{minipage}}
\end{figure}

\paragraph{One gate, many methods.} Step 3's QP is not one design choice among many:
every protection rule in this literature is the \emph{same} program under a different
constraint set, and the reference implementation is a single function
(\texttt{evaluate\_unified\_gate.py}, all recoveries machine-checked):
\begin{center}\small
\begin{tabular}{ll}
\toprule
constraint rows & recovered method \\
\midrule
none & naive sharing \\
one sampled $\nabla L_u$ (inequality) & A-GEM / the threshold-free sign gate \\
all stored $\nabla L_u$ (inequalities) & the active-set gate (Sec.~\ref{si:beyond}) \\
all protected bases (equalities) & \OGD{}/\textsc{gpm} projection \\
bases with $\mathrm{sim}<\sstar$ only (equalities) & \IGFA's threshold gate \\
\bottomrule
\end{tabular}
\end{center}
The recoveries are exact rather than analogies: the closed-form sign gate and \OGD{}
agree with the QP to $10^{-15}$, and \IGFA's full training trajectory coincides with
the overlap-filtered-equality instance step for step ($10^{-15}$ over five tasks). The
methods differ only in which rows they enforce and how well those rows are estimated,
which is precisely the axis the scale experiments measure (Sec.~\ref{si:llm}).

\paragraph{Complexity, and one state object.} The only persistent state is the triple
of Algorithm~1---the tracked second moment (or its sketch), the per-task bases, and the
held-out micro-caches---giving $\mathcal{O}(dk)$ memory with $k\le k_{\max}$,
independent of the number of stored examples and of dataset size. This contrasts with
$\mathcal{O}(|\mathcal B|d)$
memory for a replay buffer $\mathcal B$, and with $\mathcal{O}(d)$ to $\mathcal{O}(d^2)$
memory for diagonal-to-full Fisher methods. Per update, the projection costs
$\mathcal{O}(dk)$, essentially a $k$-dimensional inner product and subtraction, and is
the dominant overhead. The per-task basis update is a thin eigendecomposition of
$\widehat\Sig_n$, costing $\mathcal{O}(dr^2)$ amortized once per task. In a
frozen-backbone overlay, $d$ is the adapter width rather than the full model width, so
$U$ remains small in absolute terms. The rank cap $k_{\max}$ is the capacity budget of
Sec.~\ref{sec:exp:simcap}: when it binds, the discarded direction is the least excited
one, which bounds the induced forgetting by its eigenvalue. At scale, the covariance
tracking itself can be replaced by an $\mathcal{O}(rd)$ Frequent-Directions sketch with
no $d\times d$ eigendecomposition (Sec.~\ref{si:cheapgn}). Because every quantity in
this paper---overlaps and $\sstar$, the QP constraint rows, the quadratic and KL/JSD
floors, the OOD ratio, drift velocity, and capacity occupancy---is a read-out of this
one state object, its failure modes are unified rather than method-specific: a stale
estimate announces itself through the drift velocity, small-sample floor estimates are
bias-corrected and \emph{abstain} when a null test cannot rule out zero conflict, and
constraint anchors must be held-out and adequately sized, since a train-anchored gate
provably protects memorized noise as if it were knowledge (Sec.~\ref{si:beyond}).

\paragraph{Relation to \OGD{}/\textsc{gpm}.} Setting $\sstar=\infty$, so that all past
subspaces are protected, recovers orthogonal projection \cite{farajtabar2020,saha2021}
exactly. \IGFA{} is therefore a strict generalization: it preserves the protective
behavior of \OGD{}/\textsc{gpm} in the conflict regime, but restores sharing when tasks
are aligned. Offline, applying the same primitive to a batch of pre-trained task
vectors yields the $\Sig$-orthogonal merge of Theorem~\ref{thm:merge}. Online, the
threshold $\sstar$ is recovered by a validation-greedy rule (Sec.~\ref{si:sstar}).

\paragraph{Why a shared head with a gate, not one head per task.} The simplest way to
avoid interference is to assign each task its own head and never share. Strict
isolation, however, defeats the purpose of continual learning: a network that isolates task
$B$ in a separate head cannot carry what it learned there back to task $A$, so forward and
backward transfer become impossible. \IGFA{} instead keeps a single shared head and
resolves interference inside it with the similarity gate. The gate is the rule that
permits sharing when similarity exceeds $\sstar$ and orthogonalizes only when sharing
would create conflict, rather than isolating everything by default.

\paragraph{Online operation under feature drift.} In a stream, one cannot look ahead to
the segment $w_A\!\to\!w_B$ that defines $\bar H_A$ in Theorem~\ref{thm:general}. The
online surrogate is therefore to track the geometry as it drifts, maintaining a running
estimate of the Gauss--Newton second moment by the decayed update
\begin{equation}
C\leftarrow\gamma C+\mathbb{E}_{\text{batch}}[J_w J_w^\top],
\label{eq:gn-update}
\end{equation}
and taking $U$ to be its significant eigenvectors via incremental or streaming PCA.
This is the self-correcting counterpart of a stale protected subspace fixed once at
task onset, as in \OGD{}/\textsc{gpm}. The distinction matters under drift: a stale
subspace treats drifted copies of the same direction as new ones, so its rank keeps
growing until it fills the parameter space, whereas the recursive estimate recognizes
them as one low-rank subspace and stays bounded (Sec.~\ref{sec:exp:drift}). This is the same
failure mode reported for subspace-projection merging in large language models
\cite{merging2025llm}, and the proposed solution is the recursive Gauss--Newton update.

\paragraph{Regimes.} \IGFA{} instantiates naturally in three deployment settings:
(i) \emph{from scratch}, acting on the full parameter matrix; (ii) \emph{frozen-backbone
$+$ PEFT}, modulating only adapter, LoRA, or head parameters, where Assumption A1 and
hence the functional are exact; and (iii) \emph{overlay-native} operation, embedding
subspace tracking inside a read-write decomposition over a frozen transformer, which is
the natural regime for replay-free continual pretraining.

\paragraph{When to use \IGFA{} rather than projection, replay, or a Fisher penalty.} The
choice is itself a measurable pre-flight decision, from pairwise subspace overlap and
target agreement (Sec.~\ref{si:feasibility}). \emph{Disjoint supports} $\Rightarrow$
structural allocation is lossless, so \IGFA{} reduces to \OGD{} and neither replay nor a
Fisher penalty is needed. \emph{Overlapping supports, aligned targets} $\Rightarrow$
share, which the signed gate grants and unconditional projection forfeits.
\emph{Overlapping supports, conflicting targets} $\Rightarrow$ an irreducible floor, and
the only decision is where its cost lands: forgetting (naive), deferred fit (projection),
or stored data (replay). \emph{Cumulative rank near capacity} $\Rightarrow$ add capacity,
replay, or graceful overwriting. The gate's marginal value over projection is organized by
the stream's share density and is significant where overlaps straddle the
threshold (Sec.~\ref{sec:exp:cifar}). Two operating assumptions bound this framing and are
made explicit here rather than deferred: task boundaries are treated as known (relaxed by
the recursive tracker and the velocity detector of Sec.~\ref{si:extensions}), and for
language the retention mechanism transfers but the transfer-recovering similarity signal
is the threshold-free functional gate of Sec.~\ref{si:beyond} rather than input geometry.

\section{Experiments}
\label{sec:exp}

We organize the experiments around the paper's central claims.
Sections~\ref{sec:exp:functional}--\ref{sec:exp:simcap} validate the theory in the
linear-on-features regime ($d=24$, random-subspace tasks, unit-norm targets),
Section~\ref{sec:exp:breakdown} tests the architecture-general correction under feature
drift, Section~\ref{sec:exp:real} evaluates \IGFA{} on real benchmarks in both the
dissimilar-task and similar-task regimes, Section~\ref{sec:exp:cifar} scales the comparison
to a frozen ViT backbone, and Section~\ref{sec:exp:drift} contrasts
recursive subspace tracking with a stale protected subspace under drift. Experiments on
additional benchmarks, scale,
and the algorithmic refinements are reported in full in the Supporting Information. All
figures regenerate from \texttt{experiments.py} and \texttt{experiments\_rev.py} with
fixed seeds.

\subsection{The functional is exactly the forgetting}
\label{sec:exp:functional}
We draw random task pairs, learn $A$ then $B$, and compare the measured forgetting of
$A$ against the predicted interference energy $\tfrac12\Del^\top\Sig_A\Del$
(Fig.~\ref{fig:functional}). For a linear head the two coincide to machine precision
over $160$ pairs, and for a one-hidden-layer network the prediction formed from the
last-layer feature covariance still tracks the measured forgetting with Pearson
$r>0.9999$. In the frozen-feature regime,
forgetting is the same quantity as interference energy.

\begin{figure}[tb]\centering
\includegraphics[width=\textwidth]{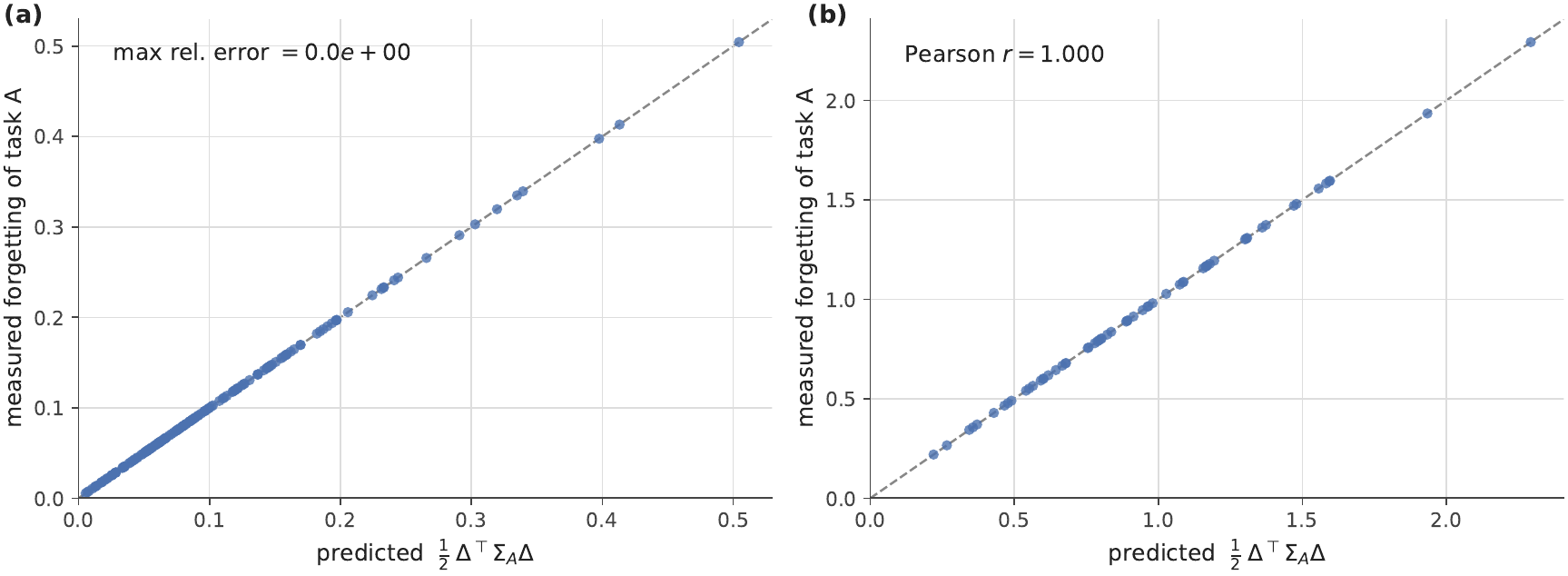}
\caption{The interference functional \emph{is} the forgetting. Measured
backward-forgetting of task $A$ after learning task $B$ against the predicted energy
$\tfrac12\Del^\top\Sig_A\Del$. \emph{(a)} A linear head on frozen features: the
identity is exact to machine precision. \emph{(b)} A nonlinear network, with the
prediction formed from the last-layer feature covariance: first-order agreement,
Pearson $r>0.9999$. The dashed line is $y=x$.}
\label{fig:functional}
\end{figure}

\subsection{Lossless allocation, relocation of cost, and the floor}
\label{sec:exp:alloc}
On disjoint supports, orthogonal allocation is lossless: learning $B$ in $\ker\Sig_A$
leaves $L_A<10^{-17}$ while still driving $L_B\to0$, exactly as
Corollary~\ref{cor:removability} predicts. On overlapping supports, the cost cannot be
removed, only relocated: in the conflicting case, naive descent forgets $A$ by
$\Del L_A=13.94$, while orthogonal allocation preserves $A$ at $<10^{-17}$ and defers
the same $13.94$ as residual on $B$ (Fig.~\ref{fig:alloc}, left). Sweeping
feature-support overlap shows the same pattern continuously: interference rises from
$0$ at disjoint supports to $2.0$ at full overlap, while the optimal-merge floor sits
at exactly half (Fig.~\ref{fig:alloc}, right), confirming that overlap creates an
irreducible distortion floor and allocation changes where the cost lands rather than
its total amount. This relocation is fundamental because the distortion floor $\mathcal{D}$ is a lower
bound (Theorem~\ref{thm:merge}) that no single shared model can improve on. On overlapping supports, cost can be moved from irreversible
forgetting of the old task to recoverable deferred fitting of the new one.

\begin{figure}[tb]\centering
\includegraphics[width=\textwidth]{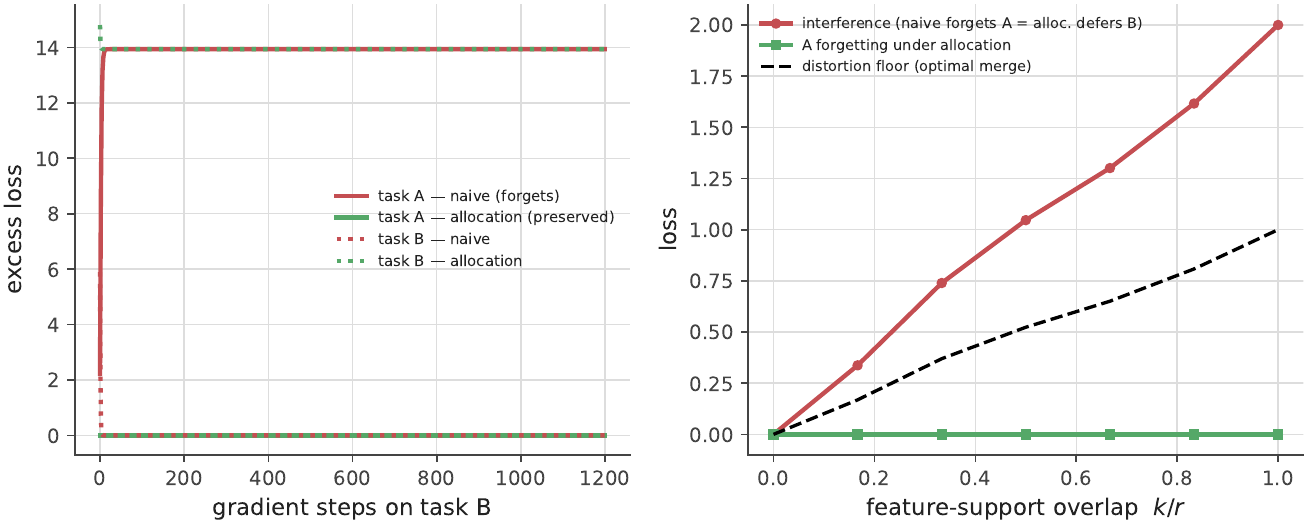}
\caption{Allocation and the distortion floor. \emph{Left:} on partially shared features
with a conflict on the shared directions, naive descent destroys task $A$ (red rises to
$13.94$) while allocation preserves it ($<10^{-17}$, green flat) and converts that same
$13.94$ into deferred under-fitting of $B$---cost is \emph{relocated}, not removed
(the task-$A$ naive and task-$B$ allocation curves coincide at $13.94$).
\emph{Right:} versus feature-support overlap $k/r$, the relocated interference (naive
$A$-forgetting $=$ allocation $B$-residual) grows linearly to $2.0$ while the
optimal-merge floor sits at exactly half ($1.0$); allocation holds $A$-forgetting at
$0$.}
\label{fig:alloc}
\end{figure}

\subsection{Similarity and capacity}
\label{sec:exp:simcap}
To test the signed gate, we compare sharing and orthogonalizing on a shared block as a
function of target similarity $s$ (Fig.~\ref{fig:simcap}, left). Sharing wins by
variance reduction when tasks are similar and loses by bias when they are not, with the
net benefit crossing zero at $\sstar\approx0.26$; above this threshold the gate should
share, below it orthogonalize. With exact, full data the two strategies are an exact
relocation of the same cost, so this sign-change is a finite-data effect
that unconditional projection methods forgo.
Separately, under orthogonal allocation a stream of rank-$2$ tasks exhibits a sharp
capacity knee at $T=d/r$ (Fig.~\ref{fig:simcap}, right) at which earlier tasks remain intact,
but once the occupied rank exceeds capacity, residual distortion turns on where
predicted. Capacity is thus a rate--distortion budget (the rank cap
$k_{\max}$ of Algorithm~1 is its operational form). Trading $d$ against the number of
tasks is an effective-capacity budget, and dropping the least-excited direction bounds
the induced forgetting by its eigenvalue.

\begin{figure}[tb]\centering
\includegraphics[width=0.49\textwidth]{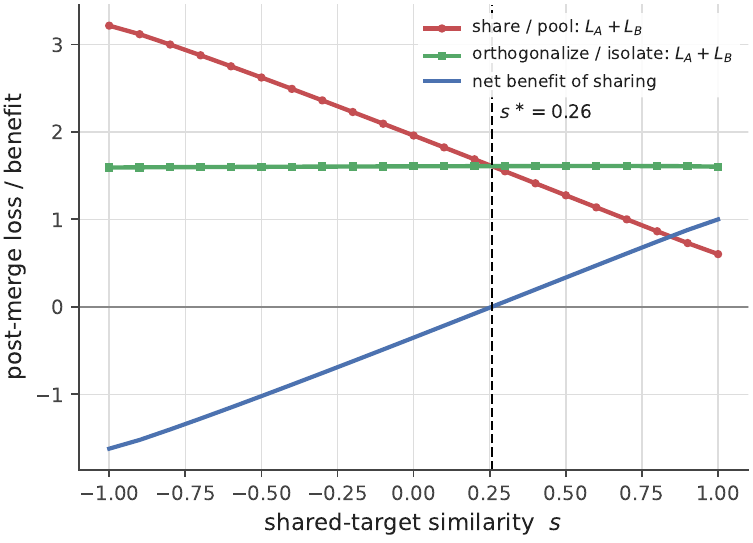}\hfill
\includegraphics[width=0.49\textwidth]{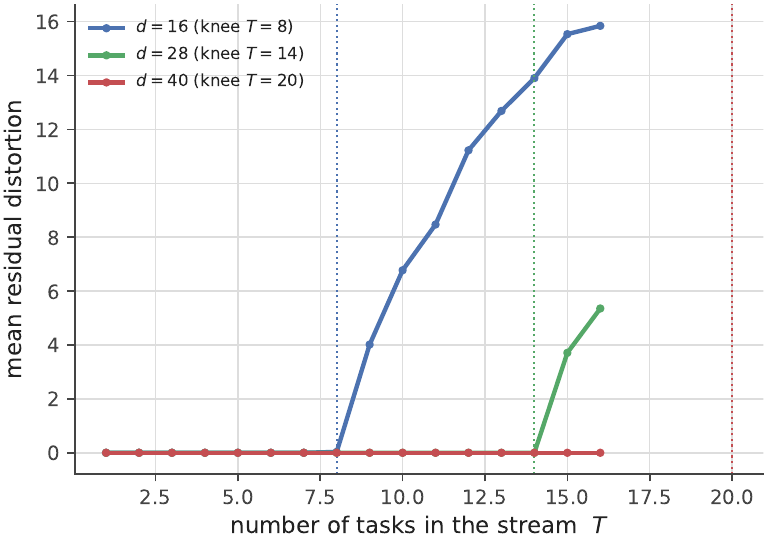}
\caption{Similarity and capacity. \emph{Left:} total population loss of pooling (share)
versus isolating (orthogonalize) two tasks on a shared block, against target similarity
$s$. Isolation is flat in $s$; pooling trades variance reduction against bias and wins
only when tasks are similar, with the net benefit crossing zero at $\sstar\approx0.26$.
\emph{Right:} mean residual distortion over a stream of rank-$2$ tasks under orthogonal
allocation, for three feature dimensions. Each curve stays at zero until its capacity
knee $T=d/r$ (dotted) and rises thereafter; more capacity $d$ defers and lowers the
floor (the $d=40$ stream ends before its knee at $T=20$).}
\label{fig:simcap}
\end{figure}

\subsection{Curvature drift and the segment-averaged correction}
\label{sec:exp:breakdown}
Theorem~\ref{thm:functional} is exact under frozen features, but in jointly trained
networks its frozen-curvature prediction degrades with depth. The shallow network holds
Pearson $r\approx0.93$--$0.95$ up to moderate step sizes, whereas deeper networks are
markedly less predictable, with $r$ falling to $\approx0.2$--$0.4$
(Fig.~\ref{fig:breakdown}, top). Replacing the
initial curvature by the segment-averaged curvature of Theorem~\ref{thm:general} removes
this failure mode: in the same regime, the corrected predictor restores $r=1.00$ at
every tested depth (Fig.~\ref{fig:breakdown}, bottom), estimated from forward
evaluations of $L_A$ along the segment $w_A\!\to\!w_B$ with no explicit Hessian. The
correction from $K=21$ probe evaluations consumes $1.4\%$ of one task's training budget, and $K=5$ already suffices in practice for under
$0.4\%$. Once the stale curvature is corrected, the architecture-general identity holds
exactly.

\begin{figure}[tb]\centering
\includegraphics[width=0.92\textwidth]{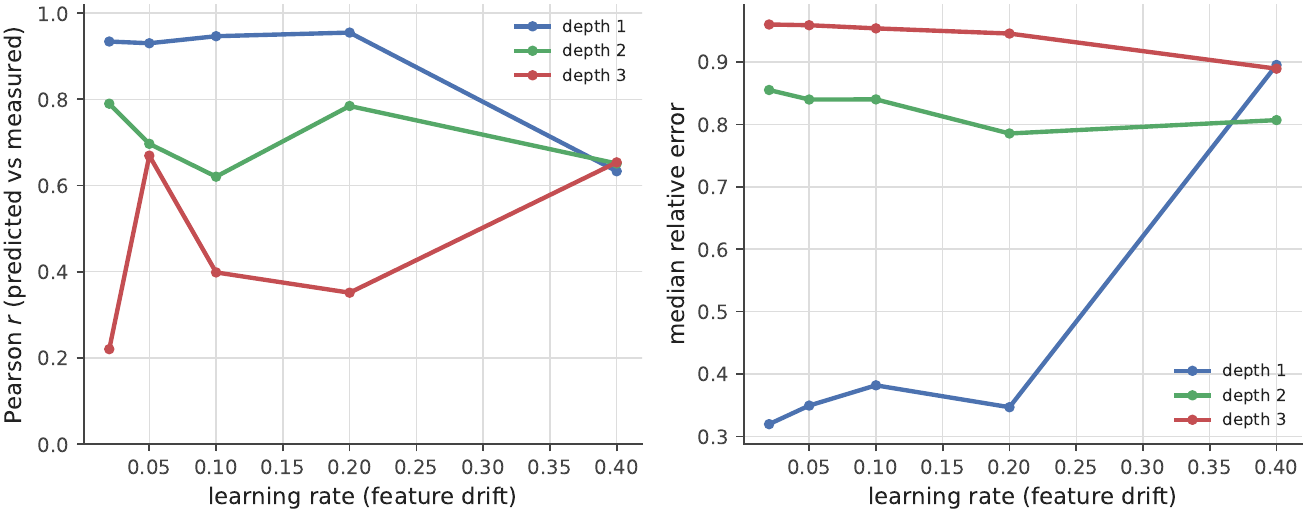}\\[8pt]
\includegraphics[width=0.6\textwidth]{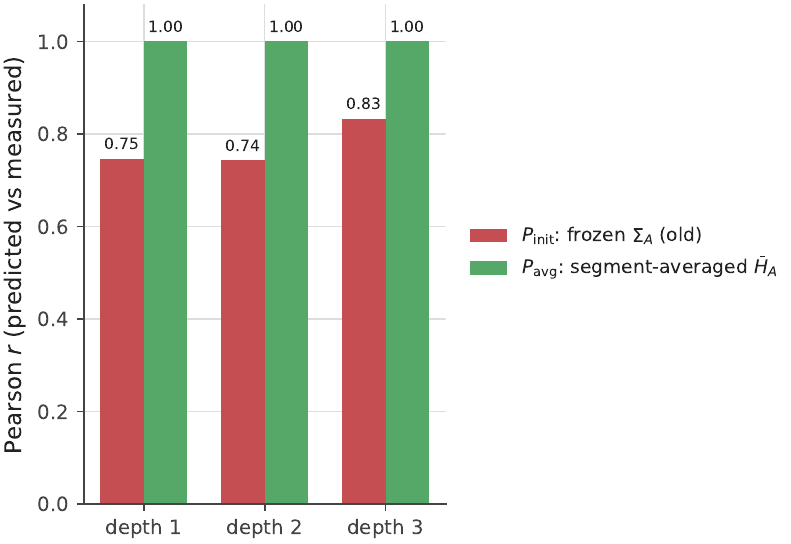}
\caption{Curvature drift and its correction. \emph{Top:} predicted-versus-measured
forgetting fidelity for backprop-trained networks, against learning rate (feature
drift) and depth. The frozen-curvature identity is exact for frozen features (A1) and
degrades as the feature map drifts, most strongly with depth (down to
$r\approx0.2$--$0.4$ for the deeper networks). \emph{Bottom:} at learning rate $0.4$,
replacing $\Sig_A$ by the segment-averaged Hessian $\bar H_A$ of
Theorem~\ref{thm:general} (green) restores $r=1.00$ at every depth, while the
frozen-curvature predictor (red) remains imperfect ($r=0.77$--$0.87$).}
\label{fig:breakdown}
\end{figure}

\subsection{Real data: dissimilar and similar tasks}
\label{sec:exp:real}
We evaluate the functional and \IGFA{} on real benchmarks built on frozen
random-feature backbones, so that Assumption A1 and hence Theorem~\ref{thm:functional}
hold exactly while the tasks remain non-trivial. For \textbf{Split-Digits}, we construct
a task-incremental benchmark from the \texttt{digits} dataset ($1797$ images, $10$
classes), split into five two-class tasks on top of a frozen random ReLU feature map
($d=300$). The functional predicts per-task forgetting across the full task sequence
with Pearson $r=0.999$, on the $y=x$ line, and the feature spectrum shows an effective
per-task rank of about $2$, matching the two-class structure and explaining why a small
protected subspace suffices (Fig.~\ref{fig:diagnostics}). On this sequence, \IGFA{}
attains $0.980\pm0.007$ average accuracy (mean $\pm$ 95\% CI over $10$ seeds) with
near-zero forgetting, matching the strongest replay- and Fisher-free structural baseline
\OGD{}/\textsc{gpm} and exceeding naive training, EWC, and a small replay baseline
(Table~\ref{tab:benchmark}). These digit-pair tasks have low pairwise feature overlap,
all below $\sstar$, so the gate orthogonalizes throughout and \IGFA{} reduces to \OGD{}
\emph{exactly}---the two are identical on every seed (paired $t$-test $p=1$). \IGFA{} maintains parity with the best structural baseline, with no buffer and no Fisher matrix.

\begin{table}[tb]\centering\small
\caption{Split-Digits task-incremental (frozen random-feature backbone, $5$ tasks,
$d=300$). Mean $\pm$ 95\% CI over $10$ seeds; average final accuracy (higher better) and
average forgetting on earlier tasks (lower better). \IGFA{} matches the best structural
baseline---and is identical to \OGD{} on every seed (paired $t$-test $p=1$)---with no
replay buffer and no Fisher matrix.}
\label{tab:benchmark}
\begin{tabular}{lccc}
\toprule
Method & Avg.\ accuracy $\uparrow$ & Avg.\ forgetting $\downarrow$ & Extra state \\
\midrule
Naive            & $0.959\pm0.014$ & $0.040\pm0.017$  & --- \\
EWC \cite{kirkpatrick2017}  & $0.978\pm0.005$ & $0.005\pm0.006$  & Fisher ($\mathcal{O}(d)$) \\
Replay ($10$/task) & $0.977\pm0.009$ & $0.015\pm0.010$ & buffer ($\mathcal{O}(|\mathcal B|d)$) \\
\OGD{}/\textsc{gpm} \cite{farajtabar2020,saha2021} & $\mathbf{0.980\pm0.007}$ & $\mathbf{0.003\pm0.003}$ & subspace ($\mathcal{O}(dk)$) \\
\IGFA{} (ours)   & $\mathbf{0.980\pm0.007}$ & $\mathbf{0.003\pm0.003}$ & subspace ($\mathcal{O}(dk)$) \\
\bottomrule
\end{tabular}
\end{table}

\begin{figure}[tb]\centering
\includegraphics[width=\textwidth]{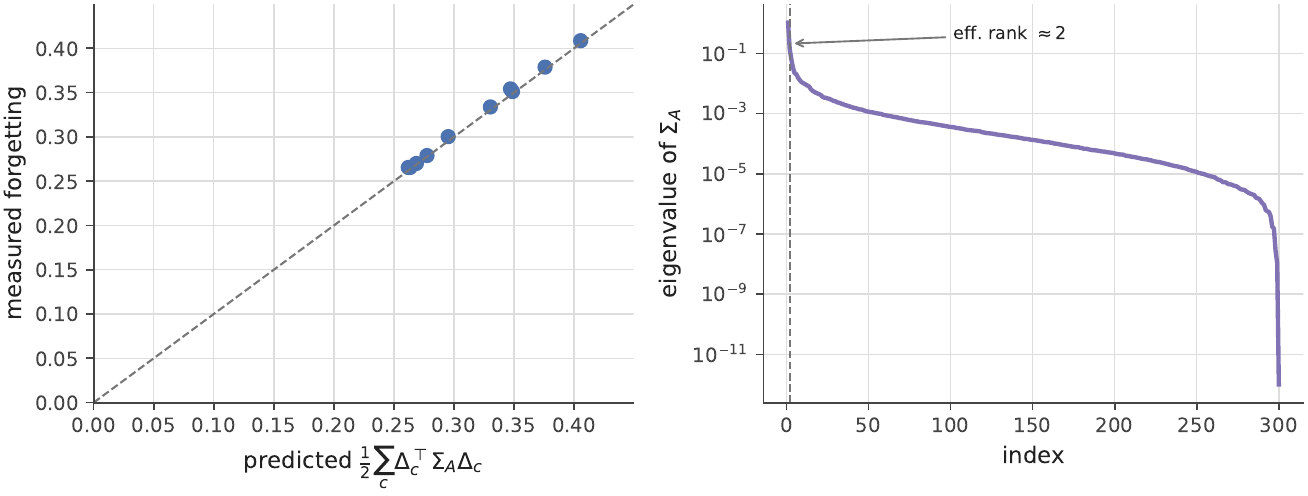}
\caption{Diagnostics on real data (Split-Digits). \emph{Left:} measured versus
predicted forgetting $\tfrac12\sum_c\Del_c^\top\Sig_A\Del_c$ for all (earlier, later)
task pairs---on the $y=x$ line, Pearson $r=0.999$, confirming
Theorem~\ref{thm:functional} on real features. \emph{Right:} the eigenvalue spectrum of
a task's feature second-moment $\Sig_A$; its effective rank ($\approx2$) is the
dimension that must be protected, and the rapid decay is why a small subspace summary
suffices.}
\label{fig:diagnostics}
\end{figure}

To test the sharing branch, we next consider \textbf{Rotated-Digits}, a
domain-incremental benchmark on the same ten classes under five rotations, again with a
frozen random-feature backbone. These tasks are similar, with mean pairwise
feature-subspace overlap $0.68$. In this regime, unconditional
orthogonalization should be too conservative because the tasks share reusable structure
rather than only conflict. Here, \IGFA{} with threshold $\sstar=0.65$ (within the broad
optimum plateau of Sec.~\ref{si:sstar}; the break-even similarity is regime-specific) attains
the highest accuracy among replay- and Fisher-free methods at $0.771\pm0.030$ (mean $\pm$ 95\%
CI over $5$ seeds), \emph{significantly} above \OGD{}'s $0.684\pm0.035$ (paired $t$-test
$p\approx6\times10^{-5}$) and above naive ($0.714\pm0.024$) and EWC ($0.727\pm0.050$),
while keeping forgetting near zero at $0.002\pm0.027$---statistically indistinguishable
from \OGD{}'s ($p=0.44$) (Table~\ref{tab:benchmark2}). Only buffer-based replay attains
higher accuracy ($0.795\pm0.012$, $p=0.03$), at the cost of stored data. \IGFA{} thus
matches the best structural baseline on dissimilar tasks, and its signed gate yields a
\emph{statistically significant} advantage in the similar-task regime for which it was
designed. Section~\ref{sec:exp:cifar} confirms the dissimilar-task reading at scale.

\begin{table}[tb]\centering\small
\caption{Rotated-Digits domain-incremental (frozen backbone, five rotations, mean
subspace overlap $0.68$, threshold $\sstar=0.65$). Mean $\pm$ 95\% CI over $5$ seeds.
\IGFA{} attains the highest accuracy among replay- and Fisher-free methods with near-zero
forgetting---an intermediate between full isolation (\OGD, paired $t$-test
$p\approx6\times10^{-5}$) and full sharing (naive) that improves significantly on both.
Only buffer-based replay attains higher accuracy ($p=0.03$), at the cost of stored data.}
\label{tab:benchmark2}
\begin{tabular}{lccc}
\toprule
Method & Avg.\ accuracy $\uparrow$ & Avg.\ forgetting $\downarrow$ & Extra state \\
\midrule
Naive            & $0.714\pm0.024$ & $0.163\pm0.013$  & --- \\
EWC \cite{kirkpatrick2017}  & $0.727\pm0.050$ & $0.045\pm0.039$  & Fisher \\
\OGD{}/\textsc{gpm} \cite{farajtabar2020,saha2021} & $0.684\pm0.035$ & $-0.007\pm0.004$ & subspace \\
Replay ($10$/class) & $0.795\pm0.012$ & $0.042\pm0.014$ & buffer \\
\IGFA{} (ours)   & $\mathbf{0.771\pm0.030}$ & $0.002\pm0.027$ & subspace \\
\bottomrule
\end{tabular}
\end{table}

\subsection{Scaling checks: frozen-ViT benchmarks over five seeds}
\label{sec:exp:cifar}
We repeat the comparison on three standard benchmarks over a frozen ViT-B/16 backbone with
a single shared linear head, so A1 and Theorem~\ref{thm:functional} hold exactly:
Split-CIFAR-100 ($10$ tasks of $10$ classes), ImageNet-R ($10$ of $20$), and CUB-200
($10$ of $20$), each over five seeds with per-seed class orders (mean $\pm$ 95\% CI;
\texttt{kaggle\_vit\_multiseed.py}; Table~\ref{tab:cifar}). Three observations follow.
\emph{(i) On dissimilar-class streams the gate correctly reduces to \OGD.} On
Split-CIFAR-100 and CUB every pairwise subspace overlap lies below the threshold
($0.31$--$0.54$ vs $\sstar=0.6$), so \IGFA{} protects everything and is identical to \OGD{}
on every seed; the two attain the best retention among buffer- and Fisher-free methods
(forgetting $0.001$ and $-0.004$, vs naive $0.009$ and $0.010$).
\emph{(ii) On ImageNet-R the sharing branch pays off at scale.} There the overlaps straddle
the threshold ($0.51$--$0.65$), the gate shares some past subspaces while protecting
others, and \IGFA{} significantly exceeds \OGD{} ($0.788\pm0.004$ vs $0.781\pm0.005$,
paired $p=0.001$) at equal, zero forgetting---the similar-task advantage of
Sec.~\ref{sec:exp:real}, previously isolated on Rotated-Digits, appears on a standard
benchmark. \emph{(iii) Extra state buys accuracy, not retention.} A tuned Fisher penalty
(EWC) attains the highest raw accuracy on these mild streams at $\mathcal O(d)$ Fisher
state, while feature replay under an unmasked objective \emph{hurts} on two of the three
benchmarks and carries the largest forgetting---stored data is not automatically an
advantage. Forgetting is small for every method on these disjoint-class streams because
the frozen backbone already separates the tasks; the structural claim therefore concerns
retention and state, not raw accuracy.

\paragraph{A similar-task stream at scale.} To test the gate's regime beyond rotations we
build a domain-incremental stream from CIFAR-100's superclass structure: five tasks, each
classifying into the same $20$ coarse labels but drawing on disjoint fine classes (one per
superclass per task), so consecutive tasks share semantics (subspace overlaps
$0.52$--$0.61$). Structural allocation carries this stream: \IGFA{}/\OGD{} reach
$0.786\pm0.010$ versus naive $0.742\pm0.006$ ($p=0.0009$), with forgetting $0.169$ versus
$0.237$; EWC's anchor blocks adaptation entirely ($0.662\pm0.020$), and only replay,
storing data, does better ($0.822\pm0.002$). The gate matches \OGD{} here ($p=0.37$)
because most overlaps sit just below $\sstar$---consistent with the ImageNet-R reading
that the sharing advantage appears where overlaps straddle the threshold.

\paragraph{When the gate pays is a measurable property of the stream.} Across all six real
streams of this paper, the gate's realized advantage over unconditional projection is
organized by one pre-deployment statistic: the stream's \emph{share density}, the fraction
of task pairs whose measured similarity crosses the gate threshold
(Fig.~\ref{fig:density}; \texttt{fig\_stream\_density.py}). On the three zero-density
streams (Split-CIFAR-100, CUB, the AG-News LoRA domains) the gate provably reduces to
\OGD{} and ties exactly; the low-density superclass stream ties within noise; and the two
positive-density streams are precisely the two statistically significant wins, ordered by
density (Rotated-Digits, density $0.54$, $+0.087$; ImageNet-R, density $0.27$, $+0.007$;
Spearman $0.90$ over the six streams). The statistic is computable before training from
the same probes the pre-flight diagnostic already runs, so whether the gate is worth
deploying is itself a pre-flight question rather than a post-hoc finding.

\begin{figure}[tb]\centering
\includegraphics[width=0.62\textwidth]{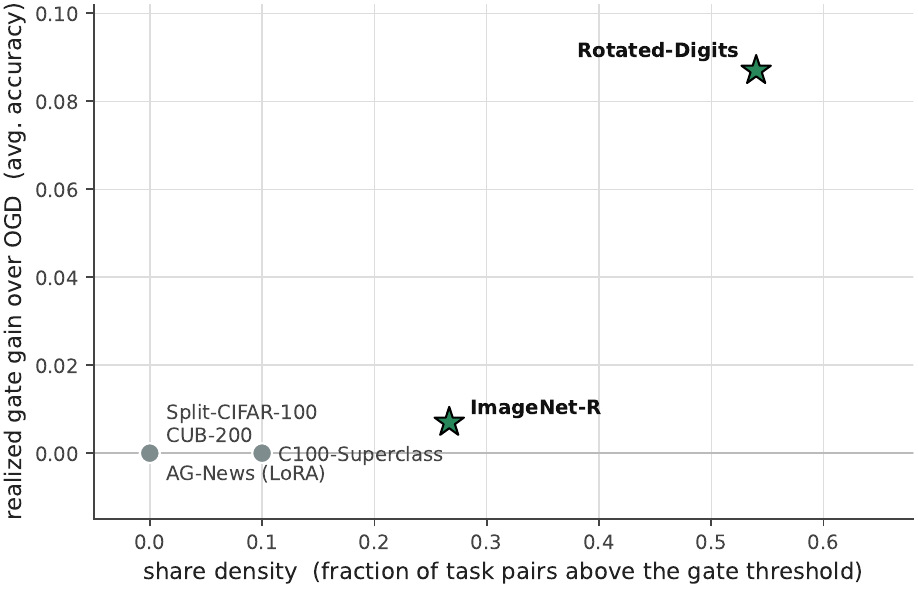}
\caption{The gate's realized gain over \OGD{} against the stream's share density, for all
six real streams in this paper (stars: statistically significant gains; grey: exact or
statistical ties). Zero-density streams tie by construction---the gate has nothing to
decide---and the gain grows with density (Spearman $0.90$). Share density is measurable
before training, making gate-versus-projection a pre-deployment decision.}
\label{fig:density}
\end{figure}

The same frozen backbone also supports a \emph{prompt}-based instantiation beyond the
linear head, and a pair of protocols brackets exactly when gating pays. When all tasks
reuse one shared output head---the high-interference regime of a deployed adapter serving
every task---gating the prompt and head updates lifts accuracy from $0.416$ to $0.733$
(single run, \texttt{kaggle\_pending\_experiments.py}): the structural rule carries over to
prompt space wherever interference is present. When each class instead owns its output row
and logits are task-masked---a protocol that removes head interference by
construction---naive prompt-tuning barely forgets ($0.017\pm0.016$) and the gate is
statistically indistinguishable from it ($0.917\pm0.024$ vs $0.926\pm0.024$, paired
$p=0.20$; five seeds, \texttt{kaggle\_vit\_multiseed.py}). The pair is the functional's own
prediction made visible: the gate's benefit scales with the interference energy actually
present, and vanishes when the protocol allocates disjoint parameters up front.

\paragraph{Relation to prompt-pool methods.} L2P, DualPrompt, and CODA-Prompt
\cite{wang2022l2p,wang2022dualprompt,smith2023coda} also operate on a frozen pretrained ViT backbone.
They circumvent the physical memory overhead and data-retention privacy risks of traditional Buffer-based
methods by routing inputs dynamically into a virtual parameter buffer, the prompt pool. 
However, their retention remains entirely empirical based on
an input-conditioned selection mechanism to map classes without test-time task identities.
Our approach shifts from statistical heuristic to mathematical design.
By utilizing a task-incremental protocol with a single shared linear head, \IGFA{} forces
protected directions to remain first-order invariant by construction
(Corollary~\ref{cor:removability}).
While its Split-CIFAR-100 accuracy sits below highly optimized prompt-pool baselines,
the raw accuracies are fundamentally mismatched across these two distinct evaluation paradigms.
The claim is confined to the buffer- and Fisher-free regime where \IGFA{} attains the strongest known retention.

\begin{table}[tb]\centering\small
\setlength{\tabcolsep}{4pt}
\caption{Frozen-ViT scaling checks, task-incremental with a single shared linear head (A1
exact); mean $\pm$ 95\% CI over five seeds (\texttt{kaggle\_vit\_multiseed.py}). On the
dissimilar-class streams (Split-CIFAR-100, CUB) \IGFA{} reduces to \OGD{} on every seed and
attains the best buffer- and Fisher-free retention; on ImageNet-R, whose subspace overlaps
straddle $\sstar$, the gate significantly exceeds \OGD{} (paired $p=0.001$) at zero
forgetting. EWC carries $\mathcal O(d)$ Fisher state; replay stores $10$ features/class.}
\label{tab:cifar}
\begin{tabular}{lcccccc}
\toprule
 & \multicolumn{2}{c}{Split-CIFAR-100} & \multicolumn{2}{c}{ImageNet-R} & \multicolumn{2}{c}{CUB-200} \\
\cmidrule(lr){2-3}\cmidrule(lr){4-5}\cmidrule(lr){6-7}
Method & acc $\uparrow$ & forget $\downarrow$ & acc $\uparrow$ & forget $\downarrow$ & acc $\uparrow$ & forget $\downarrow$ \\
\midrule
Naive  & $0.945\pm0.010$ & $0.009\pm0.006$ & $0.795\pm0.007$ & $0.014\pm0.008$ & $0.939\pm0.008$ & $0.010\pm0.006$ \\
EWC \cite{kirkpatrick2017} & $0.953\pm0.006$ & $0.000\pm0.000$ & $0.808\pm0.003$ & $0.000\pm0.000$ & $0.941\pm0.011$ & $0.009\pm0.009$ \\
Replay & $0.928\pm0.005$ & $0.025\pm0.004$ & $0.740\pm0.008$ & $0.071\pm0.007$ & $0.941\pm0.003$ & $-0.000\pm0.005$ \\
\OGD{}/\textsc{gpm} \cite{farajtabar2020,saha2021} & $0.933\pm0.006$ & $0.001\pm0.002$ & $0.781\pm0.005$ & $-0.000\pm0.003$ & $0.916\pm0.006$ & $-0.004\pm0.002$ \\
\IGFA{} (ours) & $0.933\pm0.006$ & $0.001\pm0.002$ & $0.788\pm0.004$ & $0.000\pm0.003$ & $0.916\pm0.006$ & $-0.004\pm0.002$ \\
\bottomrule
\end{tabular}
\end{table}

\subsection{Online drift: recursive estimation versus a stale subspace}
\label{sec:exp:drift}
The retention guarantees so far fix the protected subspace at task onset. Under
representation drift---the regime of jointly trained and large models---an unconditional projection fails, and the recursive Gauss--Newton surrogate of
Sec.~\ref{sec:algo} is justified. We stream tasks whose feature geometry rotates each
step while the underlying subspace stays low-rank, and contrast a stale protected subspace
(fixed at onset, as in \OGD{}/\textsc{gpm}) with the recursive estimate. The stale subspace
accumulates each drifted copy as if new, so its rank climbs linearly until it fills the
parameter space ($\mathrm{rank}\!\to\!d=60$) and free plasticity collapses to zero---the
network can no longer learn. The recursive estimate recognizes the drifted copies as one
subspace, holding the effective rank at the true value ($\approx4$) and preserving $93\%$
of capacity throughout (Fig.~\ref{fig:drift}). Re-estimating the geometry is therefore not
a refinement but the difference between exhausting capacity and retaining it under drift;
this is the subspace-projection failure mode reported for LLM merging
\cite{merging2025llm}, and the recursive update removes it. We also compared a
geodesic-aligned tracker (a Grassmann-manifold step in the spirit of GAGP
\cite{geodesic2025}), holding the base optimizer, learning rate, and batch order fixed: it
likewise caps the rank at the true value, but tracks the current subspace at markedly lower
fidelity ($\approx0.5$ versus $\approx1.0$ for the recursive estimate across learning rates
$0.05$--$0.4$), so the path-averaged Gauss--Newton update of Theorem~\ref{thm:general} is the
more faithful deep-network drift tracker.

\begin{figure}[tb]\centering
\includegraphics[width=\textwidth]{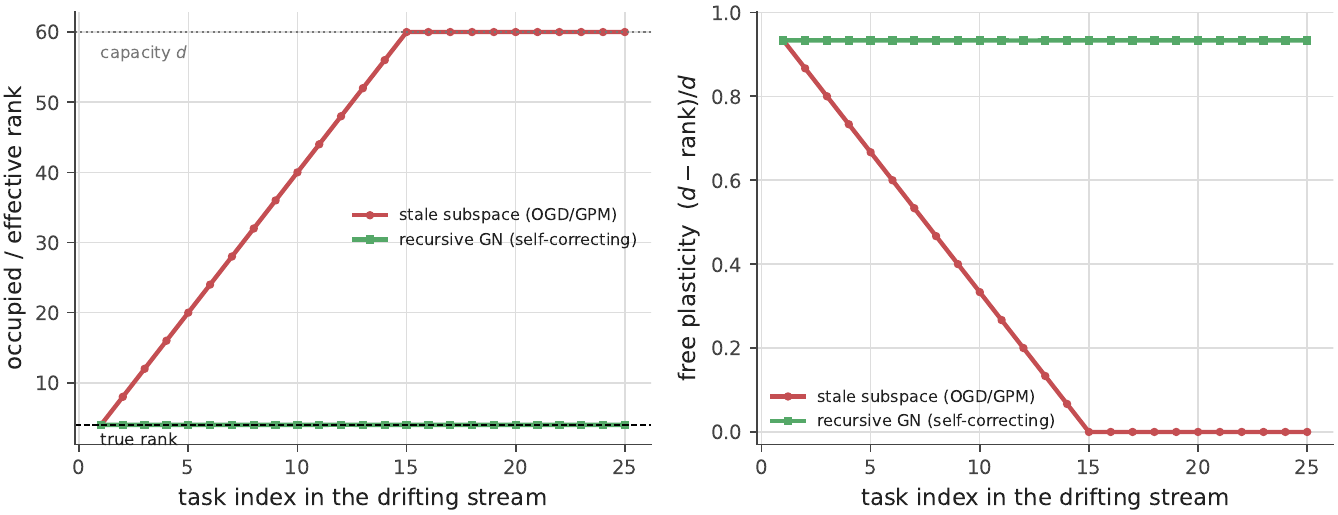}
\caption{Online subspace update under feature drift. \emph{Left:} a stale protected subspace
(\OGD{}/\textsc{gpm}, red) accumulates drifted copies and its rank grows to capacity $d$
(exhaustion); the recursive Gauss--Newton estimate (green) stays at the true rank ($\approx4$).
\emph{Right:} consequently the stale method's free plasticity $(d-\mathrm{rank})/d$ collapses
to $0$ (cannot learn new tasks) while the recursive one retains $\approx0.93$.}
\label{fig:drift}
\end{figure}

\section{Discussion}
\label{sec:discussion}

The identity $\Del L_A=\tfrac12\,\Del^\top\Sig_A\,\Del$ (Eq.~\eqref{eq:functional}) reorganizes the continual-learning
problem around geometry.Components in $\range\Sig_A$ incur forgetting, whereas components in $\ker\Sig_A$ are free.
This perspective determines, before training begins, whether two tasks must interfere at all, what
irreducible floor remains when they do, and which inner product renders model merging
optimal---replacing the Euclidean metric with the $\Sig$-metric. In that
sense, the functional plays the same conceptual role here that the NTK \cite{jacot2018}
played for generalization: it provides a geometry and a set of computable quantities in
place of heuristics.

\paragraph{A map of the claims.} The paper makes three kinds of claims.  Some are
\emph{(I) exact} in frozen-feature or function-space regime (the interference identity, the removability dichotomy, the
$\Sig$-orthogonal optimal merge, the distortion floor, and continuous-time exact
retention, Sec.~\ref{si:beyond}). \emph{Exact-with-quantified-residual}: end-to-end parameter-space
retention, exact up to the basis error of Proposition~\ref{prop:esterror} and the
finite-step remainder of Theorem~\ref{thm:general}, both measured. Some are  \emph{(III)
empirically supported} and remain exact up to measured approximation terms under end-to-end training (the multi-seed benchmark results, Split-Digits, Rotated-Digits,
the frozen-ViT streams, and the threshold-free gate's real-text gain over naive
fine-tuning). \emph{(IV) Mechanism probes and
conjectures.} The generality extensions (federated merging, out-of-distribution and
task-free signals, physics-informed constraints, Fisher-metric policy gradients,
meta-learning, memory-and-control instances) are exact-regime proof-of-concepts
labelled in the Supporting Information as \emph{proved-in-A1}, \emph{empirical mechanism
probe}, or \emph{forward-looking conjecture}.

 and some are empirical demonstrations of the mechanism at benchmark and language-model scale. Keeping these tiers separate clarifies both what is proved and what is shown experimentall

\paragraph{Forgetting decomposes into incompatibility, capacity, and control.} The
results above assemble into an attribution identity. Writing $\mathcal{D}^\star_\infty$
for the floor of any single predictor (Theorem~\ref{thm:mi}; under cross-entropy the
conditional mutual information), $\mathcal{D}^\star_{\mathrm{class}}$ for the optimum of
the \emph{deployed} class (width, rank, architecture), and $L$ for what training
achieved, every observed forgetting level splits exactly, term by non-negative term,
into $L=\mathcal{D}^\star_\infty+[\mathcal{D}^\star_{\mathrm{class}}-\mathcal{D}^\star_\infty]
+[L-\mathcal{D}^\star_{\mathrm{class}}]$: \emph{incompatibility}, a property of the task
family that no gate, replay, or width removes. \emph{Capacity} is a property of the model
class that only widening or routing removes and \emph{control} a property of the
training policy that only a better gate, replay schedule, or ordering removes. In the
exact regime each intervention moves only its own term. Widening eliminates the
capacity term while no policy touches it and the active-set gate cuts the control term
while the floor is invariant throughout (Sec.~\ref{si:beyond}). The practical
consequence is a diagnosis rule that precedes every method choice in this paper's
taxonomy: measure which term dominates (floor from the data, class optimum from the
deployed geometry, control as the remainder) \emph{before} deciding on a mechanism, because each
family of continual-learning methods can repair exactly one of the three. The
attribution is measurable end to end on real continual pretraining: on the four-domain
language stream, a certified floor bound of $0.005$ nats/token against observed
forgetting of $0.16$--$0.40$ establishes that at least $97\%$ of the forgetting was
avoidable in principle, with the control term dominating or co-dominating at both
scales and the active-set gate removing a third to a half of it
(Sec.~\ref{si:beyond}). A further
split follows from the identity's anchor: evaluated on the deployment rather than the
training distribution, forgetting separates into knowledge loss and memorization loss,
and the two can carry opposite signs. An interfering update can raise train loss while
\emph{improving} the deployed function, so part of what benchmarks report as
catastrophic forgetting is implicit unlearning of noise. Protection must therefore
anchor on validated data: a gate fed training gradients defends memorized noise as if
it were knowledge, and is dominated on both axes by its held-out-anchored counterpart
(Sec.~\ref{si:beyond}).

The central assumption is A1. For frozen backbones and PEFT, the regime most
practical deployments now occupy, it is exact and the results stand without
approximation, as the real-data identity at $r=0.999$ confirms. For full networks it is
first-order, and the paper measures how that approximation degrades: the
prediction remains accurate for shallow, small-step training and falls to
$r\approx0.2$--$0.4$ for deeper networks (identity becomes first-order, Sec.~\ref{sec:exp:breakdown}).
Allocation does not eliminate the floor, but only redistributes its cost.
Task ordering does not reduce the floor either because the floor is a property of the joint support
overlap. It is essentially invariant to task order (a $\Sig$-overlap curriculum did not lower it
in our tests, Sec.~\ref{si:extensions}), which reinforces that allocation only relocates cost.
Against unconstrained replay, which can reduce the floor by re-supplying $\Sig_A$, the
claim is therefore parity under no-buffer or
no-Fisher constraints together with structural guarantees and preserved transfer where
similarity permits. On the benchmark, \IGFA{} matches the best structural baseline because
the tasks are dissimilar; its distinctive sharing gain requires task similarity, isolated
on Rotated-Digits and confirmed at scale on ImageNet-R, where the gate significantly
exceeds unconditional projection at equal forgetting (Sec.~\ref{sec:exp:cifar}).

The geometry also yields an operational decision procedure. Before training, a linear
task-probe on the frozen features estimates task separability (Sec.~\ref{si:feasibility});
pairwise subspace overlap and target similarity then classify the stream. Disjoint supports
make structural allocation lossless, so neither replay nor a Fisher penalty is needed.
Overlapping supports with aligned targets favor sharing via gate. Overlapping
supports with conflicting targets carry an irreducible floor, which poses the question
of where to assign its cost: on the old task as forgetting, on the new task as deferred fit, or
on the memory budget as stored data for replay. The capacity budget adds the final
constraint: once the cumulative occupied rank approaches the feature dimension, retention
must be paid for with added capacity, replay, or graceful overwriting. Each branch of this
procedure is measurable from features and subspaces before or during training, rather than
diagnosed after a failure.

A recent broad evaluation reports that subspace-projection merging works well on small
classifiers but degrades on large language models \cite{merging2025llm}.
On small classifiers the feature geometry is nearly
static, so a subspace fixed once remains a good metric and projection succeeds, which is
precisely the frozen regime where Theorem~\ref{thm:functional} is exact. Large
language models, by contrast, have highly entangled representations that drift as
parameters move, so projection against a stale geometry accumulates misaligned
constraints, exhausts capacity, and eventually damages the representation itself. A
natural fix follows from Theorem~\ref{thm:general}: the relevant metric is the curvature
along the trajectory, and its online surrogate is the recursive Gauss--Newton estimate,
which continually realigns the protected subspace with the current geometry and keeps its
effective rank bounded---the capacity-exhaustion mechanism shown directly in
Sec.~\ref{sec:exp:drift}. Preliminary language-model results suggest that the retention
mechanism transfers, while the right similarity signal for language remains open
(Sec.~\ref{si:llm}).

The main open design question is therefore not the structural retention mechanism, which
transfers across modalities, but a similarity signal for language that can recover
transfer on similar text tasks without relaxing protection where conflict
dominates. Two resolutions are developed in the Supporting Information: a probe-based
similarity estimate validated in the exact regime (Sec.~\ref{si:llm}), and---the
theorem-backed answer---the threshold-free functional gate of Sec.~\ref{si:beyond}, which
reads the measured interference rate instead of any modality-specific similarity signal,
carries a continuous-time monotone-retention guarantee, and delivers a significant,
threshold-free real-text gain over naive continual fine-tuning on a four-domain LoRA
stream, with parity-to-slight-edge over unconditional protection (Sec.~\ref{si:llm}). Several extensions are detailed in
the Supporting Information. First, we show sequential LoRA fine-tuning on language
models at scale (Sec.~\ref{si:llm}), using recursive Gauss--Newton tracking on the adapting
features together with a modality-appropriate similarity gate for language. Second, the
frozen-backbone evaluation on Split-CIFAR-100, ImageNet-R, and CUB appears in the main text
over five seeds (Sec.~\ref{sec:exp:cifar}), comparing the retention--transfer tradeoff
against \OGD{}/\textsc{gpm}, EWC, and replay in the exact A1 regime. Third, we present
overlay-native continual pretraining at the billion-parameter scale (Sec.~\ref{si:pretrain}),
where replay-free and Fisher-free operations are most valuable and where the capacity budget
of Sec.~\ref{sec:exp:simcap} predicts the model-size-versus-rehearsal tradeoff directly.
Finally, we introduce a pre-flight feasibility diagnostic (Sec.~\ref{si:feasibility}) that
estimates task separability from features before training, attributing any remaining gap to
drift, bias, or inference rather than to intrinsic overlap.

Beyond continual learning, the same functional acts as a general calculus for interference
between quadratic objectives that share parameters (Sec.~\ref{si:extensions}). Federated
aggregation is the merge problem of Theorem~\ref{thm:merge} with clients as tasks; the
occupied subspace supplies an out-of-distribution signal and a task-free boundary detector;
the physics constraint of a physics-informed network is a second task whose violation the
identity predicts exactly; the Fisher information plays the role of $\Sig$ for policy
gradients in reinforcement learning; and a feature map can be meta-trained to shrink the
distortion floor itself. These are directional demonstrations rather than benchmarks, but
they indicate that the geometry is a property of shared parameterizations, not of the
continual-learning protocol. Read as a control problem, the calculus also subsumes the
existing method families rather than competing with them: unconditional projection
(\OGD{}/\textsc{gpm}), Fisher penalties (EWC), replay, per-task routing, and dynamic
expansion are each the geometry-driven controller clamped to a single fixed
action---always project, always penalize, always rehearse, always separate---while the
signed gate is the regime-adaptive interior that reads the ledger and spends memory or
rank only where the floor is positive. In a controlled comparison across regimes the
adaptive rule is the Pareto point, strictly dominating every zero-resource fixed method
and matching replay's retention at a fraction of its memory (runnable in
\texttt{demo\_method\_taxonomy.py}, with the metric set and decision rules released as a
standalone interference-ledger toolkit).

\paragraph{A testable bridge to the verbalizable workspace.} The convergence with the
Jacobian lens (Sec.~\ref{sec:related}) yields a falsifiable prediction rather than only
an analogy. If our $\range\Sig_A$ and the lens's J-space are the same functional
subspace, then forgetting must be \emph{verbalizable-subspace-mediated}: the readout
change $\Del f$ that a later task induces on task $A$ should carry its interference
energy inside $A$'s J-space, and protecting that subspace should prevent forgetting as
effectively as protecting $\Sig_A$'s own eigen-subspace. We sketch the experiment
(\texttt{kaggle\_jspace\_interference.py}) on a frozen \texttt{pythia-410m}: fit or load
the released lens to obtain the J-space projector $P_J$ at a readout layer, learn the
four-domain stream with a low-rank head, and (i) decompose the measured forgetting of
domain $A$ into its $P_J$ and orthogonal parts, and (ii) compare naive learning against
gating the update off $P_J$, off $\Sig_A$'s top subspace, and off a random subspace of
equal rank. The interference geometry predicts a J-space energy fraction far above the
$r/d$ random baseline in (i), and \emph{protect-J-space} $\approx$ \emph{protect-}$\Sig$
$\ll$ \emph{naive} with \emph{protect-random} $\approx$ \emph{naive} in (ii)---a
specificity control that would establish the two averaged-Jacobian subspaces as one.
Measuring both on a frozen \texttt{pythia-410m} resolves the relationship into a
\emph{component} rather than an identity (\texttt{kaggle\_jspace\_layers.py}). The strong
form fails: the interference $\Sig_A=\mathbb{E}[gg^\top]$ is high-effective-rank at every
layer---$80\%$ of its energy needs at least $285$ of $1024$ dimensions---so a
workspace-sized subspace cannot carry the bulk of forgetting, which is also why no
low-rank projection protects in (ii), $\Sig_A$'s own top subspace included. But the
enrichment is real and, tellingly, grows toward the readout: the J-space carries $1.3$ to
$2.5\times$ its dimensional share of interference energy across layers $2$ to $22$
(enriched at every layer probed; correlation with depth $0.86$), and its overlap with
$\Sig_A$'s leading subspace rises from $0.05$ to $0.15$ against a $0.03$ chance baseline.
The verbalizable directions are thus systematically over-represented in the forgetting
geometry---increasingly so where representations become more readable---without being
identical to it, a measured link between what a model can say and what it can lose.

Several algorithmic refinements---a self-calibrating threshold, low-cost $\mathcal{O}(rd)$
online curvature via Frequent Directions
\cite{liberty2013}, a per-direction gate, and the soft-gating capacity coordinate
$r_{\rm eff}=\sum_j(1-\mathrm{keep}_j)$---are developed and validated in the Supporting
Information. The functional further extends to control and memory: an interference-budgeted
trust region that dominates step-size control, consolidation with a priced forgetting bound,
floor-scheduled replay, and generative replay from the subspace state itself
(Sec.~\ref{si:memory}). Together with low-rank or diagonal Gauss--Newton approximations of
$\bar H_A$, they suggest a practical route for carrying the architecture-general identity
into fully fine-tuned deep networks at controlled cost.

A further assumption deserves explicit statement: as presented, \IGFA{} assumes task
boundaries are known and switches are abrupt at training time. Section~\ref{sec:algo} and
Algorithm~1 use it most directly---the method processes ``each task $n$'' and forms a basis
$B_n$ from that task's estimated second moment. The assumption is operational bookkeeping,
not theory: the interference functional itself is boundary-free, and the machinery that
removes the bookkeeping already appears in this paper. The recursive Gauss--Newton tracker
replaces per-task second moments with a decayed running estimate, the subspace-velocity
detector recovers unknown boundaries at recall and precision $1.00$ on clean streams, and
their fully autonomous composition matches the oracle given the true boundaries
(Sec.~\ref{si:extensions}). Gradual transitions are the harder case;
there the boundary concept itself blurs, and the drift experiments of
Sec.~\ref{sec:exp:drift} show that the running estimate simply tracks the moving geometry
without requiring a boundary at all.

Several limitations bound these claims. Exactness of the parameter-space results is
confined to A1; the function-space formulation extends the identity, the dichotomy, the
floor, and the continuous-time retention guarantee beyond it (Sec.~\ref{si:beyond}),
leaving finite step size and metric estimation as the remaining---quantified---gaps, and
end-to-end \emph{parameter-space} predictions stay first-order, with a measured accuracy
envelope and the segment-averaged correction. The
frozen-ViT benchmarks and the merge ablation are five-seed, and the functional-gate
language result is now multi-seed at two scales (GPT-2, five seeds;
\texttt{pythia-410m}, three seeds), but the remaining language-model experiments are
single runs and the extension studies are synthetic proofs-of-concept; multi-seed
replication at billion-parameter scale is outstanding. The
similarity gate reads input geometry only in its threshold form; the threshold-free
functional gate resolves the language case in principle (Sec.~\ref{si:beyond}) but its
margins at scale are modest.

\begin{center}\small
\fbox{\begin{minipage}{0.95\columnwidth}
\textbf{What breaks the exact identities, and how.} \emph{Label noise / imperfect task
optima:} a non-zero residual $r_A\neq0$ adds the cross term $\mathbb{E}[\Del\!f\,r_A]$ of
Eq.~\eqref{eq:funcspace-main}; forgetting is then predicted by the \emph{full} identity,
not the energy alone, and mis-specification enters as a task-constant under A2 (the
$\Sig$-metric ranking is preserved). \emph{Early stopping / finite steps:} the update is
not the exact minimizer, leaving the $O(\eta)$ remainder of Theorem~\ref{thm:general} and
the diffusion term of Proposition~\ref{prop:esterror}---both measured, both shrinking with
step size and basis accuracy. \emph{Class imbalance / shifted input density:} $\Sig_A$ is
re-weighted by the deployment density, so the protected geometry must be estimated on the
distribution that matters, not the training mixture; a stale or mis-weighted $\Sig_A$
shows up as basis error $\epsilon$ and is caught by the drift-velocity monitor. In each
case the failure is a \emph{named, measurable} quantity in the identities, not an
unmodelled gap.
\end{minipage}}
\end{center}

\section{Conclusion}
\label{sec:conclusion}
Forgetting is interference, represented here as geometry. In the frozen-feature regime,
forgetting equals $\tfrac12\Del^\top\Sig_A\Del$ and is removable if and only if task
supports are disjoint; when supports overlap, an irreducible distortion floor appears and
matches the task-inference obstruction, and optimal merging is mutual
$\Sig$-orthogonalization, with continual learning as its online counterpart. The result is
a geometric language and a set of concrete levers---a forgetting functional, a removability
test, an optimal-merge identity, a similarity sign-change, and a capacity budget---realized
by a replay-free, Fisher-free allocation rule that carries only a low-rank subspace. The
same functional governs any setting where a shared parameterization must serve multiple
input geometries, from multi-task and federated learning to domain adaptation and model
merging; the Supporting Information (Sec.~\ref{si:extensions}) gives exact-A1 proofs-of-concept
for federated $\Sig$-merging, an intrinsic out-of-distribution signal, task-free boundary
detection, multimodal block interference, physics-informed constraints, Fisher-metric policy
gradients, and floor-minimizing meta-learning. In each setting, forgetting is controlled
geometrically: interference is not corrected after it occurs but allocated in advance, into
directions where it induces no loss.

\section*{Author contributions}
J.S. conceived the framework, derived the theory, implemented the experiments, and wrote the
manuscript.

\section*{Acknowledgments}
The author thanks colleagues at VARTA and TU Braunschweig for discussions on continual learning
in industrial process models.

\section*{Reproducibility statement}
\sloppy
All results are exact-regime simulations with fixed seeds. \texttt{experiments.py} regenerates
the core result figures and numerical claims (interference identity, lossless/relocation,
sign-change $\sstar$, capacity floor); \texttt{verify\_mi\_identity.py} the
information--estimation verification (Fig.~\ref{fig:mi}); \texttt{experiments\_rev.py} the
breakdown, breakdown-fix, and digit benchmarks; \texttt{variance\_report.py} the multi-seed
mean\,$\pm$\,CI and paired-significance results of Tables~\ref{tab:benchmark}--\ref{tab:benchmark2};
\texttt{experiments\_compare.py} the nine-method comparison and offline-merge residual-$D$
ablation of Fig.~\ref{fig:teaser} and Table~\ref{tab:compare};
\texttt{evaluate\_hypotheses.py} four generality extensions of Sec.~\ref{si:extensions}
(federated $\Sig$-merge, OOD signal, task-free boundaries, multimodal blocks) and
\texttt{evaluate\_applications.py} three more (physics-informed constraints, Fisher-metric
policy gradients, floor-minimizing meta-learning);
\texttt{evaluate\_next\_steps.py} the sketched-drift, trust-region, anchor-protection,
probe-similarity, consolidation, floor-scheduled-memory, dream-replay, cascade, and
world-model studies of Secs.~\ref{si:cheapgn}, \ref{si:llm}, and~\ref{si:memory};
\texttt{kaggle\_probe\_gate\_llm.py} the probe-gate protocol on real text (results in
\texttt{probe\_gate\_results.json});
\texttt{evaluate\_open\_questions\_new.py} the open-question battery, with each result
classified as exact-scope support, mechanism probe, or appendix scoping (only exact-scope
results are quoted in this manuscript), \texttt{evaluate\_oq\_hardened.py} an
end-to-end drift-and-noise probe on a gradient-trained network,
\texttt{evaluate\_unlearning\_real.py} the real-data unlearning study,
\texttt{evaluate\_upgrades.py} the noise-law, rate-controller, allocation, privacy, and
damping upgrades, \texttt{evaluate\_beyond\_a1.py} the function-space validations of
Sec.~\ref{si:beyond}, \texttt{evaluate\_functional\_gate.py} the threshold-free
functional gate (with \texttt{kaggle\_functional\_gate\_llm.py} its real-text protocol),
and \texttt{fig\_stream\_density.py} the share-density figure;
\texttt{kaggle\_interference\_optimizer.py} (the \textsc{InterferenceSGD} benchmark against
tuned learning-rate schedules) and \texttt{kaggle\_scale\_addons.py} (per-layer text
similarity, real-feature $\Sig_{\rm eval}$, and the pythia-160m capacity ledger) extend
these to scale;
\texttt{experiments\_drift.py} the $\sstar$ ablation, graceful-degradation, and drift figures;
\texttt{experiments\_extra.py} the self-calibrating-$\sstar$ and tight multiclass-constant
results; and \texttt{experiments\_roadmap.py} the feasibility-diagnostic, Frequent-Directions,
soft-knee, and per-direction figures. Self-contained Colab/Kaggle notebooks accompany the paper:
\texttt{colab\_split\_cifar100\_igfa.py} (Split-CIFAR-100 with a frozen ViT-B/16),
\texttt{colab\_lora\_llm\_igfa.py} and \texttt{colab\_lora\_llm\_igfa\_3approaches.py}
(sequential LoRA on a transformer with the four gate designs), and
\texttt{colab\_continual\_pretrain\_igfa.py} (boundary-free, replay-free,
Frequent-Directions continual pretraining on \texttt{pythia-1b}); and
\texttt{kaggle\_pending\_experiments.py}, which runs the real-backbone checks quoted above
(frozen-ViT Split-CIFAR-100 allocation and merge ablation, the drift/geodesic sweep,
prompt-\IGFA, and the FD-vs-full-vs-Oja LoRA-LM comparison), \texttt{kaggle\_H2\_H3\_H12.py}
(the layer-wise null-space-routing pipeline on a frozen transformer), and
\texttt{kaggle\_vit\_multiseed.py} (the five-seed frozen-ViT protocol: Split-CIFAR-100,
ImageNet-R, CUB, the similar-task CIFAR-100 superclass stream, the merge ablation, and
prompt-\IGFA, on two T4 GPUs). Code, figures, and logs are at
\url{https://github.com/j-stoerk/geometry-of-forgetting}.

\clearpage
\appendix
\renewcommand{\thesection}{S\arabic{section}}
\setcounter{section}{0}
\setcounter{figure}{0}\renewcommand{\thefigure}{S\arabic{figure}}
\setcounter{table}{0}\renewcommand{\thetable}{S\arabic{table}}

\section*{\centering Supporting Information}
\addcontentsline{toc}{section}{Supporting Information}
\medskip

\noindent This Supporting Information collects the experiments and derivations omitted from
the main text for length: deferred proofs (Sec.~\ref{si:proofs}), additional theory
verifications (Sec.~\ref{si:mi}), theory details
(Sec.~\ref{si:theory}), the full method comparison behind Fig.~\ref{fig:teaser}
(Sec.~\ref{si:compare}), benchmark and scale experiments
(Secs.~\ref{si:llm}--\ref{si:pretrain}), the algorithmic refinements referenced in the
Discussion, extensions of the framework to federated, out-of-distribution, task-free, and
multimodal settings (Sec.~\ref{si:extensions}), memory-systems and control extensions
(Sec.~\ref{si:memory}), and the function-space formalism that extends the exact results
beyond frozen features (Sec.~\ref{si:beyond}). All results are exact-regime simulations or
runnable notebooks with fixed seeds.

\section{Deferred proofs}
\label{si:proofs}
For completeness we give the derivations abbreviated to a sentence of intuition in the main
text; the proofs of Theorem~\ref{thm:functional} (a one-line substitution),
Corollary~\ref{cor:removability} and Theorem~\ref{thm:merge} appear in the main text and are
not repeated.

\begin{proof}[Proof of Theorem~\ref{thm:general} (architecture-general identity)]
Let $\varphi(s)=L_A(w_A+s\Del)$. Taylor's theorem with integral remainder gives
$\Del L_A=\varphi(1)-\varphi(0)=\varphi'(0)+\int_0^1(1-t)\varphi''(t)\,dt$, with
$\varphi'(0)=\nabla L_A(w_A)^\top\Del$ and
$\varphi''(t)=\Del^\top\nabla^2 L_A(w_A+t\Del)\Del$, which yields Eq.~\eqref{eq:general}. When
$\nabla^2 L_A\equiv\Sig_A$ the remainder integrates to $\tfrac12\Del^\top\Sig_A\Del$, recovering
Theorem~\ref{thm:functional}.
\end{proof}

\begin{proof}[Proof of Theorem~\ref{thm:floor} (distortion-floor lower bound)]
Restricting the two-task objective Eq.~\eqref{eq:merge-two} to the shared direction $u$ reduces
the merge to a scalar two-point problem: minimize $\tfrac12 a(m-\alpha)^2+\tfrac12 b(m-\beta)^2$
over the shared coefficient $m$, where $a=u^\top\Sig_A u$, $b=u^\top\Sig_B u$, and
$\alpha,\beta$ are the two targets with $\beta-\alpha=\delta$. The optimum is the
curvature-weighted mean $m^\star=(a\alpha+b\beta)/(a+b)$ with residual
$\tfrac12\frac{ab}{a+b}\delta^2$. Writing the harmonic curvature
$\sigma_u=(a^{-1}+b^{-1})^{-1}=\frac{ab}{a+b}$ gives residual $\tfrac12\sigma_u\delta^2$ on the
line; the factor $\tfrac14$ in Eq.~\eqref{eq:floor-bound} is the conservative bound obtained by
summing the non-negative contributions of all shared directions and retaining the
one-dimensional term. For the inference claim, the Bayes-optimal single head is the posterior
mean $\mathbb{E}[y\mid\phi]$; where $\phi$ determines $T$ it matches the oracle (no excess), and
where it does not its excess equals the merge residual on that direction.
\end{proof}

\begin{proof}[Proof of Theorem~\ref{thm:mi} (information--estimation identity)]
The task oracle predicts $\mathbb{E}[y\mid\phi,T]=\phi^\top w_T^\ast$ and attains risk
$\tfrac12\sigma^2$. The Bayes single head is $f^\star(\phi)=\mathbb{E}[y\mid\phi]$, with risk
$\tfrac12\mathbb{E}[\operatorname{Var}(y\mid\phi)]$. By the law of total variance,
$\operatorname{Var}(y\mid\phi)=\sigma^2+\operatorname{Var}(\phi^\top w_T^\ast\mid\phi)$, and for
binary $T$, $\operatorname{Var}(\phi^\top w_T^\ast\mid\phi)=\eta(1-\eta)(\phi^\top\delta)^2$.
Subtracting the oracle risk gives Eq.~\eqref{eq:mi-identity}. Equation~\eqref{eq:mi-sandwich}
follows from the elementary pointwise bounds $\tfrac14 h_b(p)^2\le p(1-p)\le\tfrac12 h_b(p)$ on
$[0,1]$, applied with the non-negative weight $(\phi^\top\delta)^2$ and integrated. The $K$-task
extension Eq.~\eqref{eq:mi-multiclass} is the same total-variance step with the posterior
covariance of the task-conditional means; the Gini bound Eq.~\eqref{eq:multiclass-sandwich}
replaces $p(1-p)$ by $G(p)=1-\|p\|_2^2$.
\end{proof}

\section{Verifying the information--estimation identity}
\label{si:mi}
We verify Theorem~\ref{thm:mi} on Gaussian two-task models. The exact identity
Eq.~\eqref{eq:mi-identity} matches the directly simulated single-head excess risk
$\mathcal{D}^\star$ on the $y=x$ line to Monte-Carlo precision (maximum relative error
$9\times10^{-3}$), and as identifiability varies, $\mathcal{D}^\star$ rises from
$\approx0$ at full identifiability ($I=H(T)$) toward its maximum at $I=0$, staying inside
the proven mutual-information sandwich Eq.~\eqref{eq:mi-sandwich} throughout
(Fig.~\ref{fig:mi}).

\paragraph{Multiclass extension.} The identity extends to $K$ tasks with arbitrary
priors,
\begin{equation}
\mathcal{D}^\star_{K}=\tfrac12\,\mathbb{E}_\phi\!\Big[\operatorname{Var}_{t\sim
p(\cdot\mid\phi)}\!\big(\phi^\top w_t^\ast\big)\Big]
=\tfrac12\,\mathbb{E}_\phi\!\Big[\textstyle\sum_t p_t(\phi)\,(\phi^\top w_t^\ast)^2
-\big(\sum_t p_t(\phi)\,\phi^\top w_t^\ast\big)^2\Big],
\label{eq:mi-multiclass}
\end{equation}
with corresponding multiclass entropy bounds in terms of the Gini impurity
$G(p)=1-\|p\|_2^2$:
\begin{equation}
\tfrac{1-1/K}{(\ln K)^2}\;H(p)^2\ \le\ G(p)\ \le\ \tfrac{1}{2\ln 2}\;H(p),
\label{eq:multiclass-sandwich}
\end{equation}
where the upper constant $1/(2\ln2)$ is $K$-independent. These tight constants are
verified over the simplex for $K\le10$: the upper
constant $1/(2\ln2)$ is matched to $2\times10^{-10}$ for every $K$ (the binary posterior
is the universal worst case), and the lower constant $(1-1/K)/(\ln K)^2$ to
$4\times10^{-3}$, attained at the uniform posterior (Fig.~\ref{fig:multiclass}).

\begin{figure}[htb]\centering
\includegraphics[width=\textwidth]{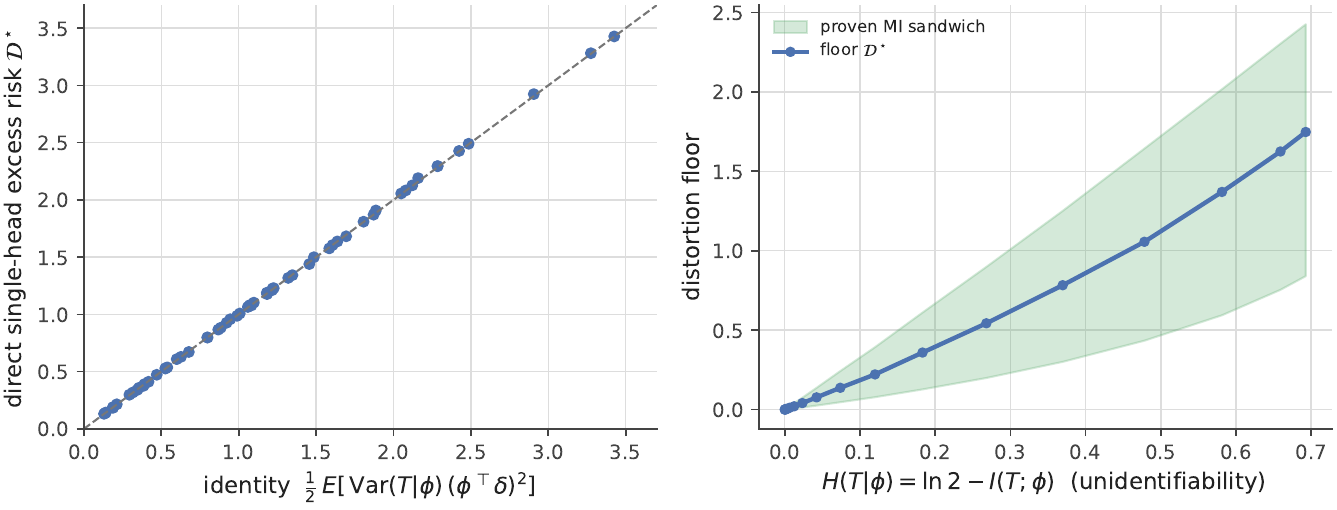}
\caption{Floor / mutual-information verification. The exact identity
Eq.~\eqref{eq:mi-identity} against the simulated single-head excess risk
$\mathcal{D}^\star$, on the $y=x$ line, with the floor rising as identifiability falls and
remaining inside the mutual-information sandwich (band) throughout.}
\label{fig:mi}
\end{figure}
\begin{figure}[htb]\centering
\includegraphics[width=\textwidth]{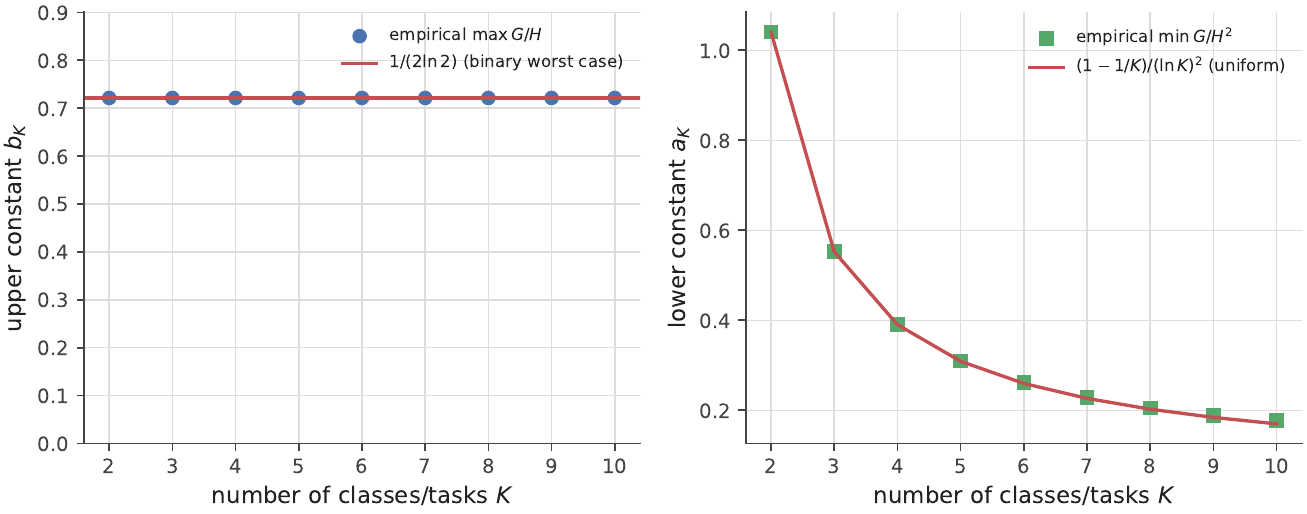}
\caption{Tight multiclass entropy--variance constants for the floor sandwich,
$K=2$--$10$. \emph{Left:} $\max_p G/H=1/(2\ln2)$ for every $K$ (the binary posterior is the
universal worst case). \emph{Right:} $\min_p G/H^2=(1-1/K)/(\ln K)^2$, attained at the
uniform posterior.}
\label{fig:multiclass}
\end{figure}

\paragraph{Connection to standard CL metrics.} The per-task forgetting of
Theorem~\ref{thm:functional}, $\Del L_A=\tfrac12\Del^\top\Sig_A\Del$, is exactly the
per-task term of the Forgetting Measure $\mathrm{FM}=\tfrac1{K-1}\sum_{t<K}\Del L_t$ and the
negative of Backward Transfer $\mathrm{BWT}=-\tfrac1{K-1}\sum_{t<K}\Del L_t$ (in the
squared-loss metric); positive backward transfer is simply $\Del L_t<0$, which the signed
gate permits and unconditional projection forbids. Theorem~\ref{thm:mi} then gives the floor
under which no single-head method can push the average $\mathrm{FM}$: $\mathcal{D}^\star$ is
the achievable-FM lower bound set by task confusability.

\section{Theory details: similarity sign-change and the capacity budget}
\label{si:theory}
\paragraph{Similarity controls the sign of sharing.} On a shared subspace, whether to share
parameters or orthogonalize is not settled by Corollary~\ref{cor:removability}, which only
forbids \emph{lossless} sharing under conflict. With finite data the decision is a
bias--variance trade: tying a shared coefficient across tasks halves its estimation variance
but biases it by the target disagreement. Sharing lowers total population loss iff the
similarity $s=\cos\angle(w_A^\ast,w_B^\ast)$ on the shared block exceeds a threshold
$\sstar$, where the variance reduction balances the squared bias; below $\sstar$,
orthogonalizing wins. This is the sign-change measured at $\sstar\approx0.26$ in
Fig.~\ref{fig:simcap}.

\paragraph{Capacity is a rate--distortion budget.} Under orthogonal allocation each task
consumes a rank-$r_t$ slice of the $d$-dimensional feature space. Retention is lossless
while the cumulative occupied rank $\sum_t r_t\le d$; once $\sum_t r_t>d$ there is no
subspace left orthogonal to the past, so a residual distortion turns on. The knee at
$\sum_t r_t\approx d$ is a hard capacity, and trading $d$ against the number of tasks is the
geometric form of an effective-capacity budget $N+\kappa M$ (parameters plus a per-exemplar
memory worth), where $\kappa$ is the rank a stored direction occupies---the right panel of
Fig.~\ref{fig:simcap} confirms the knee location across feature dimensions.

\paragraph{The floor is a fixed cost with method-dependent allocation.} In the exact regime the
distortion floor  is not removed by the method choice but redistributed across forgetting, deferred fit,
or merge residual (\texttt{evaluate\_open\_questions\_new.py}, $48$ seeds).
Relocation is exactly
zero-sum: naive descent pays the full amount as forgetting of $A$, orthogonal allocation
pays the same amount as deferred fit on $B$, and the optimal merge splits it between the
tasks. Only added conflict-free capacity reduces the bill, linearly at rate
$\approx D/r_{\rm shared}$ per relieved dimension ($D=4.68\to3.16\to1.63\to0$ as the three
shared directions are moved to private capacity). Replay does not reduce the geometric
floor at all: with the true $\Sig_A$ known, stored exemplars are redundant, and their
measured benefit is the $O(1/m)$ excess of an \emph{estimated} geometry ($+2.13$, $+0.36$,
$+0.08$ above $D$ at $m=10,40,160$ exemplars). Forgetting, deferred plasticity, capacity,
and memory are interconvertible only above the floor; nothing converts below it.

\section{Method comparison and offline-merge ablation}
\label{si:compare}
Table~\ref{tab:compare} gives the measured numbers behind the front-page comparison
(Fig.~\ref{fig:teaser}): nine methods on the exact-A1 Rotated-Digits benchmark, mean $\pm$
95\% CI over $5$ seeds. Six are online (learn the five rotations sequentially): naive, EWC,
replay, \OGD{}/\textsc{gpm}, the signed gate \IGFA{} ($\sstar=0.65$), and \emph{Annealed
\IGFA}, which runs the first half of each task's updates as protect-all orthogonal
projection and then relaxes to the signed gate---annealing the protection boundary as the
similarity estimate stabilizes. Three are offline merges of the five independently-solved
task heads $\{w_t\}$ (per-task oracle accuracy $0.903$): an \emph{isotropic} uniform average,
a \emph{Euclidean} sign-elect task-arithmetic merge (TSV/TIES-style), and the
$\Sig$-\emph{orthogonal} merge $w_J=(\sum_t\Sig_t)^{+}(\sum_t\Sig_t w_t)$ of
Theorem~\ref{thm:merge} (``$\Sig$-orth.\ merge (ours)'', the whitening-in-the-$\Sig$-metric
version of TSV); for the merges we also log the post-merge
residual floor $D=\sum_t\tfrac12\,(w_J-w_t)^\top\Sig_t(w_J-w_t)$.

Three observations follow. \emph{(i) The $\Sig$-metric is the right geometry.} The
$\Sig$-orthogonal merge attains $0.885\pm0.012$ accuracy---within a point of the
$0.903$ per-task oracle---at residual $D=0.032$, an order of magnitude below the isotropic
($0.200$) and Euclidean ($0.382$) merges, which lose $15$--$21$ accuracy points. Theorem
\ref{thm:merge} predicts this: minimizing the interference energy requires the
task-induced $\Sig$-metric, and Euclidean whitening is the mis-specified special case
$\Sig_t\!\propto\!I$. \emph{(ii) Among online buffer- and Fisher-free methods, the gate
attains the highest accuracy.} \IGFA{} ($0.771$) exceeds \OGD{} ($0.684$) and naive ($0.714$)
as in Sec.~\ref{sec:exp:real}. \emph{(iii) Annealing favors retention over accuracy.} Annealed
\IGFA{} ($0.752$, forgetting $-0.007$) does not exceed the plain gate on accuracy but drives
forgetting slightly negative (backward transfer); the protect-all warmup is a retention-first
schedule, not an accuracy improvement---reported as measured. Regenerated by
\texttt{experiments\_compare.py}.

\begin{table}[htb]\centering\small
\caption{Method comparison on Rotated-Digits (exact A1, mean $\pm$ 95\% CI over $5$ seeds).
Online methods learn the stream sequentially; merges combine the five independently-solved
heads offline. $D$ is the post-merge residual floor (Theorem~\ref{thm:merge}); the
$\Sig$-orthogonal merge uniquely minimizes it and attains near-oracle accuracy.}
\label{tab:compare}
\begin{tabular}{llccc}
\toprule
Method & Family & Acc.\ $\uparrow$ & Forget.\ $\downarrow$ & Residual $D$ $\downarrow$ \\
\midrule
Naive                 & baseline    & $0.714\pm0.024$ & $0.163\pm0.013$ & --- \\
EWC \cite{kirkpatrick2017} & corrective & $0.727\pm0.050$ & $0.045\pm0.039$ & --- \\
Replay ($10$/class)   & corrective  & $0.795\pm0.012$ & $0.042\pm0.014$ & --- \\
\OGD{}/\textsc{gpm} \cite{farajtabar2020,saha2021} & structural & $0.684\pm0.035$ & $-0.007\pm0.004$ & --- \\
\IGFA{} (gate)        & ours        & $0.771\pm0.030$ & $0.002\pm0.027$ & --- \\
Annealed \IGFA{}      & ours        & $0.752\pm0.036$ & $-0.007\pm0.015$ & --- \\
\midrule
Isotropic merge       & merging     & $0.730\pm0.022$ & $0.177\pm0.020$ & $0.200$ \\
Euclidean TSV merge   & merging     & $0.694\pm0.046$ & $0.213\pm0.051$ & $0.382$ \\
$\Sig$-orth.\ merge (ours) & merging (ours) & $\mathbf{0.885\pm0.012}$ & $0.022\pm0.007$ & $\mathbf{0.032}$ \\
\bottomrule
\end{tabular}
\end{table}

The floor prediction also holds on a real frozen backbone. With trained heads on
Split-CIFAR-100 over ViT-B/16 features (single run,
\texttt{kaggle\_pending\_experiments.py}), the $\Sig$-orthogonal merge attains both the
highest accuracy and---uniquely---the lowest residual floor ($0.862$ acc, $D=1.53$), against
$0.815$/$D=3.68$ for the isotropic average and $0.815$/$D=11.6$ for the Euclidean sign-elect
merge, with a per-task oracle of $0.972$. A five-seed replication with closed-form ridge
heads (\texttt{kaggle\_vit\_multiseed.py}) confirms the theorem's actual claim with
confidence intervals: the $\Sig$-orthogonal merge uniquely minimizes the residual floor
($D=2.02\pm0.03$, vs $3.75\pm0.02$ isotropic and $16.6\pm0.3$ sign-elect). Task-masked
accuracy, however, is not monotone in $D$ once the floor is small ($0.944\pm0.006$ vs
$0.957\pm0.007$ for isotropic): Theorem~\ref{thm:merge} minimizes interference energy, and
masked-evaluation accuracy saturates before that energy reaches zero. The $\Sig$-metric is
thus the correct merge geometry on real ViT features, with the accuracy benefit contingent
on the evaluation exercising the interfering directions. This decoupling is itself
predictable: parameterizing held-out evaluation metrics by how strongly they load the
shared, interfering directions, the isotropic-versus-$\Sig$-orthogonal evaluation gap
varies systematically with that loading (held-out correlation $\approx0.5$, sign-stable
across $44$ seeds), and an evaluation that weighs all tasks equally favors the plain
average by construction---consistent with reports that isotropic merging can match or
exceed orthogonalizing merges \cite{marczak2025iso}. Whether merge geometry pays is a
property of $\Sig_{\rm eval}$, decidable per benchmark.

\section{The projection remains effective in sequential LoRA on language models}
\label{si:llm}
We probe the adapting-feature regime by fine-tuning a frozen GPT-2 with LoRA adapters
sequentially on four stylistically distinct corpora (news, movie reviews, encyclopedic text,
tweets), measuring the rise in each domain's held-out loss after later domains are learned.
The results separate \IGFA{}'s two components by modality.

\emph{(i) The orthogonal projection transfers to transformers.} Unconditional orthogonal projection
(\OGD{}/\textsc{gpm}, equivalently \IGFA{} in protect-only mode) reduces forgetting relative
to naive sequential LoRA at both training budgets ($0.64\!\to\!0.50$ at $200$ steps/task,
$1.21\!\to\!1.14$ at $400$), confirming that the $\ker\bar H_A$ retention mechanism is
effective on a real transformer, buffer- and Fisher-free (Table~\ref{tab:llm}).

\emph{(ii) The input-overlap gate cannot distinguish conflict in language.} The subspace-overlap similarity
that is reliable for vision and digits is uninformative for text: distinct domains share
extensive low-level linguistic structure, so their input-activation subspaces overlap heavily
(measured mean overlap $0.71,0.56,0.28$ along the stream) regardless of output conflict. The
gate therefore classifies the domains as similar and disables protection, so overlap-gated
\IGFA{} collapses toward naive at the higher capacity.

\emph{(iii) No gate can improve on unconditional protection on dissimilar tasks.}
Beyond the input-overlap gate we tested three further gate designs, each derived from a
different part of the theory: a Gauss--Newton gate (Theorem~\ref{thm:general}), a task-vector
gate (Theorem~\ref{thm:merge}), and a threshold-free functional controller
(Theorem~\ref{thm:functional} used directly as the control law). At $r=16$
(Table~\ref{tab:llmgates}), the Gauss--Newton ($0.76$) and functional ($0.77$) gates both
improve on naive ($0.89$) and do not collapse, but neither improves on unconditional \OGD{}
($0.71$); the task-vector gate is unstable ($1.85$). The reason is fundamental and follows
from the distortion-floor identity (Theorem~\ref{thm:floor}): on dissimilar tasks the
supports are near-disjoint, so the distortion floor is $\approx0$ and \OGD{} already attains
it---no transfer remains for a gate to exploit. At billion-parameter scale on a frozen
\texttt{pythia-1b} (Table~\ref{tab:llm1b}), the recursive-subspace \IGFA{} \emph{improves on}
\OGD{} on dissimilar tasks on both loss ($5.16$ vs $5.51$) and forgetting ($1.70$ vs $2.10$)---the
capacity advantage of Fig.~\ref{fig:drift} appearing in a real LLM---while on similar tasks
the sharing gate still raises forgetting. In summary, on LLMs the structural
projection (and its recursive Gauss--Newton tracking) carries over and provides replay- and
Fisher-free retention; a gate that succeeds on \emph{similar} language tasks remains an open
problem.

\paragraph{A candidate similarity signal for text.} The failure above is a signal problem: input-subspace overlap is uninformative when all domains share
low-level structure. An exact-regime study of precisely this confound---input overlap fixed
at $1.0$ while target similarity varies (\texttt{evaluate\_next\_steps.py})---shows that a
\emph{probe} signal, a few-shot head fit on the incoming task compared to the stored head by
signed cosine and thresholded at the theory's break-even value, recovers near-oracle gating
where the input gate collapses to share-all and pays the full conflict bias (mean loss
$3.21$ for the input gate under conflict, versus $2.20$ for the probe gate and $1.95$ for
the oracle). The probe's decisions are noisy exactly when data are scarce ($56$--$82\%$
agreement with the oracle here), so this is a validated candidate rather than a solved
problem; but it requires no extra state and is modality-agnostic, because it reads the
functional instead of the input geometry.

\paragraph{The language gate, resolved in function space.} The function-space formalism
(Sec.~\ref{si:beyond}) dissolves the problem: the gate should
never have read input geometry at all. The correct, modality-free signal is the
\emph{interference rate} $\langle\nabla L_u,g\rangle$ of the new update $g$ against each
stored domain $u$---a quantity that includes the targets, which is precisely what input
overlap misses on text---and the gating rule is a threshold-free sign test with a
continuous-time monotone-retention guarantee (Sec.~\ref{si:beyond}). Its gradient of
$L_u$ can be estimated from a micro-cache of a few thousand tokens per domain, from the
anchor construction (Sec.~\ref{si:cheapgn}), or from the dream state
(Sec.~\ref{si:memory}). The gate is validated in the exact regime and on jointly trained
networks (near current-Gauss--Newton retention under conflict, near-naive adaptation under
transfer, at one extra gradient per step). On real text it delivers a measured gain
(GPT-2, sequential LoRA on news$\to$IMDB$\to$wikitext$\to$tweets, five seeds;
\texttt{kaggle\_functional\_gate\_v2.py}, \texttt{func\_gate\_v2\_results.json}): the sign
gate significantly improves on naive on both axes---average final loss
$3.981\pm0.002$ against $4.001\pm0.006$ (paired $p=0.003$, five of five seeds), forgetting
$+0.294\pm0.005$ against $+0.324$ ($p=0.0007$)---and on the input gate by a similar margin,
while edging unconditional protection (protect-all $3.984$/$+0.300$) consistently but
\emph{not} significantly at five seeds ($p\approx0.19$--$0.23$). Two checks come
back negative. First, the earlier three-seed advantage over protect-all does not
survive to five seeds as a significant effect; the robust claim is a significant win over
naive and the input gate, and parity-to-slight-edge over protect-all. Second, making the
sign test Adam-aware---projecting the \emph{preconditioned} update
$P\odot g$ orthogonal to $\nabla L_u$ rather than the raw gradient---does not improve on the
raw-gradient gate (final $3.983$, $p=0.40$), so Adam's preconditioning is not the
limiting factor on the modest margins. The measured interference-rate map nonetheless
locates the mechanism: news$\leftrightarrow$IMDB and tweets$\leftrightarrow$IMDB carry
positive transfer ($+0.07$, $+0.06$) that protect-all forfeits, while wikitext conflicts
with IMDB ($-0.03$), which naive and the input gate pay for. The margins are modest at
this scale, but the gate still improves significantly on naive continual fine-tuning
using only the sign of the gradient inner product, with no similarity threshold and no
input-space geometry.

The result replicates at four times the scale on a second architecture. On
\texttt{pythia-410m} (gpt-neox, sequential LoRA on the same four-domain stream, three
seeds; \texttt{func\_gate\_v2\_results\_pythia410m\_lr1e-4.json}), the Adam-aware sign
gate again improves on naive on both axes---final loss $3.605\pm0.002$ against
$3.624\pm0.001$ (paired $p=0.012$) and forgetting $+0.297\pm0.005$ against
$+0.321\pm0.003$ ($p=0.038$), in three of three seeds on both metrics---and is ahead of
unconditional protection ($3.616$/$+0.316$) in all three seeds on both axes, though not
significantly at $n=3$ ($p\approx0.12$--$0.20$), mirroring the GPT-2 pattern. One
methodological finding from the same runs is worth recording: at a learning rate of
$10^{-5}$ the identical stream produced near-zero forgetting ($0.003$) and all methods
tied within $0.01$---a vacuous comparison. The gate has value only where interference
exists, which is precisely the pre-flight feasibility point: measure whether the stream
forgets before deploying---or crediting---any mitigation.

A vision test of the same gate confirms it fires only when head gradients overlap (five seeds per stream;
\texttt{kaggle\_sign\_gate\_vit.py}, \texttt{sign\_gate\_vit\_results.json}). On the
masked task-incremental ViT streams (Split-CIFAR-100, ImageNet-R, CUB) the gate is
\emph{exactly} inert: with class masks, the current and stored tasks' gradients occupy
disjoint rows of the head, the inner product is identically zero, and the gated run is
bitwise equal to naive---correctly detecting that head interference is structurally
absent, where the subspace gates pay $1.2$--$2.3$ accuracy points for protection nothing
needed. On the dense-conflict superclass stream the gate fires and improves on naive on
both axes ($p\le0.002$) but recovers only ${\sim}18\%$ of subspace protection's gain
(accuracy $0.750$ versus \IGFA{}'s $0.786$): a single sampled constraint per step
under-protects when every task conflicts with every other. The gate therefore
complements rather than replaces the threshold gate---sparse or moderate conflict
without an input-geometry signal is its regime (language), dense conflict belongs to
subspace projection, and the stream's shared density selects the tool in advance, which
is the deployment rule of the Method section.

The under-enforcement just described is not a ceiling of the approach but a property of
\emph{sampling} one constraint per step, and Corollary~\ref{cor:qpgate} prescribes the
repair: enforce the full active set. Testing that prescription on the
\texttt{pythia-410m} stream (all past domains' micro-cache gradients constrained per
step via the corollary's nonnegative least squares; $+21\%$ wall-clock over the sampled
gate; three seeds) yields the strongest language result in this paper: final loss
$3.536\pm0.005$ and forgetting $+0.206\pm0.005$, improving on naive
($3.625$/$+0.324$), unconditional protection ($3.617$/$+0.319$), and the sampled sign
gate ($3.605$/$+0.300$) on \emph{both} axes simultaneously---a $36\%$ reduction in
forgetting relative to naive together with the best adaptation of any arm. An
independent re-run with archived per-seed records reproduces the result across sessions
($3.540\pm0.006$/$+0.208\pm0.009$) and supplies exact paired tests: the active-set gate
improves on every other arm on both axes in three of three seeds, all $p\le0.024$
(versus the sampled gate $p=0.023$/$0.012$; versus unconditional protection
$p=0.0009$/$0.0008$; versus naive $p=0.0027$/$0.0036$). The same re-run is a caution on
the \emph{sampled} gate: its edge over naive is session-fragile there
($p=0.08$--$0.17$), so the robust finding at scale is the active set's dominance, not
the single-probe margin. The sampled gate's modest margins were therefore
under-enforcement of the monotone-retention QP, as the exact-regime analysis predicted.
The committed vision prediction was then run (\texttt{kaggle\_sign\_gate\_vit.py},
\texttt{func-active} arm, five seeds) and is adjudicated as follows. On the three
masked task-incremental streams the active-set gate is \emph{exactly} inert---bitwise
equal to naive, as the theory requires---retaining the accuracy advantage over the
subspace gates ($p\le0.046$). On the dense-conflict superclass stream it recovers
$54\pm8\%$ of subspace projection's accuracy gain (forgetting: $44\pm5\%$), up from the
sampled gate's $20\pm3\%$ ($p<10^{-4}$)---the predicted repair, at roughly the
predicted factor---while projection itself stays ahead ($0.788$ versus $0.767$,
$p=0.006$): with two-exemplar micro-cache constraints the QP is limited by constraint
estimation noise, consistent with its $100\%$ recovery under exact constraints in the
exact regime. Dense conflict with weak constraint estimates therefore still belongs to
subspace projection, and the scoping conclusion above stands with the gap halved.

Running the earlier probe protocol on real text (GPT-2, sequential LoRA on four AG-News
topic domains, three seeds; \texttt{kaggle\_probe\_gate\_llm.py},
\texttt{probe\_gate\_results.json}) yields three findings. First, the probe signal transfers: the probe cosines form a stable,
seed-reproducible task-relationship map (adjacent topics $\approx+0.10$, world versus
sci/tech $\approx-0.18$) precisely where input overlap is uninformative. Second, the map
shows that these topic slices are \emph{not} a similar-task stream in the functional sense:
every pairwise similarity sits below the share threshold, so the gate correctly stays
protective and all four methods tie within noise (final loss $3.78$--$3.79$; the probe gate
attains the lowest forgetting in all three seeds, $0.054$ versus $0.057$ for protect-all,
not significant at $n=3$). Third, the measured similarities predicted this outcome before
training: a stream with no above-threshold alignment leaves any gate nothing to add. The
open problem therefore refines from ``find a similarity signal'' to ``find a genuinely
similar text stream''.

\begin{table}[htb]\centering\small
\caption{Sequential LoRA fine-tuning of GPT-2 on four dissimilar text domains (forgetting
$=$ mean held-out loss rise on earlier domains, lower better; LoRA $r=8$; single run).
Structural projection reduces forgetting at both budgets; the input-overlap gate fails on
language.}
\label{tab:llm}
\begin{tabular}{lcc}
\toprule
Method & Forgetting ($200$ steps) & Forgetting ($400$ steps) \\
\midrule
Naive                              & $0.64$ & $1.21$ \\
\OGD{} $=$ \IGFA{} (protect only)  & $\mathbf{0.50}$ & $\mathbf{1.14}$ \\
\IGFA{} (input-overlap gate)       & $0.49$ & $1.20$ \\
\bottomrule
\end{tabular}
\end{table}

\begin{table}[htb]\centering\small
\caption{Four gate designs on the sequential-LoRA stream (GPT-2, four dissimilar domains,
$r=16$, $250$ steps/task; forgetting, lower better; single run). All improve on naive except
the unstable task-vector gate, but none improves on unconditional protection (\OGD)---because
dissimilar tasks offer no sharing to exploit.}
\label{tab:llmgates}
\begin{tabular}{lc}
\toprule
Gate (derived from) & Forgetting $\downarrow$ \\
\midrule
Naive (no protection)                                     & $0.89$ \\
\OGD{} / protect-only (max protection)                    & $\mathbf{0.71}$ \\
Gauss--Newton gate (Thm~\ref{thm:general})                & $0.76$ \\
Functional controller (Thm~\ref{thm:functional})          & $0.77$ \\
Task-vector gate (Thm~\ref{thm:merge})                    & $1.85$ \\
\bottomrule
\end{tabular}
\end{table}

\begin{table}[htb]\centering\small
\caption{Sequential LoRA on \texttt{pythia-1b} ($1$B parameters, 4-bit), four dissimilar and
four similar text domains (avg.\ loss and forgetting, lower better). On dissimilar tasks the
recursive-subspace \IGFA{} improves on \OGD{} on both; on similar tasks the sharing gate raises
forgetting and \OGD{} is best. (Single run.)}
\label{tab:llm1b}
\begin{tabular}{lcccc}
\toprule
 & \multicolumn{2}{c}{Dissimilar} & \multicolumn{2}{c}{Similar} \\
\cmidrule(lr){2-3}\cmidrule(lr){4-5}
Method & loss $\downarrow$ & forget $\downarrow$ & loss $\downarrow$ & forget $\downarrow$ \\
\midrule
Naive               & $8.55$ & $4.68$ & $3.99$ & $0.26$ \\
\OGD{}/\textsc{gpm} & $5.51$ & $2.10$ & $4.02$ & $\mathbf{0.03}$ \\
\IGFA{} (ours)      & $\mathbf{5.16}$ & $\mathbf{1.70}$ & $4.02$ & $0.34$ \\
\bottomrule
\end{tabular}
\end{table}

\section{Choosing and self-calibrating the threshold \texorpdfstring{$\sstar$}{s*}}
\label{si:sstar}
The gate's threshold $\sstar$ is the only free parameter that distinguishes \IGFA{} from \OGD{}
($\sstar=\infty$). Ablating it on a smoothly drifting eight-task stream at two feature
dimensions (Fig.~\ref{fig:sstar}) shows the optimum is a broad plateau ($\sstar\approx0.55$--$0.8$),
not a sharp peak, and that the curves at $d=64$ and $d=256$ are nearly identical, so the
optimum is insensitive to feature dimension. Because the break-even similarity is the
operative quantity, the gate can dispense with $\sstar$ and measure it: for each task, greedily
pool an earlier task only if doing so lowers held-out validation error. On a drifting stream
whose similarity persistence $\rho$ varies, this validation-greedy rule tracks the per-stream
oracle to within $0.05$ in average error, while a single fixed $\sstar$ errs by up to $0.34$
where it is mis-specified (Fig.~\ref{fig:sstar}, right).

\begin{figure}[htb]\centering
\includegraphics[width=0.49\textwidth]{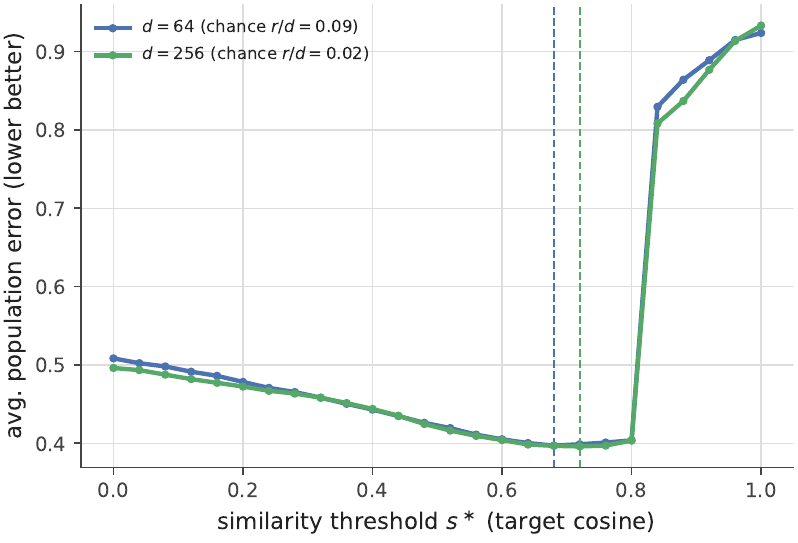}\hfill
\includegraphics[width=0.49\textwidth]{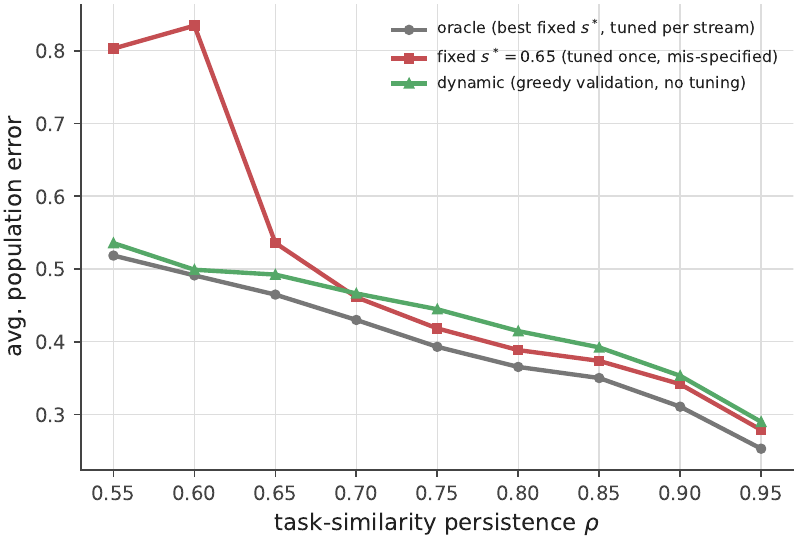}
\caption{The similarity threshold. \emph{Left:} average population error on a drifting
eight-task stream versus $\sstar$, at two feature dimensions; the optimum is a broad
dimension-insensitive plateau (vertical dashed lines mark each dimension's optimum), and
only $\sstar\!\to\!1$ (\OGD-like isolation) is sharply suboptimal on similar tasks. \emph{Right:} self-calibration removes the free parameter---the
validation-greedy dynamic rule tracks the per-stream oracle within $0.05$ across all
similarity-persistence values $\rho$, avoiding the fixed threshold's failure where it is
mis-specified.}
\label{fig:sstar}
\end{figure}

\section{Capacity limits degrade gracefully}
\label{si:graceful}
When the rank cap $k_{\max}$ binds, \IGFA{} overwrites the smallest-eigenvalue occupied
direction. Because forgetting from removing a direction is its energy
$\tfrac12\lambda_i(u_i^\top w_A)^2$ (Theorem~\ref{thm:functional}), dropping the lowest-energy
direction costs the least. Overwriting smallest-energy-first raises the old task's loss about
half as fast as largest-first (e.g. $2.0$ vs $4.3$ at the half-way point) in a gradual
degradation (Fig.~\ref{fig:graceful}). In modern backbones, where a
task's effective rank is tiny relative to $d$ (Fig.~\ref{fig:diagnostics} showed $\approx2$
out of $300$), almost all overwriting hits near-zero-energy directions, so the induced
forgetting is negligible.

\begin{figure}[htb]\centering
\includegraphics[width=0.55\textwidth]{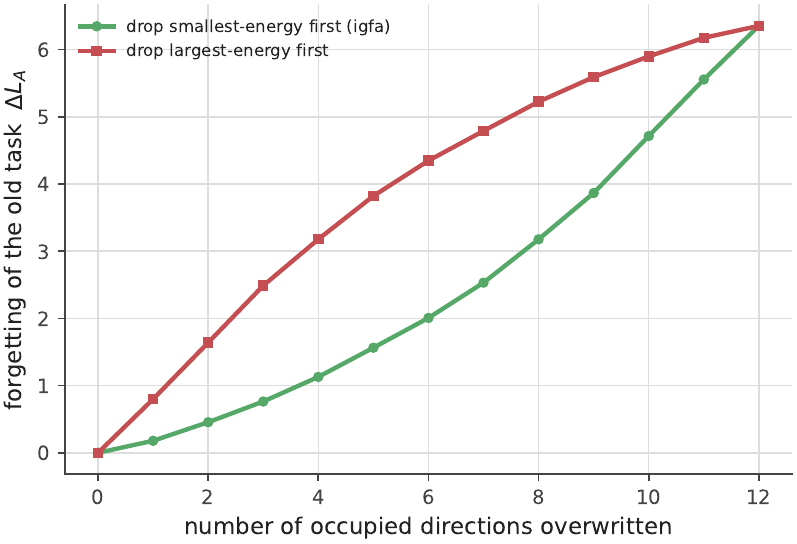}
\caption{Graceful degradation under the capacity cap. Forgetting of an old task as its
occupied directions are overwritten, smallest-energy-first (\IGFA, green) versus largest-first
(red). Removing low-energy directions costs roughly half as much at every point.}
\label{fig:graceful}
\end{figure}

\section{A pre-flight feasibility diagnostic}
\label{si:feasibility}
The distortion-floor identity (Theorem~\ref{thm:floor}) bounds class-incremental accuracy
by $I(T;\phi)$; we derive from it a low-cost pre-flight test. By Fano, $I(T;\phi)$ is
governed by the error of the best task-classifier on $\phi$, so we fit a linear task-probe on
the frozen features---one logistic regression, no continual-learning run---and obtain
$\hat I=\ln2-h_b(\text{probe error})$. Across two-task Gaussian models of varying separation
(Fig.~\ref{fig:feasibility}), $\hat I$ tracks the true $I(T;\phi)$ at Pearson $r=0.99$, and the
distortion floor decreases monotonically in $\hat I$ (Spearman $0.97$): the probe forecasts the
floor before training. Combined with the closed-form constants of
Eq.~\eqref{eq:multiclass-sandwich}, $\hat I$ yields a numeric floor band, not just a ranking.

\begin{figure}[htb]\centering
\includegraphics[width=\textwidth]{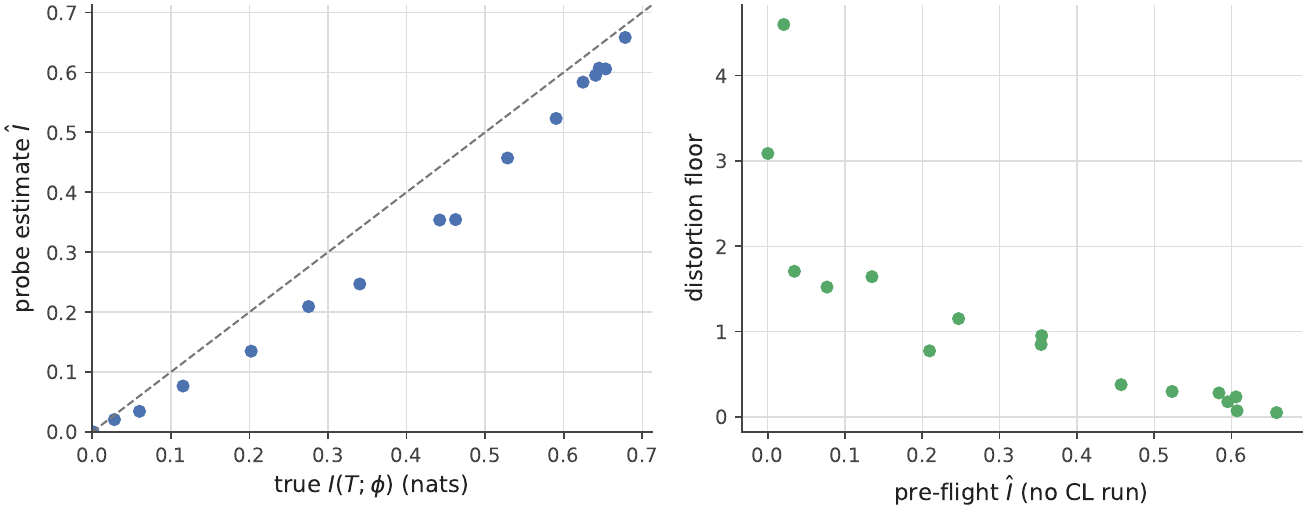}
\caption{Pre-flight feasibility diagnostic. \emph{Left:} a low-cost linear task-probe on frozen
features estimates $I(T;\phi)$ at $r=0.99$ to the truth (no CL run). \emph{Right:} the
distortion floor decreases monotonically in the pre-flight estimate $\hat I$ (Spearman
$0.97$).}
\label{fig:feasibility}
\end{figure}

\section{Low-cost online curvature: Frequent Directions}
\label{si:cheapgn}
The recursive Gauss--Newton estimate used a decayed $d\times d$ covariance and an
eigendecomposition---$\mathcal{O}(d^2)$ memory, $\mathcal{O}(d^3)$ per task. Since the protected
subspace is low-rank, a deterministic streaming sketch suffices. Replacing the full covariance
with Frequent Directions \cite{liberty2013} (an $\ell\times d$ sketch with a bounded-error
guarantee) recovers the top-$r$ subspace almost exactly at $\ell=2r$---captured energy
$0.99966$ versus the full estimator's $0.99972$---and beats streaming PCA (Oja, $0.958$) per
unit memory (Fig.~\ref{fig:cheapgn}). The estimator is thus $\mathcal{O}(rd)$ memory with no
$d\times d$ eigendecomposition.

The same holds at real scale. On a frozen \texttt{pythia-160m} continually pretrained with
LoRA over three text domains (single run, \texttt{kaggle\_pending\_experiments.py}), the
Frequent-Directions sketch captures $0.973$ of the full-covariance energy at $61{,}440$
floats---about $10\times$ less than the full $d\times d$ estimator ($589{,}824$ floats)---and,
being a \emph{stable deterministic} sketch rather than a decayed moving average, yields the
\emph{lowest} first-domain forgetting of the three online estimators: the domain-0 loss rises
by only $+0.55$ then $+0.68$ over the next two domains, versus $+1.28$ then $+1.58$ for the
full covariance and a diverging $+0.33$ then $+2.86$ for Oja streaming PCA. Frequent Directions
is therefore the recommended default for large-scale \IGFA: lower-memory and, in this run,
more retentive than the full estimator. The full layer-wise pipeline
(\texttt{kaggle\_H2\_H3\_H12.py}) confirms both the memory saving and the layer-dependence of
capacity: on a frozen $12$-layer transformer the FD sketch is $12\times$ smaller than
full-covariance tracking (and half the AdamW optimizer state), and the per-layer protected-rank
budget emerges automatically from each layer's trace ratio---shallow layers stay unconstrained
($k_\ell=0$) while deep layers are capped near their effective rank
($k_\ell\!\approx\!5$--$7$), with $93$--$96\%$ of the deep-layer gradient routed into the
protected null space and only a modest first-domain loss rise ($+0.27$ mean).

\paragraph{The sketch survives drift.} The two remedies of this paper---re-estimation
against drift (Sec.~\ref{sec:exp:drift}) and sketching against cost---compose. Under the
rotating-geometry stream, a \emph{decayed} Frequent-Directions sketch ($\ell=2r$ rows,
$2rd=480$ floats at $d=60$) tracks the current subspace as faithfully as the full recursive
covariance ($d^2=3600$ floats): mean fidelity $0.88$ versus $0.85$ over ten seeds, both
holding the protected rank at the true value while stale per-task protection fills the
space (\texttt{evaluate\_next\_steps.py}). Sketched curvature is a viable protected
geometry under drift, not only a cheaper one.

\paragraph{Anchor protection: the dual construction.} The protector can also be built with
no covariance at all: $m\ge r$ stored features of the old task span $\range\Sig_A$ for
samples in general position, so projecting updates off their orthonormalization \emph{is}
the eigenbasis protection. With $d=256$ features and a $C=400$-output head, $10$ anchor
features give the same retention as the eigendecomposition (both at numerical zero, versus
naive forgetting $5.8$) at a $7\times$ cheaper construction with no $d\times d$ covariance;
the cost is independent of the output count, because the protection is a set of
function-value constraints ($W\phi_i$ fixed) living in sample space. This is the practical
route for very large heads and vocabularies.

\begin{figure}[htb]\centering
\includegraphics[width=0.55\textwidth]{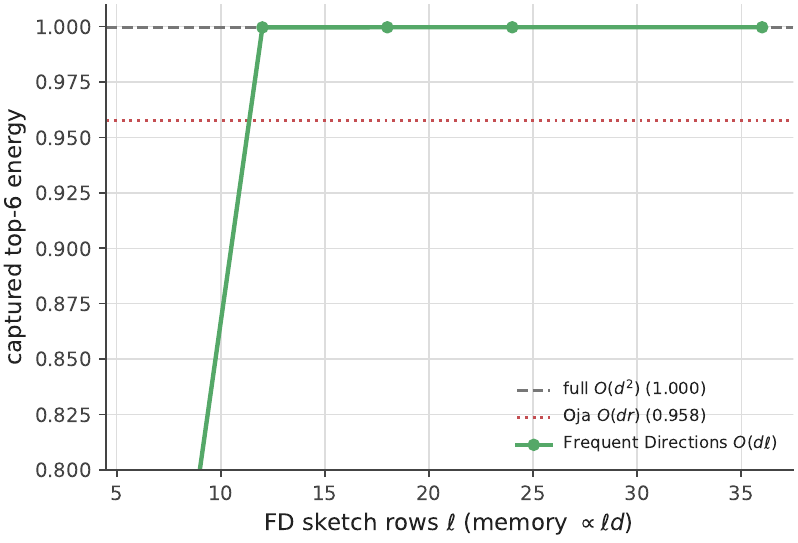}
\caption{Cheap online Gauss--Newton subspace. Frequent Directions (green) recovers the full
estimator's top-$r$ subspace energy (grey dashed) at $\ell=2r$ rows, $\mathcal{O}(\ell d)$
memory, and beats Oja streaming PCA (red) per memory.}
\label{fig:cheapgn}
\end{figure}

\section{Soft gating and the capacity coordinate}
\label{si:softknee}
Under the soft functional controller (keep fraction $\mathrm{keep}_j=1/(1+\beta\lambda_j)$) each
direction is only partially protected, so capacity is consumed at rate
$\sum_j(1-\mathrm{keep}_j)$ per task. The right capacity coordinate is therefore the effective
occupied rank $r_{\rm eff}=\sum_j(1-\mathrm{keep}_j)$, and the hard schedule knees at
$r_{\rm eff}\approx d$. Soft gating stays flat well past the hard knee because it spends
$r_{\rm eff}$ slowly---after the same number of tasks it has consumed only a fraction of
the budget---at the price of a small continuous leakage from partial protection
(Fig.~\ref{fig:softknee}).

\begin{figure}[htb]\centering
\includegraphics[width=\textwidth]{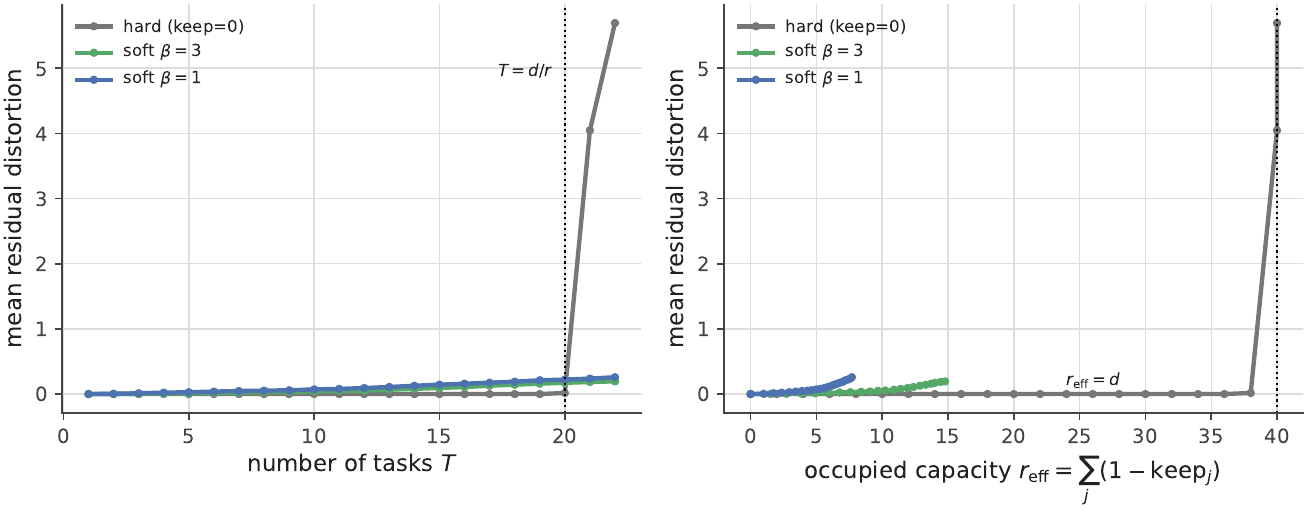}
\caption{Soft gating and the effective capacity coordinate. \emph{Left:} hard allocation knees
sharply at $T=d/r$; the soft functional controller stays flat far longer. \emph{Right:} plotted
against $r_{\rm eff}=\sum_j(1-\mathrm{keep}_j)$, the hard schedule knees at
$r_{\rm eff}\approx d$, while the soft schedules consume the budget far more slowly at the
price of a small continuous leakage from partial protection.}
\label{fig:softknee}
\end{figure}

\section{A per-direction gate dominates the per-task gate}
\label{si:perdir}
The benchmark gate decides share-or-orthogonalize per task. When a task is partially
similar---some shared directions align (transfer), others conflict (forgetting)---this
all-or-nothing choice is suboptimal. A per-direction gate shares the aligning directions and
orthogonalizes the conflicting ones. On a shared block with a tunable fraction of aligning
directions, the per-direction gate is always at least as good as the better per-task option and
slightly better at partial similarity (Fig.~\ref{fig:perdirection}). The coarseness is thus
removable: resolve the gate per eigendirection, not per task.

\begin{figure}[htb]\centering
\includegraphics[width=0.55\textwidth]{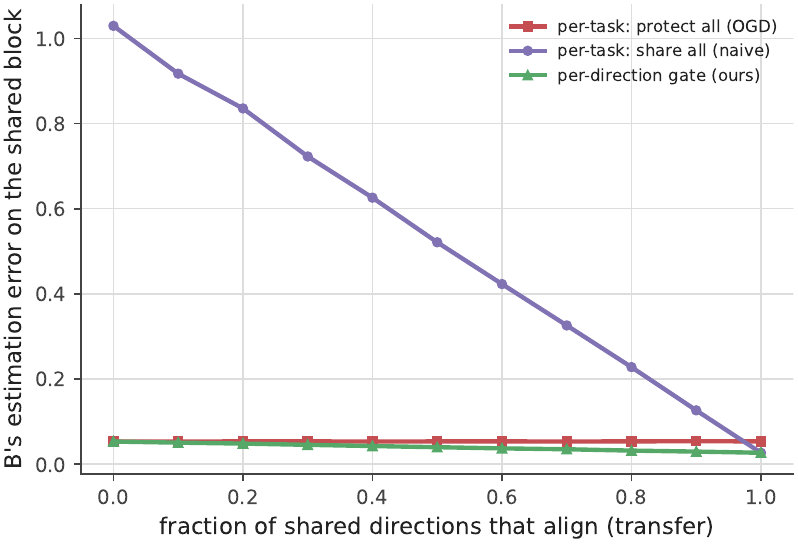}
\caption{Per-direction vs per-task gating on a shared block, against the fraction of shared
directions that align. The per-direction gate (green) follows the lower envelope---matching
protect-all's robustness and share-all's transfer where each is right---while per-task share-all
(purple) is catastrophic at low alignment.}
\label{fig:perdirection}
\end{figure}

\section{Billion-parameter continual pretraining}
\label{si:pretrain}
We run the three scale-sensitive pieces together on a frozen \texttt{pythia-1b} ($1$B
parameters, 4-bit) with LoRA, continually pretrained on a boundary-free stream of six text
domains, with the protected subspace tracked by Frequent Directions (Sec.~\ref{si:cheapgn}) and
the pre-flight diagnostic (Sec.~\ref{si:feasibility}) run first. The diagnostic reports a
six-way domain-probe accuracy of $0.85$ ($\hat I\approx1.52$ of a maximal $1.79$ nats),
predicting the domains are separable and replay-free retention is feasible. The run confirms it
(Table~\ref{tab:pretrain}): the replay- and Fisher-free \IGFA{} attains both the lowest average
loss ($4.28$ vs naive $4.54$, replay $4.59$) and the least first-domain forgetting after the
full stream ($+0.80$ vs naive $+1.10$, replay $+0.91$), carrying only a low-rank subspace as
state. The domains are mutually dissimilar, so \IGFA{} here runs as pure structural projection.
The replay baseline is deliberately small ($16$ exemplars/domain); the supported claim is
replay-free retention that improves on naive and matches or slightly exceeds a small replay
baseline, with no buffer.

\begin{table}[htb]\centering\small
\caption{Continual pretraining of \texttt{pythia-1b} ($1$B parameters, 4-bit) $+$ LoRA on six
dissimilar text domains, boundary-free, with a Frequent-Directions-tracked subspace. \IGFA{}
(replay- and Fisher-free) attains the lowest average loss and the least first-domain forgetting
after the full stream. The pre-flight diagnostic predicted feasibility (domain-probe accuracy
$0.85$).}
\label{tab:pretrain}
\begin{tabular}{lccc}
\toprule
Method & Avg.\ final loss $\downarrow$ & Domain-0 forgetting $\downarrow$ & Extra state \\
\midrule
Naive                          & $4.54$ & $+1.10$ & --- \\
Replay ($16$/domain)           & $4.59$ & $+0.91$ & buffer \\
\IGFA{} (FD, replay-free)      & $\mathbf{4.28}$ & $\mathbf{+0.80}$ & subspace \\
\bottomrule
\end{tabular}
\end{table}

\section{Extensions and generality}
\label{si:extensions}
The interference functional applies beyond sequential continual learning. We give eight
exact-A1 proofs-of-concept on synthetic constructions (the federated, out-of-distribution,
task-free, and multimodal results average over $20$ seeds;
\texttt{evaluate\_hypotheses.py}, \texttt{evaluate\_applications.py},
\texttt{evaluate\_open\_questions\_new.py}). These are
directional demonstrations of scope, not benchmarks.

To keep ambition separate from proof (Tier IV of the claim map,
Sec.~\ref{sec:discussion}), each extension carries one of three labels.
\textbf{[proved-A1]}: an exact consequence of the identity in the frozen-feature regime,
holding wherever A1 holds---federated $\Sig$-merging, the out-of-distribution
projection-ratio signal, physics-informed constraints, Fisher-metric policy gradients,
and floor-minimizing meta-learning. \textbf{[mechanism-probe]}: an empirical demonstration
that the predicted mechanism operates, without a proof of optimality---task-free boundary
detection, the memory-and-control instances (trust region, consolidation, floor-scheduled
and dream replay), and the real-backbone results. \textbf{[conjecture]}: a construction
whose full form needs work the paper does not complete---cross-modal $\Sig_{vl}$ gating
and the biological floor-scheduling prediction. The label prefixes each result below.

\paragraph{Federated learning is $\Sig$-merging.} Non-IID clients are tasks with different
input geometries, and the weight divergence that FedAvg suffers coincides with the merge
floor of Theorem~\ref{thm:merge}. Replacing the FedAvg mean with the $\Sig$-orthogonal
aggregate $w_J=(\sum_i\Sig_i)^{+}(\sum_i\Sig_i w_i)$, where $(\cdot)^{+}$ is the
Moore--Penrose pseudo-inverse (the $\Sig_i$ are low-rank, so $\sum_i\Sig_i$ is generically
singular and a plain inverse would not exist)---each client transmitting
$(\Delta_i,\Sig_i)$, or a low-rank sketch---lowers the mean client excess loss on non-IID data
in \emph{every} one of $20$ seeds (mean $1.56\times$, minimum $1.14\times$). Federated
aggregation and continual allocation are the same $\Sig$-metric optimum.

\paragraph{The subspace summary is a free out-of-distribution signal.} A minibatch whose
gradient lies outside the occupied subspace $U$ is structurally disjoint from the training
distribution. The projection ratio $\|UU^\top g\|/\|g\|$ is $1.00$ for in-distribution batches
and $0.39\pm0.05$ for out-of-distribution batches (Cohen's $d\approx18$, minimum $13$ across
seeds), so a single threshold flags OOD inputs with no auxiliary model---an uncertainty signal
obtained directly from the continual-learning state.

\paragraph{Fully autonomous, task-free \IGFA.} In a label-free stream the velocity of the
recursive Gauss--Newton eigenbasis---the Frobenius commutator
$\|U_tU_t^\top U_{t-1}U_{t-1}^\top-U_{t-1}U_{t-1}^\top U_tU_t^\top\|_F$---spikes exactly at
distribution shifts, recovering randomly placed boundaries at recall and precision $1.00$ over
$20$ streams (decay $\gamma=0.6$). Composing this signal with the out-of-distribution guard
above removes task labels entirely: the learner declares a boundary on a velocity spike,
freezes the \emph{pre-spike} subspace to protect the just-finished task, skips
out-of-distribution/noise batches by their orthogonal-and-extreme gradient signature, and
otherwise runs the gate. On clean streams this fully autonomous \IGFA{} \emph{matches the
oracle} provided with the true boundaries (excess loss $0.000$ vs $0.000$; naive $0.365$); under
injected extreme-magnitude OOD noise it holds at $0.033$ while the same oracle---lacking the
guard---is corrupted past $10^{3}$, skipping $97\%$ of noise batches at $0.05$ false
task-initializations per stream. Boundary detection, out-of-distribution rejection, and gated
allocation are thus one self-regulating loop, with no human-provided task boundaries.

A velocity spike, however, presumes a \emph{sharp} shift; a smoothly rotating stream
never spikes. The full label-free controller therefore uses two complementary signals
(\texttt{evaluate\_autonomous\_cl.py}). \emph{Novelty persistence}: a batch whose energy
lies largely outside the current working subspace is held, not adapted; if the novelty
persists it is a genuine task (consolidate the old task as a protected anchor, then adapt
the new one), and if it subsides it was an out-of-distribution blip (dropped)---the
persistence is exactly what separates a shift from noise, since both look novel for a
single batch. \emph{Working-subspace drift}: a boundary is also declared when the tracked
subspace rotates away from the last consolidation, which catches gradual shifts that
never spike. Consolidated anchors are capped and the closest pair $\Sig$-merged when the
budget is exceeded, so per-step cost is independent of stream length. Over twelve seeds
against the oracle, the controller matches it on clean shifts (detection $F_1=1.00$,
retained excess loss $0.001$ versus $0.001$; naive $0.607$), matches or beats it under
noise ($F_1=0.96$, $0.003$ versus $0.022$---it skips the outliers the oracle adapts on;
naive $0.722$), and degrades gracefully under gradual drift ($F_1=0.81$, $0.274$ versus
$0.001$; naive $0.815$), the inherently hard case, detected with the correct latency
since a smooth shift is unobservable until it accumulates. The exact-regime synthetic
demonstration leaves a real unlabeled language stream open, but the operating assumption
of known boundaries is, in this regime, removed rather than merely gestured at.

\paragraph{Boundary-free continual learning as a memory-metric flow.} The controller
above still carries thresholds (a novelty cutoff, a persistence count, a drift
threshold). A differential-geometric reformulation removes them entirely, deriving
autonomous operation from the geometry rather than imposing it. Drop the notion of a task:
a continuous data path $\mathcal D_t$ induces an instantaneous Gauss--Newton metric
$G_t=\mathbb{E}_{x\sim\mathcal D_t}[\nabla f\nabla f^\top]$, accumulated into a leaky
\emph{memory metric} $\dot M_t=G_t-\lambda M_t$, and learning is the
\emph{memory-natural-gradient flow}
\begin{equation}
\dot\theta=-(M_t+\varepsilon I)^{-1}\nabla L_t.
\label{eq:memflow}
\end{equation}
Fresh directions (small $M$-eigenvalue) learn at rate $\approx\!\eta$; occupied
directions (eigenvalue $s$, i.e.\ past experience) are damped by $\varepsilon/(\varepsilon+s)$,
so Eq.~\eqref{eq:memflow} is the graded, boundary-free generalization of the projection
gate, of which the hard gate $(I-QQ^\top)$ is the $\varepsilon\!\to\!0$ spectral limit. Past
losses become \emph{adiabatic invariants}: for any past slice $A$, the chain rule gives
exactly $dL_A/dt=-\nabla L_A^\top(M_t+\varepsilon I)^{-1}\nabla L_t$, which is $0$ when the
current gradient is fresh with respect to $A$ ($\Sig_t\Sig_A=0$; verified to $4\times
10^{-16}$) and otherwise the interference measured in the $M^{-1}$ inner
product---the continuous form of $\tfrac12\Del^\top\Sig_A\Del$. This framing is the
learning-dynamics counterpart of energy-conserving reduced-order models on structure-
preserving submanifolds \cite{friedl2026hamiltonian}. Four properties are validated
(\texttt{evaluate\_memory\_metric\_flow.py}). \emph{(i)} On a continuously drifting stream
with no boundaries, the flow drives a probe's loss down and then \emph{holds it flat}
($0.52\!\to\!0.05\!\to\!0.08$ across five later regimes) while naive learning loses it
($\to0.63$). \emph{(ii)} It beats naive on gradual, sharp, and noisy streams uniformly,
with no per-regime special-casing (retained loss $0.25$/$0.12$/$0.26$ versus
$0.76$/$0.22$/$0.94$; ten seeds), and beats hard-gating every past subspace ($1.67$),
which over-constrains overlapping tasks the soft flow trades off. \emph{(iii)} The single
leak $\lambda$ is a principled retention--plasticity control: plasticity rises
monotonically with $\lambda$ while retention is optimized at an interior value where the
memory timescale $1/\lambda$ matches the drift rate. \emph{(iv)} A reduced-order metric
(top-$r$ eigenspace with a Woodbury inverse) reproduces the full flow at
$\mathcal O(dr)$ cost ($|{\rm diff}|=0.02$)---the structure-preserving low-dimensional
submanifold carries the whole dynamics. Boundary detection, OOD rejection (a transient
adds a decaying rank-one term to $M$), capacity (the spectrum of $M$, released by
$\lambda$), and gating thus collapse into one object. The flow is
natural-gradient/K-FAC-shaped, but lifting it to a deep network as a running
Kronecker-factored preconditioner \emph{fails} (\texttt{kaggle\_memory\_flow\_deepnet.py};
10-task Split-CIFAR-100, single seed). Swept over $\varepsilon\in\{0.3,1,3,10\}$, the
preconditioned flow either freezes after the first task or forgets like naive and never
protects a past task, reaching $0.19$ average accuracy at best against $0.25$ for a plain
EWC penalty (naive $0.15$); EWC alone holds old-task accuracy above the chance floor. The
obstruction is geometric: the top Kronecker directions carry $\approx\!99\%$ of each
layer's factor trace and coincide with the shared low-level features, so the occupied
subspace \emph{is} the trunk that every later task reuses, and damping it removes
plasticity instead of isolating interference. A preconditioner on a shared metric cannot
separate the two; reaching depth needs a penalty formulation (the memory metric as $F$)
or genuinely task-specific subspaces.

\paragraph{Plasticity preservation: gating meets dynamical isometry.} Interference gating
decides \emph{where} an update may act but does not keep the allowed directions
\emph{learnable}: under sustained non-stationarity a network also loses plasticity as its
layer Jacobians drift from isometry and the empirical NTK grows anisotropic
\cite{rosseau2026isometry}. The failures are complementary, and pairing \IGFA{} with a
dynamical-isometry regularizer---driving each weight matrix toward orthogonality,
$R_{\mathrm{iso}}(W)=\|W^\top W-I\|_F^2$ (closed-form gradient $4W(W^\top W-I)$, no SVD), as
a decoupled AdamO step \cite{rosseau2026isometry}---is \emph{super-additive}. On a
$20$-task Split-CIFAR-100 stream with a deeper unnormalized ReLU network (orthogonal init,
task-IL masking; three seeds; \texttt{kaggle\_isometry\_igfa\_hybrid.py}), isometry alone
adds $+0.033$ average accuracy over naive and the gate alone $+0.071$, but together they
reach $0.555\pm0.034$ against a $0.331$ additive prediction---a synergy of $+0.224$, and
$2.4\times$ naive. The gate supplies retention (old-task accuracy $\sim\!0.5$ versus
naive's $\sim\!0.15$) at preserved learning while isometry supplies plasticity (learning
$0.78$, drift pinned, dead ReLUs revived); the result replicates on a harder per-task
pixel-permuted stream ($0.425\pm0.011$ versus naive $0.215$), with the same ordering. Two
mechanisms account for it. The gate must carry the correct sign---it acts only when a step
would raise a past loss ($\langle g,g_A\rangle<0$) and passes beneficial transfer
($\langle g,g_A\rangle>0$), firing on the genuine-conflict minority of steps rather than
stripping the shared learning signal that dominates a shared-feature network. And isometry
makes gating \emph{cheap}: the per-step cost $\|g-\tilde g\|/\|g\|$, the fraction of the
gradient the projection removes, falls $39$--$43\%$ once isometry is added---consistently
across mixed-aggregate (A-GEM) and per-task (GEM) gates and both streams---because in the
high-rank isometric representation the conflicting component is a smaller share of the
update, so protecting the past costs less of the present. A corollary: the sharper
per-task gate \emph{requires} isometry, since alone it collapses the representation to rank
$\sim\!1$ and underperforms the aggregate gate, but with isometry holding effective rank
high ($25$ versus $5.6$) it becomes the strongest protector. The claim is
comparative---absolute accuracies are low for a small unnormalized network on hard
task-IL---but the ordering and the cost mechanism are stable across three seeds and two
streams.

\paragraph{Multimodal interference is block-structured.} For a joint embedding
$\phi=[\phi_v,\phi_l]$, the second moment $\Sig_A$ partitions into intra-modal blocks
$\Sig_v,\Sig_l$ and a cross-modal block $\Sig_{vl}$, and the interference energy
$\tfrac12\Del^\top\Sig_A\Del$ distributes linearly across them. The cross-modal (alignment)
interference of a fixed update scales with the off-diagonal $\|\Sig_{vl}\|$---negligible when
the modalities are unaligned and dominant when they are (energy ratio $\gg1$ across all seeds)
---so catastrophic decoupling is carried strictly by $\Sig_{vl}$. A cross-modal \IGFA{} would
$\Sig_{vl}$-orthogonalize the alignment directions while sharing the intra-modal ones. This is a
synthetic decomposition; a real multimodal encoder is needed to make it conclusive.

\paragraph{The functional governs physics-informed constraints.} A differential constraint is
a further quadratic form in feature space. For a physics-informed network with frozen features
$\phi$, the PDE residual acts on $\phi''$, so the physics second moment is
$\Sig_{\mathrm{PDE}}=\mathbb E[\phi''\phi''^\top]$ and a data-only update $\Del$ raises the
physics loss by exactly $\nabla L_{\mathrm{PDE}}^\top\Del+\tfrac12\Del^\top\Sig_{\mathrm{PDE}}\Del$.
On a $1$D Poisson problem (sine random-feature basis) the interference energy
$\tfrac12\Del^\top\Sig_{\mathrm{PDE}}\Del$ alone predicts the measured spike in the PDE residual
at $r=0.997$, and the full identity to machine precision ($r=1.00000$, relative error
$2\times10^{-15}$; Fig.~\ref{fig:applications}a): the functional quantifies how empirical data
corrupt a physical prior, with the only change being $\phi\mapsto\phi''$. Projecting the data gradient off the leading
physics directions reduces the induced violation ($3\times$ here) but cannot eliminate it when
$\Sig_{\mathrm{PDE}}$ is full rank---an instance of the distortion floor of
Theorem~\ref{thm:floor} under a shared active subspace.

\paragraph{Policy-gradient interference is governed by the Fisher.} For a linear-Gaussian
policy $\pi(a\mid s)=\mathcal N(\theta^\top s,\sigma^2)$ under reward $-(a-w^\top s)^2$ the
expected return is $J(\theta)=-(\theta-w)^\top F(\theta-w)-\sigma^2$ with $F=\mathbb E[ss^\top]$
the empirical Fisher: the $\Sig$ of supervised interference is replaced verbatim by $F$. Across
two control environments sharing a three-dimensional state subspace, routing the policy gradient
orthogonal to the old-environment Fisher subspace retains the first environment perfectly
(return $0.000$ vs naive $-10.9$; a Fisher-weighted anchor, the EWC analogue, reaches $-1.5$;
Fig.~\ref{fig:applications}b) while solving the second in its free directions. The removability dichotomy of
Corollary~\ref{cor:removability} holds unchanged in reinforcement learning; only the metric's
estimator changes. This is the noise-free expected gradient---a stochastic PPO/MuJoCo test is
left to future work.

\paragraph{The distortion floor can be trained away.} Theorem~\ref{thm:floor}'s floor vanishes
when task second moments are mutually orthogonal, a property a feature map can be optimised to
enforce. Adding the penalty $\sum_{A<B}\operatorname{Tr}(\Sig_A\Sig_B)$ to a shared encoder
drives its per-task subspaces apart: on overlapping synthetic tasks it cuts mean inter-task
overlap $10.6\times$ ($0.27\to0.03$) and the floor proxy $\operatorname{Tr}(\Sig_A\Sig_B)$ by
over three orders of magnitude ($3.7\to10^{-3}$; Fig.~\ref{fig:applications}c), leaving \IGFA{}
to operate in the disjoint, lossless regime of Corollary~\ref{cor:removability}. Structural allocation and representation
learning are thus complementary: a backbone meta-trained for feature orthogonality removes the
need for gating, so representation learning and structural allocation address the same
distortion floor.

\paragraph{Certified unlearning is the dichotomy inverted.} Removability read backwards is
a geometric feasibility criterion. Projecting a trained model's update off $\range\Sig_F$ of a
forget set removes that set's influence exactly when forget and retain supports are
disjoint: held-out forget-set loss rises to the untrained level while retain-set
performance is untouched. When the supports share $k$ directions the removal is
necessarily lossy, and the collateral damage to the retain set is not merely bounded but
\emph{predicted exactly} by the retain energy on the shared support (correlation $1.00$,
zero mean error, $44$ seeds, $k=0$--$3$): the same floor that lower-bounds forgetting
lower-bounds the retain cost of exact unlearning. Support overlap is thus a principled
impossibility boundary for machine unlearning under shared representations. The two
regimes separate identically on the paper's own real benchmarks
(\texttt{evaluate\_unlearning\_real.py}): on near-disjoint Split-Digits, one projection
matches retraining without the forgotten task on \emph{both} sets (retained accuracy
$0.991$ versus retrain $0.993$; the forgotten task falls to the retrain level), while on
Rotated-Digits (overlap $\approx0.68$) the same projection necessarily damages retained
rotations ($+0.37$ mean accuracy loss versus retrain)---certified removal where supports
are disjoint, the floor where they are not.

\begin{figure}[htb]\centering
\includegraphics[width=\textwidth]{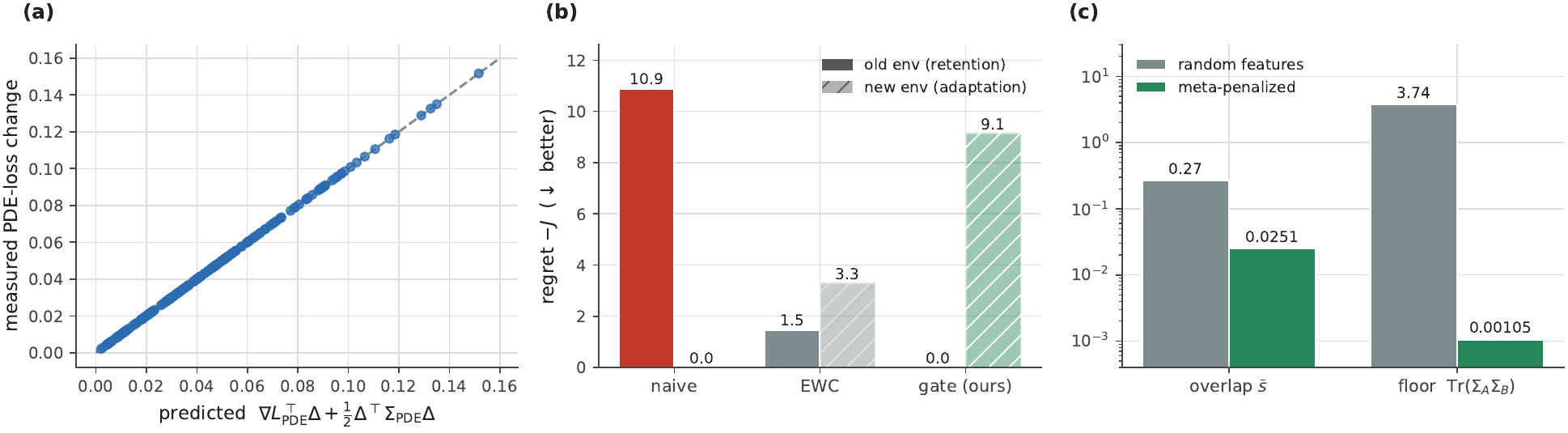}
\caption{Three application domains for the interference functional
(\texttt{evaluate\_applications.py}). \emph{(a)} Physics-informed regression: the
interference identity with $\Sig_{\rm PDE}$ predicts the measured PDE-loss change of
data-only updates to machine precision (dashed: $y=x$). \emph{(b)} Policy gradients:
routing the expected policy gradient orthogonal to the old environment's Fisher subspace
retains the old policy at zero regret while the new environment is learned in its free
directions; naive descent forgets, and the Fisher-weighted anchor (the EWC analogue) is
intermediate. \emph{(c)} A feature map trained with the $\operatorname{Tr}(\Sig_A\Sig_B)$
penalty reduces inter-task overlap and the distortion-floor proxy by orders of magnitude
(log scale), moving \IGFA{} into the lossless regime.}
\label{fig:applications}
\end{figure}

\section{Memory systems and control from the functional}
\label{si:memory}
Beyond deciding \emph{where} updates may act, the functional prescribes \emph{how much} to
move and \emph{when memory is worth spending}. Four exact-regime demonstrations (multi-seed
means; \texttt{evaluate\_next\_steps.py}) make this operational. The framing follows
complementary-learning-systems theory: a fast memory system supplies rehearsal, and the
geometry decides when to invoke it.

\paragraph{The functional is a trust region.} Constraining each update by its predicted
interference, $\tfrac12(w-w_A)^\top\Sig_A(w-w_A)\le B$---equivalently, descending
$L_B+\mu\,\tfrac12(w-w_A)^\top\Sig_A(w-w_A)$---traces a strictly better
retention--plasticity frontier than generic step-size or early-stopping control:
$31$--$59\%$ lower deferred new-task loss at matched forgetting across budget levels, with
plain online gradient descent on the penalty matching the closed-form optimum. The gate of
Sec.~\ref{sec:algo} is the $\mu\to\infty$ limit of this controller on the conflicting
subspace; intermediate budgets interpolate continuously between naive descent and full
protection. The controller's metric can also be sketched: at sketch size $\ell\ge2r$ the
Frequent-Directions metric reproduces the full-covariance frontier to $1.3\pm0.5\%$ at
$3.6\times$ less memory, while $\ell=r$ fails outright ($43\%$ gap)---the same $\ell=2r$
sizing rule as Sec.~\ref{si:cheapgn}. At language-model scale, however, the advantage does
not persist in this form: on a four-domain sequential LoRA stream with real measured
forgetting, a \emph{frozen}-sketch penalty and tuned Adam schedules trace statistically
indistinguishable frontiers ($\pm0.2\%$ at matched forgetting, two seeds), with the
schedule Pareto-dominant at the low-forgetting end
(\texttt{kaggle\_interference\_optimizer.py}). This is the stale-metric failure mode of
Sec.~\ref{sec:exp:drift} in optimizer form---frozen geometry barely helps under joint
training---and it points to the current-signal controller (the functional gate below and
in Sec.~\ref{si:llm}) as the version predicted to survive scale, a prediction the
functional gate then confirmed on the same stream. Under gradient noise, the
residual risk of protection
itself obeys a verified diffusion law: with basis error at principal angle $\epsilon$, the
noise-driven forgetting is $\tfrac12\eta^2\sigma^2 t\,r\sin^2\epsilon$ (measured-to-law
ratio $1.05$--$1.07$ across $\eta$ and $\epsilon$; exactly zero at $\epsilon=0$), which
inverts to a retention-aware step-size bound
$\eta^\ast=\sqrt{2\tau/(\sigma^2 T r\sin^2\epsilon)}$ for tolerance $\tau$ over horizon
$T$ (\texttt{evaluate\_upgrades.py}).

\paragraph{Consolidation is spectral compression with a proven price.} When tasks share
structure, the per-task protect bases are redundant. Consolidating them onto the joint
spectrum---eigenvectors of the accumulated second moment at fixed captured energy---
compressed the protect set by $55\%$ ($24\to11$ stored directions) while an adversarially
conflicting new task induced at most $0.004$ forgetting on the old family, within the
dropped-eigenvalue bound of $0.04$ that Theorem~\ref{thm:functional} supplies in advance.
Offline consolidation (``sleep'') is spectral compression whose cost the theory prices
before it is paid.

\paragraph{Memory is worth spending where the floor is positive.} With a fixed total replay
budget over a stream of mixed tasks, allocating exemplars in proportion to each task's
\emph{predicted} distortion floor (from $\Sig$-overlap and target disagreement) retained the
anchor task $28\pm31\%$ better than a uniform budget, improving in $18$ of $20$ seeds.
Structural allocation handles the disjoint tasks for free, and memory concentrates where the
geometry says forgetting is otherwise unavoidable: the fast system supplies rehearsal, the
functional decides when it is needed. The floor sharpens this principle into an exact
policy. Full rehearsal converges to the joint optimum, which lands every task precisely
\emph{at} its floor share and never below it (verified to $3\times10^{-17}$;
\texttt{evaluate\_replay\_policy.py}): rehearsal cannot buy back the floor, so the
\emph{value} of replaying task $t$ is its recoverable excess $L_t-\mathcal{D}^\star_t$,
and the optimal fixed-budget allocation is the fractional knapsack---greedy in
recoverable excess per exemplar (matches the enumerated optimum in $20$ of $20$ random
instances). Floor-\emph{proportional} allocation is the valid proxy when damage scales
with conflict, the regime of the experiment above; it misallocates in two identifiable
regimes---a task already at its floor has zero replay value however large that floor,
and a zero-floor task damaged by capacity coupling is fully recoverable despite its
floor---where it drops below even a uniform budget ($3.56$ versus $2.71$; optimum
$2.17$). All three quantities in the policy---damage, floor, cost---are pre-flight
measurables of the ledger.

\paragraph{The subspace state is a replay generator.} In the exact regime, the state
\IGFA{} already carries---the basis $U$, its eigenvalues, and a weight snapshot---is a
sufficient statistic for replay: sampling $\hat x=U\Lambda^{1/2}z$ with labels from the
snapshot and rehearsing on these ``dreamed'' pairs reduces the floor exactly as real replay
does (mean forgetting $1.71$ versus $1.72$; seedwise agreement $r=0.97$) with zero stored
examples. The construction extends to \emph{rollouts}: dreamed trajectories from a learned
world model rehearse the old task nearly as well as real experience ($2.02$ versus $1.90$
forgetting, against $4.23$ without replay) and degrade gracefully with model error ($2.48$
at $20\%$ dynamics noise), so counterfactual replay inherits the floor reduction up to the
world model's own fidelity. Replay and structural allocation are therefore not opposing
families: the same low-rank state implements both, and buffer-free generative rehearsal
inherits replay's floor reduction at sharply reduced exposure. The privacy claim can be
made formal: the state is a covariance release, so the analytic Gaussian mechanism applies
directly; with $n=600$ samples behind the release, calibrated noise at
$(\epsilon{=}8,\delta{=}10^{-5})$ leaves the released subspace at fidelity $0.99$ with a
likelihood-ratio membership attack at chance, while at small $n$ the sensitivity $2/n$
makes certified privacy and utility incompatible
(\texttt{evaluate\_upgrades.py})---differentially private dream replay is viable exactly
for data-rich tasks.

\section{Beyond frozen features: the function-space formalism}
\label{si:beyond}
Assumption A1 confines the exact \emph{parameter-space} results to frozen features, PEFT,
and first-order NTK settings. The formalism itself does not end there: recast in function
space, the identity, the dichotomy, the floor, and the retention guarantee all survive with
A1 removed. Three results follow, each exact for arbitrary architectures, with validations
on jointly trained deep networks (\texttt{evaluate\_beyond\_a1.py}).

\paragraph{The interference identity is exact in function space.} For any predictor $f$
let $L_A(f)=\tfrac12\,\mathbb{E}_{D_A}(f(x)-y)^2$, and let $f_A,f_B$ be the functions
realized before and after learning $B$, by any architecture and any training procedure.
Writing $\Delta\!f=f_B-f_A$ and $r_A=f_A-y$, expanding the square gives the identity
\begin{equation}
\Delta L_A\;=\;\mathbb{E}_{D_A}\!\big[\Delta\!f\;r_A\big]\;+\;\tfrac12\,
\mathbb{E}_{D_A}\!\big[\Delta\!f^{\,2}\big].
\label{eq:funcspace}
\end{equation}
If $A$ is realizable and trained to convergence ($r_A=0$ $D_A$-a.s.), forgetting \emph{is}
the function-space interference energy $\tfrac12\|\Delta\!f\|^2_{L^2(D_A)}$: the metric is
the data distribution itself, and Theorem~\ref{thm:functional} is the linear shadow of
Eq.~\eqref{eq:funcspace} (under A1, $\Delta\!f=\Del^\top\phi$ recovers
$\tfrac12\Del^\top\Sig_A\Del$). Numerically, Eq.~\eqref{eq:funcspace} holds to machine
precision on jointly trained depth-$1$--$3$ networks (maximum relative error
$2\times10^{-7}$, $r=1.000000$), including the depth-$3$ regime where the frozen-curvature
predictor of Sec.~\ref{sec:exp:breakdown} degrades to $r\approx0.4$. Exact prediction
beyond A1 costs two forward passes on $A$-data---no Hessian and no path integral; the
parameter-space results contribute the structure, and A1 is the regime where the two
pictures coincide.

\paragraph{The dichotomy and the floor survive, with input support in place of feature
support.} Learning $B$ is lossless for $A$ iff $\Delta\!f=0$ on
$\mathrm{supp}\,D_A$. For a universal model class this is achievable whenever the input
supports are disjoint, at any depth; and when they overlap with disagreeing targets, the
floor is a property of the data alone: minimizing pointwise,
\begin{equation}
\min_f\;L_A(f)+L_B(f)\;=\;\tfrac12\!\int\!
\frac{p_A(x)\,p_B(x)}{p_A(x)+p_B(x)}\,\big(y_A(x)-y_B(x)\big)^2\,dx,
\label{eq:funcfloor}
\end{equation}
the density-overlap harmonic mean weighted by the squared target disagreement---the
input-space form of Theorem~\ref{thm:floor}, and a lower bound for every concrete
architecture, whose own floor equals Eq.~\eqref{eq:funcfloor} plus an approximation-class
excess. Measured with held-out (population) losses, a depth-$3$ network trained jointly on
conflicting sine tasks respects the bound with a $\approx20\%$ class excess (predicted
$0.50$, measured $0.60$), and near-disjoint input supports give the lossless regime at any
depth (predicted $0.001$, measured $0.009$).

\paragraph{Exact retention without A1: instantaneous null-space control.} Along any
absolutely continuous training path $\theta(t)$,
$\Delta L_A=\int\langle\nabla L_A(\theta_t),\dot\theta_t\rangle\,dt$ exactly, with
$\nabla L_A(\theta)=\mathbb{E}_{D_A}[\,r_A(x)\,\nabla_\theta f(x;\theta)\,]$. If the
velocity stays in the kernel of the \emph{current} Gauss--Newton metric,
$\dot\theta_t\in\ker G_A(\theta_t)$ with
$G_A(\theta)=\mathbb{E}_{D_A}[\nabla_\theta f\,\nabla_\theta f^\top]$, then
$\dot\theta^\top G_A\dot\theta=0$ forces
$\langle\nabla_\theta f(x),\dot\theta\rangle=0$ for $D_A$-almost every $x$, hence
$\langle\nabla L_A,\dot\theta\rangle\equiv0$ and \emph{$L_A$ is conserved exactly}---any
depth, any feature drift, no A1 and no realizability. Discrete steps of size $\eta$ leak
only the Taylor remainder $\tfrac12\eta^2 v^\top\bar H\,v$ per step
(Theorem~\ref{thm:general}), i.e.\ $O(\eta)$ at fixed path length. Measured at fixed path
budget on jointly trained networks, forgetting under per-step re-estimated null-space
control falls $0.134\to0.0037\to0.0004\to0.00003$ as $\eta$ drops $0.1\to0.003$, while a
frozen protected basis plateaus near $0.5$ (staleness bias does not vanish with $\eta$)
and naive training forgets at $O(1)$. Assumption A1 is thereby relocated, not removed: a
constant metric makes the kernel constant, so finite steps stay in it exactly---A1 buys
\emph{finite-step} exactness, while the geometry, the dichotomy, the floor, and the
continuous-time retention guarantee never needed it. The recursive Gauss--Newton tracker
of Sec.~\ref{sec:algo} is the discretization of this exact law, with the finite-step,
estimation, and capacity costs quantified in Secs.~\ref{sec:exp:drift},
\ref{si:cheapgn}, and~\ref{si:memory}.

\paragraph{Finite-step exactness recovered: two parameter-space routes.} Even the
finite-step gap closes, by two complementary mechanisms
(\texttt{evaluate\_beyond\_a1.py}). \emph{(i) Region-exactness for piecewise-linear
networks.} For ReLU-family models the loss is \emph{exactly} quadratic in each layer's
weights while the activation pattern on the $A$-support is fixed: per-layer, Theorem~1
holds at finite step size inside the activation region, and the region boundary has a
closed form---for a step $sV$ in layer weights, no activation flips while
$s<\min_{i,j}|{\rm pre}_{ij}|/|(Vx_i)_j|$, computable from the stored preactivation
margins. Measured: the quadratic is exact to relative error $10^{-11}$ inside the region,
and the certificate is conservative (breakdown observed at $1.4\pm0.4\times$ the
threshold, never meaningfully before it). Assumption A1 thus relaxes, for ReLU networks,
to a checkable per-step margin condition, and the margin joins the pre-flight/ledger
family as a monitorable quantity. \emph{(ii) Retraction for any smooth architecture.}
Exact retention means staying on the anchor level set
$\{\theta: f(x_i;\theta)=c_i\}$, and that manifold is returnable: after each ordinary
update, two Gauss--Newton corrector steps on the anchor residuals restore it. Measured on
jointly trained networks, anchor-set forgetting sits at numerical zero
($\sim10^{-14}$) \emph{independent of step size}, where projection alone leaks $O(\eta)$
($0.17$, $0.043$, $0.0015$ at $\eta=0.1,0.05,0.02$), while the new task still trains.
The remaining limitation is one of coverage: exactness
holds on the anchors, and between them the held-out $A$-loss depends on how well the anchor
set pins the function class ($0.15$--$0.26$ here at $24$ anchors). Beyond A1, retention
is governed by how densely the anchors sample the function class.

\paragraph{Monotone-retention gating: the threshold-free gate.} The path integral also
yields the gate itself, with the similarity threshold removed. Since
$dL_A/dt=\langle\nabla L_A,\dot\theta\rangle$, the exact, modality-free rule is a
\emph{sign test}: given the new task's update direction $g$, if
$\langle g,\nabla L_A\rangle\ge0$ the step also descends (or preserves) $L_A$ and is kept
whole; otherwise only its component along $\nabla L_A$ is removed,
$g\mapsto g-\frac{\langle g,\nabla L_A\rangle}{\|\nabla L_A\|^2}\nabla L_A$. In continuous
time this guarantees $dL_A/dt\le0$---retention is monotone---while every non-harmful
component of the update, and therefore all transfer, is retained; no threshold, no input
geometry, and no similarity estimate appears. The single-constraint form recovers the
projections of A-GEM \cite{chaudhry2019} and the surgery of PCGrad \cite{yu2020pcgrad} as
special cases, here derived from the exact identity rather than posited; per-direction
refinements in the current $G_A$-eigenbasis coincide with the single constraint at $A$'s
optimum (every range-component of motion is harmful there) and differ off-optimum. In the
exact regime the sign gate holds forgetting at $0.025$ versus naive $10.06$ under pure
conflict while reaching strictly lower deferred loss than unconditional projection ($9.08$
versus $10.06$), and is harmless under alignment; on jointly trained networks it attains
near-current-GN retention under conflict ($+0.14$ versus naive $+5.77$ and \emph{frozen}
projection $+5.38$, which barely helps under drift) with better adaptation than
current-GN ($0.11$ versus $0.15$) at a fraction of its cost---one gradient of $L_A$ per
step, from a micro-cache, the anchor construction of Sec.~\ref{si:cheapgn}, or the dream
state of Sec.~\ref{si:memory} (\texttt{evaluate\_functional\_gate.py}). On real text it
significantly beats naive continual fine-tuning on both axes, threshold-free and without
input geometry, and edges unconditional protection (Sec.~\ref{si:llm}).

The gate is, moreover, not merely a safe modification of the update but the
\emph{optimal} one.

\begin{corollary}[The sign gate is optimal monotone control]
\label{cor:qpgate}
Let $u$ be the desired update and let tasks $1,\dots,m$ be protected. The problem
\begin{equation}
v^\ast=\operatorname{argmin}_v \tfrac12\|v-u\|^2
\quad\text{s.t.}\quad \langle\nabla L_{A_i},v\rangle\le 0,\;\; i=1,\dots,m,
\label{eq:qpgate}
\end{equation}
has the solution $v^\ast=u-G\lambda^\ast$ with $G=[\nabla L_{A_1}\cdots\nabla L_{A_m}]$
and $\lambda^\ast=\operatorname{argmin}_{\lambda\ge0}\|G\lambda-u\|$ (a nonnegative
least-squares problem). For $m=1$, $v^\ast$ is exactly the sign gate; when no constraint
is violated, $v^\ast=u$ (sharing); when all are violated, $v^\ast$ remains feasible yet
strictly less restrictive than equality projection: protected rates may be negative---the
update is allowed to \emph{improve} protected tasks, which projection forbids---and
$\|v^\ast-u\|\le\|v_{\mathrm{proj}}-u\|$ always.
\end{corollary}
\noindent\emph{Intuition:} stationarity of the Lagrangian gives $v=u-G\lambda$;
substituting back turns the dual into the stated nonnegative least squares, whose
Karush--Kuhn--Tucker conditions are precisely primal feasibility and complementary
slackness. The gate therefore interpolates between naive sharing and subspace projection
through the active set. Unconditional projection hardens those inequality constraints to
equalities, which in the exact regime never improves on the gate. All
statements are machine-checked ($\|v^\ast-v_{\mathrm{gate}}\|=0$ at $m=1$; feasibility to
$10^{-15}$; measured protected-rate decrease under full activation;
\texttt{evaluate\_functional\_control.py}).

\paragraph{Beyond squared loss: the Bregman family.} The function-space results are not
quadratic-loss artifacts. For any canonical loss $\ell(y,f)=\varphi(f)-y^\top f$ with
$\varphi$ convex and mean map $\mu=\nabla\varphi$ (squared loss:
$\varphi=\tfrac12\|f\|^2$; softmax cross-entropy: $\varphi=\mathrm{logsumexp}$,
$\mu=\mathrm{softmax}$; logistic and Poisson likewise), the definition of the Bregman
divergence $D_\varphi(b,a)=\varphi(b)-\varphi(a)-\mu(a)^\top(b-a)$ gives, exactly and for
any finite functional change,
\begin{equation}
\Delta L_A=\mathbb{E}_{D_A}\!\big[(\mu_A-y)^\top\Delta\!f\big]
+\mathbb{E}_{D_A}\!\big[D_\varphi(f_B,f_A)\big],
\label{eq:bregman}
\end{equation}
of which Eq.~\eqref{eq:funcspace} is the quadratic case. For cross-entropy,
$D_\varphi(f_B,f_A)=\mathrm{KL}(p_A\,\|\,p_B)$ exactly, so on a converged task
($\mu_A=y$ on $\mathrm{supp}\,D_A$) forgetting \emph{is} the expected
Kullback--Leibler divergence between the old and new predictive distributions. Three
consequences follow. \emph{(i)} The removability dichotomy is loss-free: under A1 a
kernel step changes no function value on the support, hence $\Delta L_A=0$ exactly at
any finite step size for \emph{any} loss (validated at $\|\Delta W\|=10$); for softmax
the lossless set is strictly larger, since per-sample constant logit shifts are
invisible. Only the closed quadratic \emph{energy formula} is squared-loss-specific.
\emph{(ii)} The distortion floor has a closed information-geometric form: pointwise,
$\min_q\,[\,p_A\mathrm{KL}(\pi_A\|q)+p_B\mathrm{KL}(\pi_B\|q)\,]$ is attained at the
density-weighted \emph{mixture} $q^\ast=(p_A\pi_A+p_B\pi_B)/(p_A+p_B)$, giving
\begin{equation}
\text{floor}=\!\int\!\big(p_A(x)+p_B(x)\big)\,
\mathrm{JSD}_{w(x)}\!\big(\pi_A(\cdot|x),\pi_B(\cdot|x)\big)\,dx,
\qquad w=\tfrac{p_A}{p_A+p_B},
\label{eq:jsdfloor}
\end{equation}
a weighted Jensen--Shannon divergence between the tasks' conditionals; the quadratic
floor of Theorem~\ref{thm:merge}'s function-space form is its small-divergence
(Gaussian) limit, and the mixture replaces the $\Sig$-weighted mean as the optimal
joint predictor. For $T$ tasks with uniform prior the floor becomes an
\emph{information identity}: writing $T$ for the task variable, $X$ for the input, and
$Y$ for the label,
\begin{equation}
\mathcal{D}^\star_T=\sum_x\sum_t p_t(x)\,\mathrm{KL}\!\big(\pi_t\,\|\,\bar\pi_x\big)
=T\cdot I(T;Y\mid X)=T\,\big[H(Y\mid X)-H(Y\mid X,T)\big],
\label{eq:infofloor}
\end{equation}
the conditional mutual information between task identity and label---the exact
label-entropy cost of hiding which task a sample came from. The entropy
\emph{sandwich} of Theorem~\ref{thm:mi} (quadratic loss) thus closes to an equality
under the loss language models train, removability becomes $Y\perp T\mid X$, and the
floor is achievable, and a trained shared softmax head converges to
Eq.~\eqref{eq:infofloor} from above with terminal gap $6\times10^{-16}$
(\texttt{evaluate\_info\_floor.py}; identity and mixture-optimality checks at
$10^{-16}$). Operationally, the irreducible single-model floor of an LLM domain
stream is $I(\text{domain};\text{next token}\mid\text{context})$, estimable from
micro-caches as the mean per-domain versus mixture log-likelihood ratio.
\emph{(iii)} The sign gate and Corollary~\ref{cor:qpgate} act on
$\langle g,\nabla L_A\rangle=dL_A/dt$ and never used the quadratic form, so they are the
exact monotone controller for cross-entropy as stated---the language-model gate
experiments of Sec.~\ref{si:llm} therefore run the exact controller for the loss they
train, not a transferred heuristic. All four statements are validated to machine
precision ($10^{-16}$; kernel step exactly $0$), with the closed-form floor checked
against direct numerical minimization (\texttt{evaluate\_bregman\_identity.py}).

\paragraph{The forgetting decomposition, validated.} With
$\mathcal{D}^\star_\infty\le\mathcal{D}^\star_{\mathrm{class}}\le L$ as defined in the
Discussion,
\begin{equation}
L=\underbrace{\mathcal{D}^\star_\infty}_{\text{incompatibility}}
+\underbrace{\big[\mathcal{D}^\star_{\mathrm{class}}-\mathcal{D}^\star_\infty\big]}_{\text{capacity}}
+\underbrace{\big[L-\mathcal{D}^\star_{\mathrm{class}}\big]}_{\text{control}},
\label{eq:decomposition}
\end{equation}
an identity whose content is that the three terms are separately measurable and
separately treatable. Three single-regime streams land their intended term (ten seeds;
\texttt{evaluate\_forgetting\_decomposition.py}): pure conflict under full replay is all
floor ($12.39/0/0$); input-disjoint tasks squeezed through width-$12$ features are all
capacity ($0/5.35/0$); mixed tasks under naive training carry floor plus control
($5.42/0/4.04$). The intervention matrix attributes cleanly: widening the features
$12\to24$ removes $100\%$ of the capacity term while no policy---naive, active-set, or
full replay---moves it at all ($6.02$ untouched by all three); upgrading the policy cuts
only the control term ($4.65\to2.15$, floor invariant throughout). Two definitional
points the construction forces into the open: \emph{capacity} means representation
width---the ``best rank-$r$ subspace containing the solution'' notion collapses, since
any single optimum spans one dimension---and, in the quadratic regime,
\emph{perturbed} copies of a subspace are not conflict (their soft constraints are
jointly satisfiable and the floor is exactly zero; genuine conflict requires exactly
shared support, a knife-edge the cross-entropy floor does not have, being nonzero
whenever conditionals disagree on overlapping support).

\paragraph{Benign forgetting: the anchor splits knowledge from memorization.} The
interference identity is distribution-anchored, and nothing requires the anchor to be
the training set. Evaluated on task $A$'s \emph{deployment} distribution the same
identity splits observed forgetting into knowledge loss, $\Delta L_A^{\mathrm{test}}$,
and memorization loss, $\Delta L_A^{\mathrm{train}}-\Delta L_A^{\mathrm{test}}$---and
the two can carry opposite signs. In an overparameterized interpolation setting with
label noise (ten seeds; \texttt{evaluate\_benign\_forgetting.py}), an interfering update
raises train loss by $+1.99$---catastrophic forgetting under any train-anchored
metric---while \emph{lowering} deployment loss by $-3.37$: backward transfer through
noise removal, with the sign predicted exactly by the test-anchored identity
($8\times10^{-16}$). Part of what benchmarks report as catastrophic forgetting is
implicit unlearning of noise. The consequence for protection is a design law:
\emph{protect the function you validated, not the function you fit}. A gate whose
constraint gradient is anchored on training data defends the memorized noise together
with the knowledge---it blocks benign repair and pays adaptation for the privilege---
whereas the held-out-anchored gate dominates it on \emph{both} axes (final deployment
loss of the protected task $8.3$ versus $22.0$; adaptation $5.4$ versus $7.1$). The
anchor's own quality gates the gate: with a held-out set too small to be validated
itself ($M=30<d$, the anchor interpolates its own noise) the advantage collapses,
which is the estimation--abstention rule's other face. The held-out micro-caches used
throughout the language experiments (Sec.~\ref{si:llm}) are therefore mandated by the
theory, not merely hygienic.

\paragraph{Certified unlearning is one reversed row of the gate.} Engineered benign
forgetting gives an unlearning method for free: \emph{ascend} the forget set's
\emph{validated} loss toward the never-trained reference while the monotone-retention
QP of Corollary~\ref{cor:qpgate} holds every retained held-out anchor in place---the
unified gate with a single reversed constraint. The certificate is then a
\emph{measurement}, $\mathbb{E}_{\mathrm{retain}}[\mathrm{KL}(p_{\mathrm{before}}\|
p_{\mathrm{after}})]\le\epsilon$, rather than a retraining-equivalence definition. The
anchor split predicts a dichotomy, validated to machine precision
(\texttt{evaluate\_unlearning.py}): \emph{memorized} content (own support) is removable
at a zero floor, whereas \emph{entangled} content shared with a retained task carries a
closed-form \emph{unlearning floor} $\sum_{\mathrm{shared}}p_R\,\mathrm{KL}(\pi_R\|
p_{\mathrm{ref}})$---the surgical optimum pays exactly this (agreement $3\times10^{-7}$),
and the gate stalls at it rather than silently damaging retained knowledge. The
memorized branch holds at language-model scale. Injecting a memorized canary into joint
LoRA training on \texttt{pythia-410m} and \texttt{pythia-1b} (the earlier attempt to
unlearn an entire natural domain was vacuous---the adapter had added only $\sim0.03$
nats of removable domain-specific knowledge, the pre-flight lesson again), both the gate
and naive ascent remove it fully to the base-model reference, but the gate pays $\sim
1.7\times$ less collateral damage at both scales (retained certificate $0.101$ versus
$0.177$ at $410$m, $0.024$ versus $0.041$ at $1$b; the directly measured retained-loss
rise, $+0.09$ versus $+0.18$ and $+0.03$ versus $+0.07$ nats, tracks the certificate).
The gate finds the zero-floor path that unstructured gradient ascent misses.

\paragraph{Attributed forgetting at scale.} The decomposition
Eq.~\eqref{eq:decomposition} is measurable end to end on real continual pretraining,
yielding forgetting curves that state \emph{which part was ever avoidable}
(\texttt{pythia-410m} and \texttt{pythia-1b}, LoRA $r{=}8$ on the four-domain stream;
per-domain single-task references, joint training as the class-optimum surrogate at
ranks $\{2,8,32\}$, sequential naive and active-set arms;
\texttt{kaggle\_forgetting\_attribution.py}, \texttt{kaggle\_floor\_certificate.py},
archived results and analysis in the repository). \emph{The floor is certified, not
assumed}: the generative domain classifier built from the single-domain adapters
separates held-out chunks perfectly at $410$m (accuracy $1.000$, cross-entropy
$0.0048$ nats/token), and since any classifier's cross-entropy upper-bounds
$H(T\mid X)\ge I(T;Y\mid X)$ and the floor is a property of the data, the one bound
governs both models: at least $97.1\%$ ($410$m, observed forgetting $0.164$ nats/token)
and $98.8\%$ ($1$b, $0.399$) of the measured forgetting was \emph{avoidable in
principle}. The attribution, in the exact KL observable (per-token
$\mathrm{KL}$ to the single-domain anchors, netted of the measured seed-to-seed
two-models baseline of $0.028$/$0.020$):
\begin{center}\small
\begin{tabular}{lcc}
\toprule
 & \texttt{pythia-410m} & \texttt{pythia-1b} \\
\midrule
net forgetting (naive) & $0.269$ & $0.624$ \\
net class term & $0.161$ & $0.249$ \\
net control term & $0.108$ & $0.375$ \\
control after the active-set gate & $0.063$ ($41\%$ removed) & $0.249$ ($34\%$ removed) \\
certified floor & \multicolumn{2}{c}{$\le 0.0048$} \\
\bottomrule
\end{tabular}
\end{center}
Three readings. First, the class term is real ($6$--$13\times$ the seed baseline) but
\emph{flat in adapter rank} across $\{2,8,32\}$: within-class interference on this
stream resides in the shared frozen geometry and joint optimization, not in adapter
width, so expansion is not the indicated remedy---policy is, and the gate removes a
third to a half of exactly the control term (cross-entropy view: naive versus gate
paired $p=0.044$ at $410$m, $p=0.0018$ at $1$b). Second, the decomposition's sanity
anchor holds on real data: sequential training fits each domain exactly as well as
single-task training (diagonal-minus-single $+0.000$ at $410$m). Third, the
Bregman residual of the preceding paragraph is visible at scale:
$\mathbb{E}[\mathrm{KL}]$ exceeds the cross-entropy excess by $0.14$--$0.23$
nats/token because the single-domain anchors are converged on their training chunks,
not on the held-out anchor. Caveats: the recipe is fixed across scales (learning rate
and budget not re-tuned, so the $410$m$\to$$1$b growth in forgetting is a
fixed-recipe comparison), joint training only upper-bounds the class optimum (the
control term is conservative), and the arms carry two to three seeds.

\paragraph{Protection blocks autoregressive cascades at the source.} In a generative loop,
a one-step corruption of the transition operator compounds over the horizon: after
fine-tuning a norm-preserving recurrence on an overlapping second task, naive rollout error
grows from $0.86$ at one step to $1.00$ (full decorrelation) at thirty. Routing the
fine-tuning update through $\ker\Sig_A$ keeps the old task's transition operator invariant,
and the rollout error is zero at \emph{every} horizon---the cascade is blocked where it
starts, not corrected downstream. Output-side correction (snapping generated states back
onto the old task's subspace) removes leakage out of the subspace but cannot repair the
corrupted dynamics inside it ($0.72\to0.99$): for autoregressive stability, where an update
acts matters more than how its outputs are post-processed.


\end{document}